\documentclass[10pt,journal,compsoc]{IEEEtran}

\newif\ifanonymous     \anonymousfalse     
\newif\iflinenumbers   \linenumbersfalse   

\ifCLASSOPTIONcompsoc
  \usepackage[nocompress]{cite}   
\else
  \usepackage{cite}   
\fi

\usepackage{bibunits}

\usepackage{amsmath}      
\usepackage{amssymb}      
\usepackage{amsthm}       
\usepackage{bm}           
\usepackage{mathrsfs}     

\usepackage{graphicx}     
\graphicspath{{./figs/}}
\usepackage{stfloats}     
\usepackage{subcaption}   

\usepackage{booktabs}     
\usepackage{multirow}     
\usepackage{makecell}     
\usepackage{array}        
\usepackage[table]{xcolor}

\usepackage[ruled,vlined,linesnumbered]{algorithm2e} 

\usepackage{textcomp}     
\usepackage{url}          
\usepackage{verbatim}     

\usepackage[switch]{lineno}   
\iflinenumbers
  \linenumbers
  \modulolinenumbers[1]
\fi

\usepackage[bookmarks=true,bookmarksnumbered=true,pdfpagelabels=true,hypertexnames=false]{hyperref} 
\usepackage{orcidlink}
\hypersetup{
  pdfpagemode={UseOutlines}, plainpages=false,
  colorlinks=true, linkcolor=blue, citecolor=blue, urlcolor=blue, filecolor=blue
}

\newtheorem{lemma}{Lemma}[section]
\newtheorem{corollary}{Corollary}[section]
\newtheorem{proposition}{Proposition}[section]

\DeclareRobustCommand{\circledot}[1]{\textcolor{#1}{\raisebox{-0.1ex}{\scalebox{1.4}{\ensuremath{\bullet}}}}}
\definecolor{figpfedfda}{HTML}{65ADD1}
\definecolor{figfedlog}{HTML}{8FC5DF}
\definecolor{figpfedgm}{HTML}{BADDED}
\definecolor{figfedgrc}{HTML}{F0AF80}

\newcommand{\ms}[2]{#1{\tiny$\pm$#2}}
\newcommand{\msb}[2]{\textbf{#1}{\tiny$\pm$#2}\rlap{$^{\dagger}$}}    
\newcommand{\mss}[2]{\textbf{#1}{\tiny$\pm$#2}\rlap{$^{*}$}}
\definecolor{gainpos}{RGB}{230,120,0}             
\definecolor{gainneg}{RGB}{70,140,210}             
\newcommand{\dltp}[1]{{\color{gainpos}\bfseries\boldmath#1}}
\newcommand{\dltm}[1]{{\color{gainneg}\bfseries\boldmath#1}}
\newcommand{\grp}[1]{{\color{gray!60!black}\textit{#1}}}
\newsavebox{\tabbox}                              
\newlength{\incw}                                 
\newcommand{\gp}[1]{{\color{gainpos}#1}}
\newcommand{\gm}[1]{{\color{gainneg}#1}}
\newcommand{\ab}[2]{#1\hspace{0.3em}\makebox[\incw][r]{(#2)}}   %
\newcommand{\abn}[1]{#1\hspace{0.3em}\makebox[\incw][r]{}}
\newcommand{\yes}{\checkmark}
\newcommand{\no}{--}

\newcommand{\vect}[1]{\boldsymbol{#1}}
\DeclareMathOperator{\tr}{tr}
\DeclareMathOperator{\diag}{diag}
\DeclareMathOperator{\Var}{Var}
\DeclareMathOperator*{\argmax}{arg\,max}

\newcommand{\unit}[1]{\operatorname{unit}\!\left(#1\right)}

\newif\ifequationcolor 
\equationcolortrue
\colorlet{eqblue}{gainneg!45!white}
\colorlet{eqorange}{figfedgrc!65!white}
\newcommand{\eqshade}[2]{\begingroup\setlength{\fboxsep}{2pt}%
  \ifequationcolor\colorbox{#1}{$\displaystyle #2$}%
  \else\colorbox{black!7}{$\displaystyle #2$}\fi\endgroup}
\newcommand{\bluepart}[2]{\underbrace{\eqshade{eqblue}{#1}}_{\text{\normalfont\scriptsize\shortstack{#2}}}}
\newcommand{\orangepart}[2]{\underbrace{\eqshade{eqorange}{#1}}_{\text{\normalfont\scriptsize\shortstack{#2}}}}

\makeatletter
\newcommand{\bibliographynamespace}[1]{\def\@extra@b@citeb{.#1}}
\newcommand{\applybibliographycontrol}{%
  \@startbibunitorrelax
  \bstctlcite[@bibunitaux]{IEEEtranBSTcontrol}}
\AtBeginDocument{%
  \renewcommand{\bibcite}[2]{%
    \global\@namedef{b@#1\@extra@b@citeb}{%
      \hyper@@link[cite]{}{cite.#1\@extra@b@citeb}{#2}}}}
\makeatother

\begin{document}

\title{Is High Local Accuracy Enough? Geometric Consensus with Support-Adaptive Statistical Personalization for Federated Learning}

\ifanonymous
  \author{Anonymous Authors}
\else
  \author{Junyu~Lin\,\orcidlink{0009-0004-9562-2714},
          Bosong~Qi\,\orcidlink{0009-0008-7011-5141},
          Pengcheng~Shi\,\orcidlink{0000-0003-2504-9890},
          and~Jiayuan~Li\,\orcidlink{0000-0002-9850-1668}%
    \IEEEcompsocitemizethanks{%
      \IEEEcompsocthanksitem Junyu Lin, Bosong Qi, Pengcheng Shi, and Jiayuan Li
      are with the School of Remote Sensing and Information Engineering,
      Wuhan University, Wuhan 430072, China (e-mail:
      \{jonathanlin1206, poisson0622, shipc\_2021, ljy\_whu\_2012\}@whu.edu.cn).
      \IEEEcompsocthanksitem Corresponding author: Jiayuan Li;
      e-mail: ljy\_whu\_2012@whu.edu.cn.}}
\fi

\markboth{}{}

\IEEEtitleabstractindextext{%
\begin{abstract}
\rightskip=0pt\relax
Real-world deployment of personalized federated learning (PFL) requires adapting to heterogeneous clients while preserving recognition performance beyond local distributions. However, high local accuracy can conceal failures on locally rare or absent classes.
Effective personalization should therefore strengthen global recognition while raising the personalization floor.
We propose an evaluation protocol that complements local accuracy with global, rare-class, and unseen-class evaluations to reveal generalization obscured by local scores, and assesses personalization floors among the worst-performing clients to expose weaknesses hidden by aggregate performance.
To meet these deployment challenges, we propose FedGRC, a unified PFL framework integrating generalization and personalization through (i) Consensus-Guided Representation Learning (CGRL), which learns discriminative representations through shared directional and magnitude guidance; (ii) Multilayer Metric Synthesis (MMS), which constructs a shared covariance metric from these multilayer representations for statistical classification; and (iii) Support-Adaptive Shrinkage Personalization (SASP), which adapts class means within this metric by shrinking local estimates toward shared means according to local sample support, enabling adaptation while retaining shared knowledge.
Extensive experiments, including benchmarking against 21 methods across 5 heterogeneity settings on 3 datasets and comprehensive ablation studies, validate the effectiveness of FedGRC.
On CIFAR-10 and CIFAR-100, FedGRC surpasses state-of-the-art (SOTA) baselines in global recognition and personalization floors by at least 2.17\% and 2.13\%, respectively.
On the real-world heterogeneous FedRS dataset, it outperforms the respective best baselines by 2.15\%, 3.30\%, and 4.61\% in local adaptation, global recognition, and the personalization floor, demonstrating its practical deployment value.
\end{abstract}

}

\maketitle
\IEEEdisplaynontitleabstractindextext
\IEEEpeerreviewmaketitle

\bibliographynamespace{paper}
\begin{bibunit}[IEEEtran]
\applybibliographycontrol
\ifCLASSOPTIONcompsoc
  \IEEEraisesectionheading{\section{Introduction}\label{sec:intro}}
\else
  \section{Introduction}\label{sec:intro}
\fi
\IEEEPARstart{P}{ersonalized} federated learning (PFL) adapts client models to local data using knowledge shared across the federation, without centralizing raw training data~\cite{FedRep_ICML21,FedROD_ICLR22,PFL_Survey_TNNLS23}.
This combination of knowledge sharing and local adaptation is relevant to healthcare~\cite{Med_NatMed21}, autonomous driving~\cite{Car_WACV25}, and remote sensing, where data distributions and acquisition conditions vary across clients.
Remote sensing provides a representative example: differences in sensors and geographic coverage produce heterogeneous image appearances and skewed local label distributions, leaving some classes well represented at a client and others rare or absent~\cite{HFL_Survey_CSUR24,FedRSBench_FedKDD25}.
This poses a core deployment challenge for heterogeneous PFL: combining all-class recognition with effective local adaptation.

Several PFL methods, however, prioritize local personalization without adequately preserving generalization beyond the local distribution~\cite{HPFL_ICLR25}.
DBE~\cite{DBE_NeurIPS23} and GPFL~\cite{GPFL_ICCV23} identify a loss of general feature information when representations are trained to serve personalized classifiers, while other studies examine the tension between personalization and generalization under changing test distributions~\cite{FedTHE_ICLR23,BTFL_KDD25}.
Motley~\cite{Motley_arXiv22} provides a concrete example: additional local fine-tuning can improve accuracy on local data while reducing accuracy on data sampled from other clients.
These observations raise the question: \textit{Is high local accuracy sufficient to assess a personalized model's recognition ability?}

The results in Fig.~\ref{fig:motivation} reveal a counterintuitive pattern for pFedFDA~\cite{pFedFDA_NeurIPS24}, FedLog~\cite{FedLog_TMLR26}, and pFedGM~\cite{pFedGM_arXiv26}: as label heterogeneity increases from Dir(0.3) to Dir(0.1), their \textit{local} accuracy rises while their \textit{global} accuracy falls.
Here, \textit{global} evaluation assesses the same client predictors on a class-balanced test set covering all federation-known classes.
Local evaluation assigns most weight to frequent classes, while locally rare or absent classes receive little or none.
A high \textit{local} score can therefore conceal weaknesses in recognition across the full label space~\cite{FDG_Survey_ProcIEEE25}.

\begin{figure}[!t]
  \centering
  \includegraphics[width=\columnwidth]{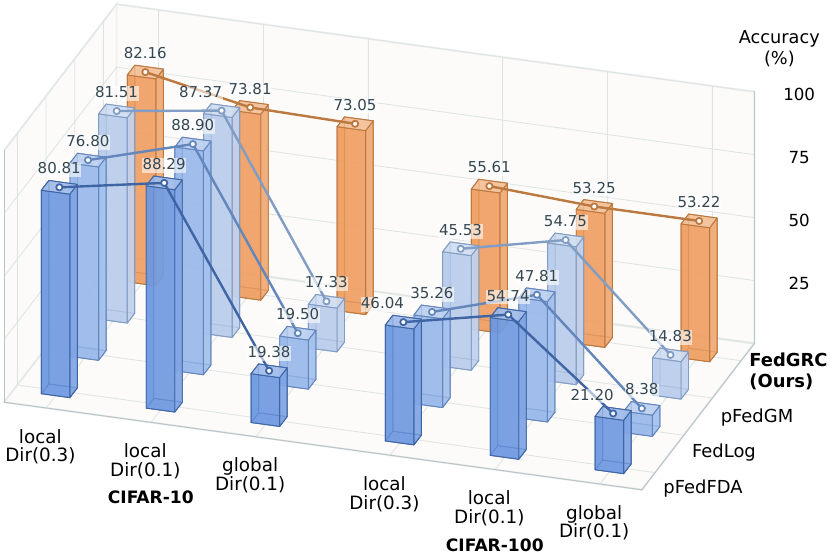}
  \caption{Local personalization and global generalization on CIFAR-10 and CIFAR-100. Colors identify methods: \circledot{rgb,255:red,102;green,145;blue,220}~pFedFDA~\cite{pFedFDA_NeurIPS24}, \circledot{rgb,255:red,151;green,183;blue,234}~FedLog~\cite{FedLog_TMLR26}, \circledot{rgb,255:red,193;green,211;blue,238}~pFedGM~\cite{pFedGM_arXiv26}, and \circledot{rgb,255:red,239;green,158;blue,96}~our FedGRC. Dir(0.3) and Dir(0.1) denote Dirichlet label partitions with strong and extreme heterogeneity, respectively. \textit{Local} denotes accuracy on local test data; \textit{global} denotes accuracy on a class-balanced test set covering all classes.}
  \label{fig:motivation}
\end{figure}

An informative evaluation of PFL under label heterogeneity should therefore extend beyond \textit{local} accuracy.
It should assess (i) adaptation to each client's data distribution; (ii) recognition across the full label space; (iii) recognition of locally rare or absent classes; and (iv) weaknesses in individual clients that aggregate results may conceal~\cite{FL_GRF_Survey_TPAMI24}.
HtFLlib~\cite{HtFLlib_KDD25} likewise evaluates client models on a global test set to assess generalization beyond local data.
Our objective is thus to improve recognition across all federation-known classes and strengthen poorly performing clients' weakest assessed abilities while retaining effective local adaptation.

Meeting this objective requires addressing three interdependent design challenges.
First, representation learning must preserve shared discriminative structure under heterogeneous local updates; the feature anchors in FedFA~\cite{FedFA_TMC24} and common class geometry in FedETF~\cite{FedETF_ICCV23} offer ways to guide this process.
Second, useful representations alone do not eliminate classifier bias, as demonstrated by CCVR~\cite{CCVR_NeurIPS21}; the classifier must draw on distributed class information without relying excessively on local label frequencies.
Statistical personalization in pFedFDA~\cite{pFedFDA_NeurIPS24} and pFedGM~\cite{pFedGM_arXiv26} provides a basis for combining shared and local feature-distribution information.
Third, adaptation must account for local class support: well-observed classes support stronger local adaptation, whereas locally rare or absent classes require greater reliance on shared knowledge.
These challenges are interdependent: the representation shapes the class statistics used by the classifier, while local sample support governs how strongly those statistics should be adapted.

We propose FedGRC, a unified PFL framework built on Geometric Representation Consensus.
FedGRC keeps the learned representation and covariance metric shared during personalization and adapts class means by balancing local and shared estimates according to local sample counts.
The counts determine the influence of local estimates, while equal class priors avoid adding a prediction preference for locally frequent classes.
Our main contributions are:
\begin{itemize}
  \item \textbf{Comprehensive PFL Evaluation beyond Local Accuracy.} We introduce a protocol that jointly assesses local adaptation, generalization across the full label space, and recognition of locally rare and absent classes. It also measures a recognition floor by taking each client's lowest accuracy across local adaptation and rare- and unseen-class recognition, then summarizing the lowest-performing clients. This reveals weaknesses hidden by aggregate scores.
  \item \textbf{Unified Personalization and Generalization.} We develop FedGRC to connect geometric representation learning with statistical personalization for local adaptation and recognition across the full label space. Shared directional references and class- and layer-specific magnitude references guide representation learning. Multilayer class statistics then yield a shared covariance metric, and support-dependent shrinkage adapts class means within this geometry.
  \item \textbf{Systematic Benchmarking and Validation.} We reproduce and benchmark 21 external methods: 15 complete federated learning (FL) methods and 6 statistical classifiers, with the latter evaluated on 2 backbone types. The evaluation spans 5 heterogeneity settings across 3 datasets with different label-space sizes, client counts, and degrees of heterogeneity, including real-world remote sensing data from FedRS~\cite{FedRSBench_FedKDD25}. Comparative experiments and ablations establish the contributions of the framework and its components. On FedRS, FedGRC gains 3.30\% in global accuracy and 4.61\% in the worst-performing clients' recognition floor over the respective best baselines, with higher local personalization accuracy.
\end{itemize}

\textbf{Open-Source Resources.} The benchmark and method implementations, design details, complete evaluation protocol, and all experimental configurations and parameter settings are publicly available at \url{https://github.com/JL181/FedGRC-Benchmark}.

\section{Related Work}\label{sec:related}

\subsection{Balancing Personalization and Generalization in Federated Learning}

PFL methods combine shared knowledge with local adaptation through parameter regularization, representation sharing, or predictor combination~\cite{PFL_Survey_TNNLS23,HFL_Survey_CSUR24}.
Ditto~\cite{Ditto_ICML21} regularizes personalized models toward a shared model to improve robustness and fairness, but requires additional model updates for personalization.
FedRep~\cite{FedRep_ICML21} separates shared representations from local heads; however, these heads still guide representation learning, allowing local label imbalance to influence the shared features.
FedROD~\cite{FedROD_ICLR22} combines a class-balanced generic classifier with a personalized residual classifier on a shared feature extractor.
This additive combination supports both generic recognition and local adaptation, but does not adjust the classifiers' relative contributions as the test distribution changes.

Test-time weighting and predictor selection offer alternatives to fixed predictor combinations.
FedTHE~\cite{FedTHE_ICLR23} optimizes a mixture weight using entropy and feature alignment informed by test history.
BTFL~\cite{BTFL_KDD25} instead infers weights from current and historical observations through Bayesian modeling, reducing iterative adaptation while relying on approximate feature likelihoods.
HPFL~\cite{HPFL_ICLR25} uses markers to select among shared personalized plug-ins, expanding the available predictor pool at the cost of storage and matching operations.
These mechanisms improve deployment flexibility, but their decisions depend on cues extracted from test observations.
When data sources or class distributions change, confidence and feature similarity may not reliably identify the most suitable predictor or combination.
FedGRC completes client adaptation before deployment and therefore requires neither test-time predictor selection nor mixture-weight estimation.

Training or selecting personalized predictors does not by itself establish whether each deployed model retains recognition across all federation-known classes.
Motley~\cite{Motley_arXiv22} shows that local fine-tuning can improve local accuracy while degrading performance on other clients' data.
Other studies evaluate personalized models on both local and global test sets~\cite{HtFLlib_KDD25,GPPFL_InfSci25,SemanticDFL_POMACS26,LocalVsOOD_arXiv25}, while broader benchmarks and fairness studies examine generalization and disparities across clients~\cite{FL_GRF_Survey_TPAMI24,pFLBench_NeurIPS22,PropFair_TMLR23,SuperquantileFL_ML24}.
Aggregate performance can nevertheless conceal poor recognition of locally rare or absent classes, as well as weaknesses concentrated in particular clients.
Our evaluation protocol assesses the same deployed client models on both local and all-class test distributions, explicitly evaluating locally rare and unseen classes.
Our protocol further measures each client's weakest assessed ability and summarizes the lowest of these client-level scores, so that stronger abilities or performance at other clients do not mask these weaknesses.

\subsection{Learning and Preserving Shared Representations under Data Heterogeneity}

Shared training signals differ in both the information they convey and the aspects of learning they constrain.
At the logit level, FedLC~\cite{FedLC_ICML22} uses class frequencies to calibrate predictions under label imbalance, but this correction does not by itself ensure that local updates preserve knowledge of locally absent classes.
Methods targeting locally absent classes address this retention problem directly~\cite{FedVLS_AAAI25,MMCFL_KBS26}.
FedVLS combines global-model distillation with logit suppression against minority-to-majority errors; distillation uses only teacher predictions on locally available inputs.

At the representation level, methods differ in the granularity of their shared references and the feature properties they constrain.
DBE~\cite{DBE_NeurIPS23} corrects client-specific bias through local bias memory and global mean regularization, but a single client-wide mean cannot fully characterize class-specific differences across clients.
GPFL~\cite{GPFL_ICCV23} combines angular guidance and Euclidean alignment to class references through additional feature transformations and auxiliary objectives along a personalized training path.
FedProto~\cite{FedProto_AAAI22} exchanges class prototypes for feature regularization, while FedFA~\cite{FedFA_TMC24} uses Euclidean anchors for feature alignment and classifier calibration.
Class-reference methods align feature centers or directions across clients~\cite{FedNH_AAAI23,FedProc_FGCS23,FedGP_InfSci26}.
Feature magnitudes have also been shaped through learned transformations~\cite{FedDream_InfSci26}, class-average norm regularization~\cite{FNRFL_IJCAI24}, and layerwise squared activation-norm penalties~\cite{MAN_WACV24}.
These representation constraints do not by themselves determine how a classifier should use within-class variation and feature correlations.
FedGRC uses final-layer directional guidance and shared magnitude references for each class and selected layer. Its multilayer statistical classifier uses common layer scales that preserve within-layer magnitude differences.

Fixed classifiers provide another way to establish shared feature geometry.
FedETF~\cite{FedETF_ICCV23} uses an equiangular tight-frame (ETF) classifier to impose common class geometry and prevent classifier drift.
Within this line of work, FedDr+~\cite{FedDrPlus_TMLR25} identifies missing-class forgetting under ordinary dot-regression and adds global feature distillation to improve retention, at the cost of computing teacher features.
Separately, FedYoYo~\cite{FedYoYo_ICCV25} combines self-distillation across weakly and strongly augmented views with distribution-aware logit adjustment to strengthen minority-class learning and generalization.
These advances improve shared representations, but preserving their recognition ability during local adaptation remains a separate design problem.
In our personalization experiments, local fine-tuning of FedETF improves local accuracy but substantially reduces all-class and locally unseen-class recognition (Section~\ref{sec:exp-personalization}).
FedGRC links representation learning to personalization through statistical classification based on shared multilayer features.
For FedGRC, the remaining design choices concern which classifier statistics to keep shared and how strongly each local class estimate should influence personalization.

\subsection{Federated Statistical Classifier Construction and Local Adaptation}

Statistical classifiers construct predictors from distributed feature summaries and have been used in one-shot and communication-efficient FL.
Their designs differ in how they construct decision rules, represent within-class variation, and estimate feature correlations.
CCVR~\cite{CCVR_NeurIPS21} calibrates a classifier using sampled virtual features and server-side optimization, while leaving the underlying representation fixed.
FedNCM~\cite{FedNCM_NeurIPS23} provides a simple mean-based readout, but does not account for class spread or feature correlations when comparing samples with class centers.
Fed3R~\cite{Fed3R_ICML24} uses second-order feature statistics to construct a closed-form ridge classifier whose solution is independent of how fixed features are partitioned across clients.
Exchanging second-order matrices increases communication cost relative to mean-only summaries and adds matrix-processing overhead.
FedCGS~\cite{FedCGS_AAAI25} constructs a Gaussian classifier from a full covariance centered at the global feature mean, without synthetic sampling or iterative head training.
GH-LDA~\cite{GHLDA_ICLR26} instead estimates shared within-class covariance by pooling residuals around class means.
Covariance estimation introduces a further distinction: FedCOF~\cite{FedCOF_NeurIPS25} reduces communication by estimating covariance from client means, but class-conditional shifts can affect the estimate and the resulting decision boundaries.
These methods differ in both the decision rule derived from shared statistics and the variation represented by those statistics.
In continual learning, LayUP~\cite{LayUP_TMLR24} constructs a Gram-matrix-based ridge readout, whereas FedGRC separates shared within-class covariance estimation from local class-mean adaptation.
A related design choice is how to scale multilayer features.
For OOD detection, MM++~\cite{MMpp_arXiv26} normalizes each sample's feature vector at each layer to unit length before estimating a shared within-class covariance from the concatenated features.
FedGRC instead applies a common scale to all samples at each layer, preserving within-layer magnitude differences for its statistical classifier.

Beyond constructing a shared classifier, statistical personalization must determine which quantities to adapt and how local sample support should influence that adaptation.
FedCGS~\cite{FedCGS_AAAI25} additionally supports personalization through prototype-regularized local fine-tuning.
FedLog~\cite{FedLog_TMLR26} personalizes local feature extractors and learns a shared head from sufficient statistics through iterative server-side maximum a posteriori (MAP) optimization.
Its adaptation therefore operates on local representations, whereas FedGRC keeps the representation shared and adapts class means.
pFedFDA~\cite{pFedFDA_NeurIPS24} guides representations with a global generative classifier and interpolates local and shared means and covariances.
Its interpolation coefficients are selected per client using local validation, rather than explicitly determined for each class from its sample count.
Its local class priors favor locally frequent classes, whereas FedGRC uses equal class priors to avoid adding this frequency preference to the prediction rule.
pFedGM~\cite{pFedGM_arXiv26} combines inter-class and intra-class representation objectives with count-informed distribution fusion, followed by local optimization of class means, scales, and biases.
This flexibility requires estimating several quantities from potentially sparse class observations and can increase the risk of local overfitting.
TOFA and SPEAR also use class counts to regulate mean shrinkage.
TOFA~\cite{TOFA_AAAI26} uses local and global class counts to shrink local means toward global estimates while estimating a client-specific covariance.
SPEAR~\cite{SPEAR_arXiv26} pools local means shrunk toward other clients' averages into a shared classifier.

FedGRC combines a shared multilayer covariance metric with client-specific class means adapted according to local class counts, using a reference strength fixed across clients and classes.
By keeping the representation and covariance metric shared, FedGRC limits personalization to class means and avoids estimating a separate local metric.
FedGRC retains shared means for locally absent classes, keeps estimates for sparsely observed classes close to the shared means, and permits stronger adaptation for well-observed classes.
Local counts govern mean adaptation; equal priors support all-class recognition without favoring locally frequent classes.

\section{Preliminaries}\label{sec:prelim}

\subsection{Personalized Federated Learning}
Consider $K$ clients and a common label space $\mathcal Y=\{1,\ldots,C\}$.
Client $k$ holds $\mathcal D^{(k)}=\{(x_i^{(k)},y_i^{(k)})\}_{i=1}^{N^{(k)}}$, with $N=\sum_kN^{(k)}$.
Parenthesized superscripts $(k)$ identify clients; $i$, $c$, $\ell$, and $t$ index samples, classes, layers, and communication rounds, respectively.
A notation reference is provided in Appendix~\ref{app:notation}.
For a model $g(\cdot;\vect\theta)$ and classification loss $\ell$, the local empirical loss is defined as
\begin{equation}\label{eq:local-objective}
 \mathcal L^{(k)}(\vect\theta)
 =\frac1{N^{(k)}}\sum_{i=1}^{N^{(k)}}
 \ell\!\left(g(x_i^{(k)};\vect\theta),y_i^{(k)}\right).
\end{equation}
Standard FL coordinates local updates to learn a shared model without pooling raw data, commonly by minimizing the sample-weighted objective
\begin{equation}\label{eq:fl-objective}
 \min_{\vect\theta}\sum_{k=1}^{K}\omega^{(k)}\mathcal L^{(k)}(\vect\theta),
 \;\omega^{(k)}=\frac{N^{(k)}}{N}.
\end{equation}

When client distributions differ, a single shared predictor may not fit every client well.
PFL therefore allows client-specific models while retaining collaboration through shared parameters, representations, or regularization~\cite{PFL_Survey_TNNLS23}.
A common representation-based formulation shares a feature extractor $h_{\vect\phi}$ and assigns each client a head $q_{\vect\psi^{(k)}}$~\cite{FedRep_ICML21}.
With $\vect\theta=(\vect\phi,\vect\psi)$ and $g(x;\vect\phi,\vect\psi)=q_{\vect\psi}(h_{\vect\phi}(x))$, the objective is
\begin{equation}\label{eq:pfl-objective}
 \min_{\vect\phi,\{\vect\psi^{(k)}\}_{k=1}^{K}}
 \sum_{k=1}^{K}\omega^{(k)}\mathcal L^{(k)}(\vect\phi,\vect\psi^{(k)}),
\end{equation}
where $\mathcal L^{(k)}(\vect\phi,\vect\psi)$ is the loss in Eq.~\eqref{eq:local-objective}, expressed in terms of the feature-extractor and head parameters.

We denote the deployed label predictor by $f^{(k)}(x)\in\mathcal Y$ and assess both local adaptation and broader recognition using this same predictor.
Shared representations provide common knowledge, but the reliability of a locally fitted head depends on the sample support of each class.
Constructing a head from class-level statistics allows shared and local evidence to be combined separately for each class.

\subsection{Statistical Classification}
Statistical classifiers construct decision rules from feature-distribution summaries~\cite{FedCGS_AAAI25,GHLDA_ICLR26}.
Let $\mathbf z\in\mathbb R^d$ be a feature vector, $\vect\mu_c$ the mean of class $c$, $\vect\Sigma\succ0$ a covariance matrix shared across classes, and $\pi_c>0$ the prior probability of class $c$.
Under a shared-covariance Gaussian model, Bayes' rule gives the posterior probability of class $c$ given $\mathbf z$:
\begin{equation}\label{eq:gaussian-posterior}
 p(y=c\mid\mathbf z)
 =\frac{\pi_c\mathcal N(\mathbf z;\vect\mu_c,\vect\Sigma)}
 {\sum_{r=1}^{C}\pi_r\mathcal N(\mathbf z;\vect\mu_r,\vect\Sigma)}.
\end{equation}
Here $\mathcal N(\mathbf z;\vect\mu_c,\vect\Sigma)$ is the multivariate Gaussian density evaluated at $\mathbf z$.
Taking logarithms and discarding terms common to all classes yields the linear discriminant~\cite{ESL_Springer09}
\begin{equation}\label{eq:statistical-score}
 s_c(\mathbf z)
 =\vect\mu_c^\top\vect\Sigma^{-1}\mathbf z
 -\tfrac12\vect\mu_c^\top\vect\Sigma^{-1}\vect\mu_c
 +\log\pi_c.
\end{equation}
The mean sets the class center, the inverse covariance defines the distance metric, and the log prior contributes an additive class-specific offset.
For $\mathbf M\succ0$, write $\|\mathbf v\|_{\mathbf M}^2=\mathbf v^\top\mathbf M\mathbf v$.
With equal priors, maximizing Eq.~\eqref{eq:statistical-score} is equivalent to minimizing $\|\mathbf z-\vect\mu_c\|_{\vect\Sigma^{-1}}^2$.

In a federation, shared statistics support classes that are rare or absent locally, while local statistics can capture client-specific structure.
The reliability of these estimates depends on sample support.
The central choices in Eq.~\eqref{eq:statistical-score} are which means, covariance, and priors to share or adapt, and how strongly local evidence should influence them.

\subsection{Problem Formulation}\label{sec:problem}
For each client and class, define the training subset and sample support as
\begin{equation}\label{eq:class-support}
\begin{aligned}
 \mathcal D_c^{(k)}&=\{(x,y)\in\mathcal D^{(k)}:y=c\},\\
 n_c^{(k)}&=|\mathcal D_c^{(k)}|,\;
 N_c=\sum_{k=1}^{K}n_c^{(k)}.
\end{aligned}
\end{equation}
The label space $\mathcal Y$ comprises all federation-known classes: $N_c>0$ for every $c$, although a client may have $n_c^{(k)}=0$.
The locally unseen set is $\mathcal U^{(k)}=\{c:n_c^{(k)}=0\}$.
Both label frequencies and class-conditional input distributions may differ across clients.
Every client predicts over all $C$ classes.

Let $P^{(k)}$ be client $k$'s local target distribution and $Q_c$ the distribution of inputs from class $c$ in the shared all-class evaluation population.
The class-balanced distribution $Q$ assigns probability $1/C$ to each class, with $X\mid Y=c\sim Q_c$.
We use $\Pr$, $\mathbb E$, and $\Var$ for probability, expectation, and variance, respectively, and $\mathbf1\{\cdot\}$ for the indicator function.
For any labeled target distribution $R$, write $\mathcal E_R(f)=\Pr_{(x,y)\sim R}[f(x)\ne y]$.
The class-wise error of the deployed predictor $f^{(k)}$ is
\begin{equation}\label{eq:class-risk}
 e_c^{(k)}=\mathbb E_{x\sim Q_c}
 \big[\mathbf1\{f^{(k)}(x)\ne c\}\big],
\end{equation}
and its local risk and global all-class risk are
\begin{subequations}\label{eq:recognition-risks}
\begin{align}
 \mathcal E_{\mathrm{local}}^{(k)}
 &=\mathbb E_{(x,y)\sim P^{(k)}}\!\left[\mathbf1\{f^{(k)}(x)\ne y\}\right],\label{eq:local-risk}\\
 \mathcal E_{\mathrm{global}}^{(k)}
 &=\mathcal E_Q(f^{(k)})=\frac1C\sum_{c=1}^{C}e_c^{(k)}.\label{eq:all-risk}
\end{align}
\end{subequations}
These risks assess the same client predictor under two distributions: the local risk is evaluated under $P^{(k)}$, while the global all-class risk gives equal weight to every federation-known class.
A low $\mathcal E_{\mathrm{local}}^{(k)}$ therefore need not imply a low $\mathcal E_{\mathrm{global}}^{(k)}$.
Even a low all-class risk can conceal large $e_c^{(k)}$ for locally rare or unseen classes, and averaging over clients can hide particularly weak client predictors.
Jointly assessing local adaptation, class coverage, and client-level weaknesses better characterizes each predictor's recognition ability.

Our research goal is to improve recognition across the full label space and strengthen the weakest assessed abilities of poorly performing clients while retaining effective local adaptation.
We assess these requirements jointly through local, all-class, rare-class, unseen-class, and client-floor metrics in Section~\ref{sec:eval-protocol}.
FedGRC pursues this goal through shared representation learning, a common statistical metric, and support-dependent mean adaptation.

\section{Method}\label{sec:method}

We present FedGRC, a unified PFL framework that integrates geometric representation consensus with statistical personalization for local adaptation and full-label-space recognition.
The framework comprises three successive stages:
\textbf{(i) Consensus-Guided Representation Learning (CGRL)} regularizes feature directions and magnitudes through shared class references;
\textbf{(ii) Multilayer Metric Synthesis (MMS)} constructs a shared statistical geometry from the frozen multilayer representations; and
\textbf{(iii) Support-Adaptive Shrinkage Personalization (SASP)} adapts class means according to local sample support within this geometry.
Together, these stages aim to retain recognition of locally rare or absent classes during personalization and strengthen recognition floors across clients.
Fig.~\ref{fig:overview} provides an overview of the framework; the complete procedure is summarized in Algorithm~\ref{alg:main} and detailed in the following subsections.

\begin{figure*}[!t]
  \centering
  \includegraphics[width=\textwidth,pagebox=cropbox,clip]{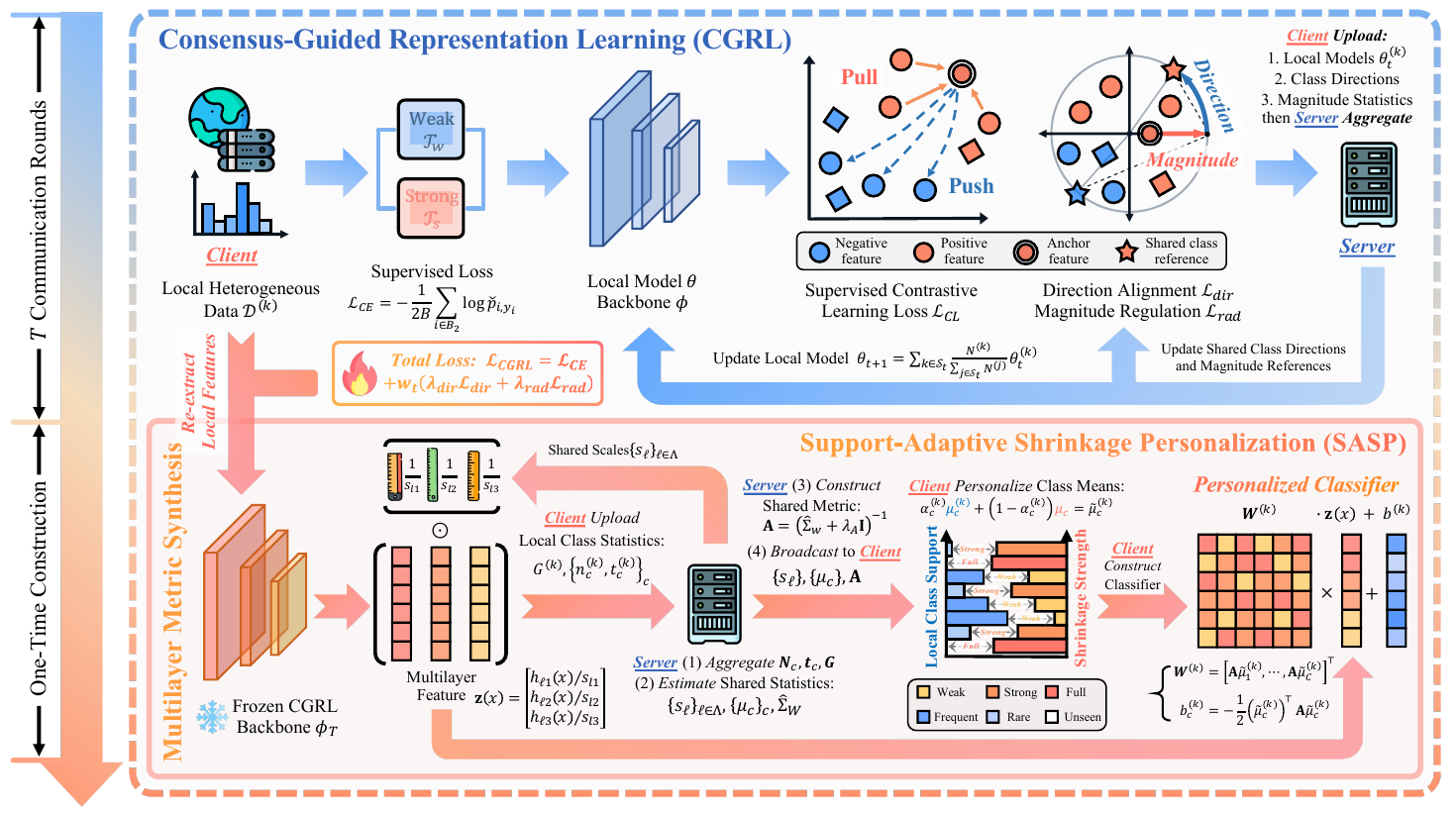}
  \caption{Overview of FedGRC. CGRL learns shared representations over $T$ communication rounds; MMS establishes a shared statistical basis in a single additional round, enabling SASP to construct personalized classifiers locally.}
  \label{fig:overview}
\end{figure*}

\begin{algorithm}[!t]
 \small\PrintSemicolon
 \SetAlCapFnt{\small\bfseries}\SetAlCapNameFnt{\small\bfseries}
 \SetKwInput{KwIn}{Input}\SetKwInput{KwOut}{Output}
 \SetKwComment{Phase}{\textcolor{black}{\(\blacktriangleright\)}\ }{}
 \SetCommentSty{textbf}
 \caption{FedGRC: Overall Procedure}
 \label{alg:main}
 \KwIn{Client data $\{\mathcal D^{(k)}\}_{k=1}^K$; layers $\Lambda$;
 Communication rounds $T$; local steps $E^{(k)}$;
 learning rate $\eta$;}
 \KwOut{$\vect\phi_T$, $\{s_\ell\}$, $\{\mathbf W^{(k)},\mathbf b^{(k)}\}_{k=1}^K$;}
 \vspace{2pt}\Phase{\normalsize (i) CGRL: Representation Learning}\vspace{1.2pt}
 Initialize client model $\vect\theta_0=(\vect\phi_0,\vect\psi_0)$\;
 Initialize shared references\;
 \For(\textcolor{gray}{\footnotesize // Communication round}){$t=0,\ldots,T-1$}{
  Sample participating clients $\mathcal S_t$\;
  \ForEach(\textcolor{gray}{\footnotesize // Client local training}){$k\in\mathcal S_t$ \textnormal{in parallel}}{
   Set local model $\vect\theta^{(k)}\leftarrow\vect\theta_t$\;
   Reset log-norm totals $H_{\ell,c}^{(k)}\leftarrow0,\ q_c^{(k)}\leftarrow0$\;
   \For(\textcolor{gray}{\footnotesize // Local updates}){$e=1,\ldots,E^{(k)}$}{
    Sample a mini-batch $\mathcal B\sim\mathcal D^{(k)}$\;
    Form paired views $\mathcal B_2\leftarrow[\mathcal T_{\mathrm w}(\mathcal B);\mathcal T_{\mathrm s}(\mathcal B)]$\;
    {\setlength{\jot}{0pt}\DontPrintSemicolon$\begin{aligned}[t]
     &\text{Compute CGRL total loss:}\\
     &\quad\mathcal L_{\mathrm{CGRL}}\leftarrow\mathcal L_{\mathrm{CE}}+w_t(\lambda_{\mathrm{dir}}\mathcal L_{\mathrm{dir}}+\lambda_{\mathrm{rad}}\mathcal L_{\mathrm{rad}});
    \end{aligned}$\;}
    Collect log-norm sums/counts $H_{\ell,c}^{(k)},q_c^{(k)}$\;
    Update $\vect\theta^{(k)}\leftarrow\vect\theta^{(k)}-\eta\nabla_{\vect\theta^{(k)}}\mathcal L_{\mathrm{CGRL}}$\;
   }
   Estimate class directions $\bar{\mathbf u}_c^{(k)}$ via Eq.~\eqref{eq:direction-update}\;
   Upload $\{\vect\theta^{(k)},\bar{\mathbf u}_c^{(k)},n_c^{(k)},H_{\ell,c}^{(k)},q_c^{(k)}\}$ to \underline{\textbf{\textit{Server}}}\;
  }
  The \underline{\textbf{\textit{Server}}} aggregates local models into $\vect\theta_{t+1}$ via Eq.~\eqref{eq:model-aggregation}, updates references via Eqs.~\eqref{eq:direction-ema}, \eqref{eq:radial-update}, then broadcasts both to \underline{\textbf{\textit{Clients}}}\;
 }
 \vspace{2pt}\Phase{\normalsize (ii) MMS: Shared Metric Synthesis}\vspace{1.2pt}
 The \underline{\textbf{\textit{Clients}}} freeze backbone $\vect\phi_T$ from trained model $\vect\theta_T$\;
 \ForEach(\textcolor{gray}{\footnotesize // Local statistics}){$k=1,\ldots,K$ \textnormal{in parallel}}{
  Extract and concatenate layer features $\mathbf h_i^{(k)}$\;
  Count class samples $n_c^{(k)}\leftarrow|\mathcal D_c^{(k)}|$\;
  {\setlength{\jot}{0pt}\DontPrintSemicolon$\begin{aligned}[t]
   &\text{Form class feature sums and outer-product sum:}\\
   &\quad\mathbf t_c^{(k)}\leftarrow\textstyle\sum_{i:y_i^{(k)}=c}\mathbf h_i^{(k)},\ \mathbf G^{(k)}\leftarrow\textstyle\sum_i\mathbf h_i^{(k)}(\mathbf h_i^{(k)})^\top;
  \end{aligned}$\;}
  Upload $\{\{n_c^{(k)},\mathbf t_c^{(k)}\}_c,\mathbf G^{(k)}\}$ to \underline{\textbf{\textit{Server}}}\;
 }
 {\setlength{\jot}{0pt}\DontPrintSemicolon$\begin{aligned}[t]
  &\text{The \underline{\textbf{\textit{Server}}} aggregates client statistics:}\\
  &\quad N_c\leftarrow\textstyle\sum_k n_c^{(k)},\ \mathbf t_c\leftarrow\textstyle\sum_k\mathbf t_c^{(k)},\ \mathbf G\leftarrow\textstyle\sum_k\mathbf G^{(k)};
 \end{aligned}$\;}\vspace{1.5pt}
 Compute shared layer scales $s_\ell$ via Eqs.~\eqref{eq:pooled-feature-covariance}, \eqref{eq:layer-variance}\;
 Derive shared statistics $\vect\mu_c$, $\widehat{\vect\Sigma}_{\mathrm w}$ via Eq.~\eqref{eq:statistics-from-sums}\;
 {\setlength{\jot}{0pt}\DontPrintSemicolon$\begin{aligned}[t]
  &\text{Construct regularized shared metric:}\\
  &\quad\lambda_A\leftarrow\gamma\tr(\widehat{\vect\Sigma}_{\mathrm w})/d,\ \mathbf A\leftarrow(\widehat{\vect\Sigma}_{\mathrm w}+\lambda_A\mathbf I_d)^{-1};
 \end{aligned}$\;}
 Broadcast $\bigl\{\{s_\ell\}_{\ell\in\Lambda},\{\vect\mu_c\}_c,\mathbf A\bigr\}$ to \underline{\textbf{\textit{Clients}}}\;
 \vspace{2pt}\Phase{\normalsize (iii) SASP: Support-Adaptive Shrinkage}\vspace{1.2pt}
 \ForEach(\textcolor{gray}{\footnotesize // Client personalization}){$k=1,\ldots,K$}{
  Compute observed-class means $\vect\mu_c^{(k)}$ via Eq.~\eqref{eq:statistics-from-sums}\;
  \ForEach(\textcolor{gray}{\footnotesize // Class adaptation}){$c=1,\ldots,C$}{
   Compute mean shift $\vect\Delta_c^{(k)}$ via Eq.~\eqref{eq:local-displacement}\;
   {\setlength{\jot}{0pt}\DontPrintSemicolon$\begin{aligned}[t]
    &\text{Personalize class mean using local support:}\\
    &\quad\widetilde{\vect\mu}_c^{(k)}\leftarrow\vect\mu_c+[n_c^{(k)}/(n_c^{(k)}+m_0)]\vect\Delta_c^{(k)};
   \end{aligned}$\;}
   Set classifier weights $\mathbf w_c^{(k)}\leftarrow\mathbf A\widetilde{\vect\mu}_c^{(k)}$\;
   Set classifier bias $b_c^{(k)}\leftarrow-1/2(\widetilde{\vect\mu}_c^{(k)})^\top\mathbf w_c^{(k)}$\;
  }
 }
 \Return{$\vect\phi_T,\{s_\ell\},\{\mathbf W^{(k)},\mathbf b^{(k)}\}_{k=1}^K$}\;
\end{algorithm}

\subsection{Consensus-Guided Representation Learning}\label{sec:cgrl}
CGRL combines cross-entropy supervision with shared directional and magnitude guidance to learn representations for statistical classification.
At each round, clients receive the shared model and class references and jointly optimize a backbone and a training-only classifier, parameterized by $\vect\theta=(\vect\phi,\vect\psi)$.
Label-preserving weak and strong augmentations, $\mathcal T_{\mathrm w}$ and $\mathcal T_{\mathrm s}$, produce $2B$ views from a mini-batch of $B$ examples.
The semicolon between augmented batches in Algorithm~\ref{alg:main} denotes batch concatenation.
Let $\mathcal B_2$ index these views and $\check p_{i,c}$ denote the predicted probability of class $c$ for view $i$.
The cross-entropy loss over both augmented views:
\begin{equation}\label{eq:cgrl-ce}
 \mathcal L_{\mathrm{CE}}=-\frac1{2B}\sum_{i\in\mathcal B_2}\log\check p_{i,y_i}.
\end{equation}

Let $\Lambda$ be the set of selected layers and $\mathbf h_\ell(x;\vect\phi)\in\mathbb R^{d_\ell}$ the pooled representation at layer $\ell\in\Lambda$, where $d_\ell$ is its feature dimension and $L\in\Lambda$ denotes the final backbone layer.
Statistical classification depends on feature magnitudes as well as directions, since both determine class means and within-class scatter.
CGRL therefore combines directional supervision at the final layer with magnitude regularization across all layers used by MMS.

For each of two features, let $R_i=\|\mathbf h_i\|_2>0$ and $\mathbf u_i=\mathbf h_i/R_i$. Their squared distance decomposes as
\begin{equation}\label{eq:angular-radial-distance}
 \|\mathbf h_i-\mathbf h_j\|_2^2
 =(R_i-R_j)^2+R_iR_j\|\mathbf u_i-\mathbf u_j\|_2^2.
\end{equation}
This decomposition motivates regulating both magnitude and direction: differences in either contribute to feature dispersion.
Within a local update, client and round indices are omitted when unambiguous.

We define the two geometric objectives and then describe their joint optimization and federated reference updates.

\textbf{(1) Shared Directional Guidance.}
Directional guidance combines batch-level discrimination with shared class references.
Let $\mathbf u_i=\unit{\mathbf h_L(x_i;\vect\phi)}$, with $\operatorname{unit}(\mathbf v)=\mathbf v/\|\mathbf v\|_2$ for $\mathbf v\ne\mathbf0$ and $\operatorname{unit}(\mathbf0)=\mathbf0$.
For $\mathcal P_i=\{j\in\mathcal B_2:j\ne i,\ y_j=y_i\}$, we use the supervised contrastive loss $\mathcal L_{\mathrm{CL}}$~\cite{SupCon_NeurIPS20} to promote within-class alignment and between-class separation:
\begin{equation}\label{eq:supcon}
 \mathcal L_{\mathrm{CL}}
 =-\frac1{2B}\sum_i\frac1{|\mathcal P_i|}
 \sum_{j\in\mathcal P_i}\log\frac{\exp(\mathbf u_i^\top\mathbf u_j/\tau)}{Z_i},
\end{equation}
where $Z_i=\sum_{j\in\mathcal B_2:j\ne i}\exp(\mathbf u_i^\top\mathbf u_j/\tau)$ and $\tau>0$ is the temperature.
The second view of each example ensures $|\mathcal P_i|>0$.
Local batches may omit classes, and their comparisons do not provide common directional references across clients.
We therefore introduce a prototype cross-entropy loss using shared class directions $\mathbf p_c(t)$:
\begin{equation}\label{eq:prototype-loss}
 \mathcal L_{\mathrm{proto}}
 =-\frac1{2B}\sum_{i\in\mathcal B_2}
 \log\frac{\exp(\mathbf u_i^\top\mathbf p_{y_i}(t)/\tau)}
 {\sum_{c=1}^{C}\exp(\mathbf u_i^\top\mathbf p_c(t)/\tau)}.
\end{equation}
This loss anchors local features to a common class geometry and includes references for classes absent from the current batch.
The directional objective combines batch-level discrimination with shared class alignment weighted by $\beta$:
\begin{equation}\label{eq:direction-loss}
 \mathcal L_{\mathrm{dir}}=\mathcal L_{\mathrm{CL}}+\beta\mathcal L_{\mathrm{proto}}.
\end{equation}

\textbf{(2) Magnitude Consensus through Shared References.}
Directional alignment does not directly constrain feature magnitudes, which also contribute to the statistics used by MMS.
We introduce a shared log-norm reference $a_{\ell,c}(t)$ for each class and layer, allowing their characteristic scales to differ.
Unlike the common radius of Ring loss~\cite{RingLoss_CVPR18}, separate references allow class- and layer-specific scales.
For $\ell\in\Lambda$ and a numerical stability constant $\varepsilon>0$, define the stabilized norm and its logarithm as
\begin{equation}\label{eq:log-norm}
 r_\ell(x)=\sqrt{\|\mathbf h_\ell(x;\vect\phi)\|_2^2+\varepsilon^2},
 \; \rho_\ell(x)=\log r_\ell(x).
\end{equation}
The identity $\rho_\ell-a_{\ell,c}=\log(r_\ell/\exp(a_{\ell,c}))$ shows that the penalty measures relative magnitude deviations.
Magnitude alignment is disabled until a shared reference is established, and the magnitude penalty is omitted for features with very small norms.
To enforce these conditions, let $v_c(t)\in\{0,1\}$ indicate whether the shared magnitude references for class $c$ have been initialized in an earlier round, and define the binary indicator $\chi_{\ell i}=v_{y_i}(t)\mathbf1\{r_\ell(x_i)\ge r_{\min}\}$, with safe-norm threshold $r_{\min}>0$.
The magnitude loss is
\begin{equation}\label{eq:radial-loss}
 \mathcal L_{\mathrm{rad}}
 =\frac1{2B|\Lambda|}\sum_{\ell\in\Lambda}\sum_i
 \chi_{\ell i}\big[\rho_\ell(x_i)-a_{\ell,y_i}(t)\big]^2.
\end{equation}
The indicator $\chi_{\ell i}$ is held fixed during differentiation and affects only the magnitude penalty.
Lemma~\ref{lem:radial} characterizes the penalty by separating within-class log-norm dispersion from the mean offset relative to the shared reference.
\begin{lemma}[Magnitude Consensus]\label{lem:radial}
Fix a class $c$, layer $\ell$, and round $t$, and treat $a=a_{\ell,c}(t)$ as constant.
Let $\rho$ and $\chi$ denote the corresponding log norm and binary indicator for samples of class $c$.
Let $p_\chi=\Pr(\chi=1)>0$ and $\bar\rho_\chi=\mathbb E[\rho\mid\chi=1]$, with $\mathbb E[\rho^2\mid\chi=1]<\infty$.
Then
\begin{equation}\label{eq:radial-decomposition}
 \mathbb E[\chi(\rho-a)^2]
 =p_\chi\!\Bigl[\underbrace{\Var(\rho\mid\chi=1)}_{\text{\normalfont dispersion}}
 +\underbrace{(\bar\rho_\chi-a)^2}_{\text{\normalfont reference offset}}\Bigr].
\end{equation}
\end{lemma}
\begin{proof}
The proof is given in Appendix~\ref{app:radial}.
\end{proof}
Thus, among samples with $\chi=1$, the penalty regularizes log-norm dispersion and mean deviation from shared references, whose scales may vary across classes and layers.

\textbf{(3) Local Objective and Federated Reference Updates.}
The local objective augments cross-entropy with directional and magnitude regularization:
\begin{equation}\label{eq:cgrl-objective}
 \mathcal L_{\mathrm{CGRL}}=\mathcal L_{\mathrm{CE}}
 +w_t\Bigl(\bluepart{\lambda_{\mathrm{dir}}\mathcal L_{\mathrm{dir}}}{directional\\consensus}
 +\orangepart{\lambda_{\mathrm{rad}}\mathcal L_{\mathrm{rad}}}{magnitude\\consensus}\Bigr).
\end{equation}
The weights $\lambda_{\mathrm{dir}}$ and $\lambda_{\mathrm{rad}}$ control the two regularizers.
The schedule $w_t$ gradually introduces geometric guidance as the shared references are established.
Local models are aggregated by client sample size; after training, the classifier is replaced by the statistical head.

Shared references remain fixed during local optimization and are updated between communication rounds.
Directional references aggregate post-update class directions with class-count weights, whereas magnitude references aggregate online log-norm observations.
Temporal averaging moderates reference fluctuations across rounds.
Appendix~\ref{app:updates} details the reference estimators, update conventions, and warm-up schedule.

\textbf{Complementary Geometry for Statistical Classification.}
We first analyze how the two objectives regulate feature directions and magnitudes, then connect these effects to the within-class statistics used by MMS.

Fix the references in one local step and write $\mathbf h_i=\mathbf h_L(x_i;\vect\phi)\ne\mathbf0$.
Let $q_c(\mathbf u_i)$ be the softmax probability of class $c$ from logits $\{\mathbf u_i^\top\mathbf p_r(t)/\tau\}_{r=1}^C$, and let $\ell_{\mathrm p,i}=-\log q_{y_i}(\mathbf u_i)$.
Thus, $\mathcal L_{\mathrm{proto}}=(2B)^{-1}\sum_{i\in\mathcal B_2}\ell_{\mathrm p,i}$.
Differentiation through unit normalization gives
\begin{subequations}\label{eq:direction-gradients}
\begin{align}
 \frac{\partial\mathbf u_i}{\partial\mathbf h_i}
 &=\frac{\mathbf I-\mathbf u_i\mathbf u_i^\top}{\|\mathbf h_i\|_2},\label{eq:unit-jacobian}\\
 \nabla_{\mathbf h_i}\ell_{\mathrm p,i}
 &=\frac{\mathbf I-\mathbf u_i\mathbf u_i^\top}{\tau\|\mathbf h_i\|_2}
 \Bigl(\sum_{c=1}^Cq_c(\mathbf u_i)\mathbf p_c(t)-\mathbf p_{y_i}(t)\Bigr).
 \label{eq:prototype-gradient}
\end{align}
\end{subequations}
The projection $\mathbf I-\mathbf u_i\mathbf u_i^\top$ removes the radial component, leaving a gradient tangent to the sphere at $\mathbf h_i$. Both directional terms depend on raw features through unit normalization and share this tangency property.

Abbreviate $\mathbf h=\mathbf h_\ell(x_i;\vect\phi)$ and $\ell_{\mathrm r}=\chi(\rho-a)^2$.
With the reference and indicator fixed, the magnitude penalty has the feature gradient
\begin{equation}\label{eq:radial-gradient}
 \nabla_{\mathbf h}\ell_{\mathrm r}
 =\frac{2\chi(\rho-a)}{\|\mathbf h\|_2^2+\varepsilon^2}\,\mathbf h.
\end{equation}
When $\chi=1$, the gradient is radial, inducing contraction above the reference and expansion below it under a sufficiently small feature-space descent step.
At the final layer, write $\mathbf g_i=\nabla_{\mathbf h_i}\mathcal L_{\mathrm{dir}}$ and $\mathbf v_i=\nabla_{\mathbf h_i}\ell_{\mathrm r,i}$.
The tangential and radial gradients satisfy
\begin{equation}\label{eq:feature-complementarity}
 \mathbf h_i^\top\mathbf g_i=0,\;\mathbf g_i^\top\mathbf v_i=0.
\end{equation}
Directional supervision therefore preserves magnitude to first order in feature space, while the radial term directly regulates it.
This complementarity concerns feature-space gradients; both objectives act through shared backbone parameters, whose gradients need not remain orthogonal.

To relate these feature-space mechanisms to the statistics used by MMS, fix a layer $\ell$ and class $c$ and pool its $N_c$ features across clients under a common backbone.
Abbreviate these features as $\mathbf h_i=R_i\mathbf u_i$, with mean $\bar{\mathbf h}_{\ell,c}$ and covariance $\widehat{\vect\Sigma}_{\ell,c}$ defined using denominator $N_c$.
Writing $\tr$ for the matrix trace, we expand squared distances about the mean and apply Eq.~\eqref{eq:angular-radial-distance} to obtain
\begin{equation}\label{eq:pairwise-scatter}
\begin{aligned}
 \tr(\widehat{\vect\Sigma}_{\ell,c})
 &=\frac1{N_c}\sum_{i=1}^{N_c}\|\mathbf h_i-\bar{\mathbf h}_{\ell,c}\|_2^2\\
 &=\frac1{2N_c^2}\sum_{i,j=1}^{N_c}\|\mathbf h_i-\mathbf h_j\|_2^2\\
 &=\underbrace{\frac1{2N_c^2}\sum_{i,j=1}^{N_c}(R_i-R_j)^2}_{\text{\normalfont\scriptsize radial scatter}}\\
 &\quad+\underbrace{\frac1{2N_c^2}\sum_{i,j=1}^{N_c}R_iR_j\|\mathbf u_i-\mathbf u_j\|_2^2}_{\text{\normalfont\scriptsize magnitude-weighted angular scatter}}.
\end{aligned}
\end{equation}
Radial scatter can persist even when directions are aligned, and feature magnitudes also weight the angular contribution.
This identity links the geometry regulated by final-layer directional supervision and layerwise magnitude guidance to within-class covariance estimation.
It explains the statistical relevance of both forms of guidance; CGRL optimizes their respective objectives rather than directly minimizing the covariance trace.

Appendix~\ref{app:gradients} details the derivation and numerical normalization conventions. We then freeze the final aggregated backbone to build the statistical classifier.

\subsection{Multilayer Metric Synthesis}\label{sec:mms}
MMS constructs a shared statistical metric from the frozen multilayer representations.
Integrating multiple depths adds information beyond the final layer; pooling class statistics across clients avoids separate covariance estimation from limited local observations.
MMS recomputes all statistics from features extracted by the final aggregated backbone $\vect\phi_T$, ensuring that statistics from different clients describe the same feature space.

\textbf{(1) Layer Scaling and Feature Composition.}
A single scalar per layer balances average coordinate variance across layers while preserving within-layer norm ratios.
With $\bar{\mathbf h}_\ell=N^{-1}\sum_{k,i}\mathbf h_\ell(x_i^{(k)};\vect\phi_T)$, define the shared layer scale as
\begin{equation}\label{eq:layer-variance}
 s_\ell^2=\frac1{Nd_\ell}\sum_{k,i}
 \big\|\mathbf h_\ell(x_i^{(k)};\vect\phi_T)-\bar{\mathbf h}_\ell\big\|_2^2.
\end{equation}
For a fixed ordering $(\ell_1,\ldots,\ell_m)$ of $\Lambda$, with $m=|\Lambda|$, we concatenate the scaled features as follows:
\begin{equation}\label{eq:mms-scale}
 \mathbf z(x)=\begin{bmatrix}
  \mathbf h_{\ell_1}(x;\vect\phi_T)/s_{\ell_1}\\
  \vdots\\
  \mathbf h_{\ell_m}(x;\vect\phi_T)/s_{\ell_m}
 \end{bmatrix}\in\mathbb R^d,
 \quad d=\sum_{\ell\in\Lambda}d_\ell.
\end{equation}
The scales $s_\ell$ are estimated from globally centered features, while the concatenated features are scaled without centering.
For $s_\ell>0$, each scaled block has unit average coordinate variance, establishing comparable numerical scales for covariance regularization.
Uniform scaling within each layer preserves norm ratios across samples, retaining the magnitude structure shaped by CGRL.

\textbf{(2) Federated Estimation of Class Statistics.}
Let $\mathcal I_c^{(k)}=\{i:y_i^{(k)}=c\}$.
Shared class means provide references when local observations are sparse or absent, while local means capture client-specific deviations for personalization:
\begin{subequations}\label{eq:mms-means}
\begin{align}
 \vect\mu_c^{(k)}&=\frac1{n_c^{(k)}}\sum_{i\in\mathcal I_c^{(k)}}\mathbf z(x_i^{(k)}),
 \; n_c^{(k)}>0,\label{eq:local-mean}\\
 \vect\mu_c&=\frac1{N_c}\sum_{k:n_c^{(k)}>0}n_c^{(k)}\vect\mu_c^{(k)}.\label{eq:global-mean}
\end{align}
\end{subequations}
To compare deviations from class centers, MMS uses a covariance-based metric that accounts for within-class variability and feature correlations.
MMS estimates within-class covariance relative to shared class means, capturing both within-client variability and cross-client mean shifts.
With residuals $\mathbf e_i^{(k)}=\mathbf z(x_i^{(k)})-\vect\mu_{y_i^{(k)}}$, the pooled covariance is
\begin{equation}\label{eq:mms-covariance}
 \widehat{\vect\Sigma}_{\mathrm w}
 =\frac1N\sum_{k,i}\mathbf e_i^{(k)}(\mathbf e_i^{(k)})^\top.
\end{equation}
Proposition~\ref{prop:scatter} decomposes this covariance into within-client variation and between-client mean shifts.
\begin{proposition}[Within- and Between-Client Covariance Decomposition]\label{prop:scatter}
For every nonempty client--class group, let $\widehat{\vect\Sigma}_c^{(k)}$ be its covariance about $\vect\mu_c^{(k)}$, computed with denominator $n_c^{(k)}$, and let $\vect\Delta_c^{(k)}=\vect\mu_c^{(k)}-\vect\mu_c$.
For $\mathcal J=\{(k,c):n_c^{(k)}>0\}$,
\begin{equation}\label{eq:scatter-decomposition}
 \widehat{\vect\Sigma}_{\mathrm w}
 =\sum_{(k,c)\in\mathcal J}\frac{n_c^{(k)}}N
 \Bigl(\underbrace{\widehat{\vect\Sigma}_c^{(k)}}_{\text{\normalfont\scriptsize\shortstack{within-client\\variation}}}
 +\underbrace{\vect\Delta_c^{(k)}(\vect\Delta_c^{(k)})^\top}_{\text{\normalfont\scriptsize\shortstack{between-client\\variation}}}\Bigr).
\end{equation}
\end{proposition}
\begin{proof}
The proof is given in Appendix~\ref{app:scatter}.
\end{proof}

Pooling covariances computed about client-specific class means would omit the between-client mean-shift term.
MMS retains this variation in the shared covariance, while SASP adapts class means along the same displacements according to local support.

The concatenated representation also allows MMS to retain within-class cross-layer covariance.
Partition $\mathbf e_i^{(k)}$ into layer blocks $\mathbf e_{\ell i}^{(k)}\in\mathbb R^{d_\ell}$.
The $(\ell,\ell')$ block is
\begin{equation}\label{eq:cross-layer-covariance}
 [\widehat{\vect\Sigma}_{\mathrm w}]_{\ell\ell'}
 =\frac1N\sum_{k,i}\mathbf e_{\ell i}^{(k)}(\mathbf e_{\ell' i}^{(k)})^\top,
 \;\ell,\ell'\in\Lambda.
\end{equation}
The off-diagonal blocks retain cross-covariances between layer-wise residuals, enabling the shared metric to account for their joint variation.

The total within-class variation connects directly to the feature geometry analyzed in CGRL.
With $\widehat{\vect\Sigma}_{\ell,c}$ in Eq.~\eqref{eq:pairwise-scatter} recomputed under $\vect\phi_T$, shared layer scaling gives
\begin{equation}\label{eq:mms-scatter-trace}
 \tr(\widehat{\vect\Sigma}_{\mathrm w})
 =\sum_{c=1}^{C}\frac{N_c}{N}\sum_{\ell\in\Lambda}
 \frac{\tr(\widehat{\vect\Sigma}_{\ell,c})}{s_\ell^2}.
\end{equation}
The radial--angular decomposition links MMS's total within-class variation to radial and magnitude-weighted angular scatter, each scaled by $s_\ell^{-2}$ within layer $\ell$.

Appendix~\ref{app:moments} derives the layer scales, class means, and pooled covariance exactly from additive client counts and first- and second-order feature moments.

\textbf{(3) Regularized Shared Metric Construction.}
The estimated within-class covariance may be singular or ill-conditioned, preventing direct inversion or making class comparisons sensitive to estimation errors.
Since all personalized heads use the resulting metric, we add an isotropic ridge proportional to the estimated covariance's average variance before inversion~\cite{RDA_JASA89}:
\begin{equation}\label{eq:mms-metric}
 \lambda_A=\gamma\frac{\tr(\widehat{\vect\Sigma}_{\mathrm w})}{d},
 \;
 \mathbf A=(\widehat{\vect\Sigma}_{\mathrm w}+\lambda_A\mathbf I_d)^{-1}.
\end{equation}
Here $\tr(\widehat{\vect\Sigma}_{\mathrm w})/d$ equals the average coordinate variance, or equivalently the mean covariance eigenvalue.
Thus, $\gamma>0$ controls the relative regularization strength.
For positive within-class trace, $\lambda_A>0$ makes the regularized covariance positive definite, ensuring a well-defined positive-definite metric $\mathbf A$.
The shared metric $\mathbf A$ replaces $\vect\Sigma^{-1}$ in Eq.~\eqref{eq:statistical-score}.

\textbf{Spectral Interpretation and Conditioning.}
For the eigendecomposition $\widehat{\vect\Sigma}_{\mathrm w}=\mathbf Q\diag(\sigma_1,\ldots,\sigma_d)\mathbf Q^\top$, where $\mathbf Q$ contains orthonormal eigenvectors and $\sigma_j$ are the eigenvalues, the squared distance to a class center $\vect\mu$ is
\begin{equation}\label{eq:spectral-metric}
 \|\mathbf z-\vect\mu\|_{\mathbf A}^2
 =\sum_{j=1}^{d}\frac{[\mathbf Q^\top(\mathbf z-\vect\mu)]_j^2}{\sigma_j+\lambda_A}.
\end{equation}
The metric downweights deviations along directions with large within-class variance, while the ridge bounds sensitivity along directions with small variance.

For extreme covariance eigenvalues $\sigma_{\min}$ and $\sigma_{\max}$, the metric's operator norm and spectral condition number quantify its maximum amplification and conditioning:
\begin{subequations}\label{eq:metric-conditioning}
\begin{align}
 \|\mathbf A\|_2
 &=\frac1{\sigma_{\min}+\lambda_A}\le\frac1{\lambda_A},\label{eq:metric-gain}\\
 \kappa_2(\mathbf A)
 &=\frac{\sigma_{\max}+\lambda_A}{\sigma_{\min}+\lambda_A}
 \le 1+\frac d\gamma.\label{eq:metric-condition-number}
\end{align}
\end{subequations}
The ridge bounds the maximum amplification by $1/\lambda_A$ and the condition number by $1+d/\gamma$, including when the estimated within-class covariance is singular.
Appendix~\ref{app:conditioning} gives the derivation.

\subsection{Support-Adaptive Shrinkage Personalization}\label{sec:sasp}
Using the representation learned by CGRL and the class statistics and metric constructed by MMS, SASP builds a personalized classifier for each client.
It first adapts class means according to local support, then combines these means with the shared metric to construct a classifier over all federation-known classes.
This design targets client-specific adaptation while retaining shared references for classes with limited or no local support.
The backbone, layer scales, and covariance metric remain shared, avoiding covariance estimation from sparse local observations.

\textbf{(1) Support-Dependent Mean Adaptation.}
Local class means capture client-specific shifts, but estimates based on few samples can be unreliable.
SASP therefore shrinks each local mean toward its shared counterpart according to class support.
We define the local displacement as follows, setting it to zero for locally absent classes:
\begin{equation}\label{eq:local-displacement}
 \vect\Delta_c^{(k)}=
 \begin{cases}
 \vect\mu_c^{(k)}-\vect\mu_c,& n_c^{(k)}>0,\\
 \mathbf0,& n_c^{(k)}=0.
 \end{cases}
\end{equation}
For a reference strength $m_0>0$ shared across clients and classes, the personalized mean is defined by
\begin{equation}\label{eq:sasp-adaptation}
 \newcommand{\supportfraction}{\frac{n_c^{(k)}}{n_c^{(k)}+m_0}}
 \newcommand{\supportannotation}{\raisebox{2pt}{$\displaystyle\overbrace{\phantom{\supportfraction}}^{\alpha_c^{(k)}}$}}
 \widetilde{\vect\mu}_c^{(k)}
 =\bluepart{\vect\mu_c}{shared center}
 +\vphantom{\supportannotation}
 \orangepart{\rlap{\smash{\supportannotation}}\supportfraction\vect\Delta_c^{(k)}}{support-weighted local shift}.
\end{equation}
The adaptation weight is defined as $\alpha(n)=n/(n+m_0)$, where $n$ denotes local class support.
The corresponding shared weight is $1-\alpha(n)$.
The local weight increases with class support, while a larger $m_0$ increases the shared weight.

\textbf{Regularized Estimation and Geometric Interpretation.}
We justify the adaptation rule as a regularized mean estimator that balances local observations against the shared reference.
Proposition~\ref{prop:shrinkage} shows that this balance yields exactly the support-dependent update in Eq.~\eqref{eq:sasp-adaptation}.
\begin{proposition}[Support-Adaptive Mean Estimation]\label{prop:shrinkage}
For fixed $\mathbf A\succ0$, $\vect\mu_c$, and $m_0>0$, the unique minimizer over the candidate class mean $\mathbf m\in\mathbb R^d$ of
\begin{equation}\label{eq:shrinkage-objective}
 J_c^{(k)}(\mathbf m)
 =\sum_{i\in\mathcal I_c^{(k)}}\underbrace{\|\mathbf z(x_i^{(k)})-\mathbf m\|_{\mathbf A}^2}_{\text{\normalfont local evidence}}
 +\underbrace{m_0\|\mathbf m-\vect\mu_c\|_{\mathbf A}^2}_{\text{\normalfont shared reference}}
\end{equation}
 is $\widetilde{\vect\mu}_c^{(k)}$ in Eq.~\eqref{eq:sasp-adaptation}, including $n_c^{(k)}=0$ with an empty data sum.
\end{proposition}
\begin{proof}
The proof is given in Appendix~\ref{app:shrinkage}.
\end{proof}

The data term scales with class support, while $m_0$ gives the shared mean a fixed effective sample weight~\cite{Murphy_Gaussian07}.
The common positive-definite metric $\mathbf A$ cancels from the optimality condition; it determines classification geometry without altering the estimated mean.

The following identities quantify how closely the adapted mean fits the local center and how far it departs from the shared reference.
For an observed class, its distance to the local center satisfies
\begin{equation}\label{eq:local-fit}
 \|\widetilde{\vect\mu}_c^{(k)}-\vect\mu_c^{(k)}\|_{\mathbf A}
 =(1-\alpha_c^{(k)})\|\vect\mu_c-\vect\mu_c^{(k)}\|_{\mathbf A}.
\end{equation}
The departure from the shared center is $\|\widetilde{\vect\mu}_c^{(k)}-\vect\mu_c\|_{\mathbf A}=\alpha_c^{(k)}\|\vect\Delta_c^{(k)}\|_{\mathbf A}$.
For fixed local and shared means, greater support reduces the residual to the local center while permitting a larger displacement from the shared center.
This displacement determines the score perturbations examined in the stability analysis below.

\textbf{(2) Personalized Classifier Construction.}
The adapted means and the shared MMS metric define a client-specific classifier over all federation-known classes.
SASP uses class counts to control mean adaptation and equal priors to avoid an additional preference for locally frequent classes.
This choice is consistent with the equal class weighting of $\mathcal E_{\mathrm{global}}^{(k)}$; an unsmoothed local frequency prior would also exclude locally absent classes.
With $\pi_c=1/C$, the log-prior term is constant across classes.
Substituting the personalized means and shared metric into Eq.~\eqref{eq:statistical-score} yields
\begin{subequations}\label{eq:sasp-head}
\begin{align}
 \mathbf w_c^{(k)}&=\mathbf A\widetilde{\vect\mu}_c^{(k)},\;
 b_c^{(k)}=-\tfrac12(\widetilde{\vect\mu}_c^{(k)})^\top\mathbf w_c^{(k)},\label{eq:head-parameters}\\
 s_c^{(k)}(\mathbf z)&=(\mathbf w_c^{(k)})^\top\mathbf z+b_c^{(k)},\label{eq:head-score}\\
 f^{(k)}(x)&=\argmax_{c\in\mathcal Y}s_c^{(k)}(\mathbf z(x)).\label{eq:head-prediction}
\end{align}
\end{subequations}
The rows $(\mathbf w_c^{(k)})^\top$ and entries $b_c^{(k)}$ form the client head $(\mathbf W^{(k)},\mathbf b^{(k)})$.
This rule assigns each sample to the nearest personalized class mean under the shared metric $\mathbf A$.
Under the Gaussian model with personalized means and common covariance $\mathbf A^{-1}$, it is also the equal-prior Bayes rule.
Appendix~\ref{app:head} gives the distance and likelihood derivations.

\textbf{Prediction Stability under Mean Adaptation.}
We compare the personalized classifier with its shared counterpart, obtained by setting all adaptation weights to zero.
Although absent-class scores remain unchanged, adapting observed-class means can alter their decision regions.
Let $s_c^0(\mathbf z)=\vect\mu_c^\top\mathbf A\mathbf z-\tfrac12\vect\mu_c^\top\mathbf A\vect\mu_c$ be the shared score and $\vect\delta_c^{(k)}=\alpha_c^{(k)}\vect\Delta_c^{(k)}$ the adapted displacement.
Substituting $\widetilde{\vect\mu}_c^{(k)}=\vect\mu_c+\vect\delta_c^{(k)}$ into Eq.~\eqref{eq:head-score} yields the exact change
\begin{equation}\label{eq:score-change}
\begin{aligned}
 s_c^{(k)}(\mathbf z)-s_c^0(\mathbf z)
 ={}&\alpha_c^{(k)}(\vect\Delta_c^{(k)})^\top\mathbf A(\mathbf z-\vect\mu_c)\\
 &-\tfrac12(\alpha_c^{(k)})^2\|\vect\Delta_c^{(k)}\|_{\mathbf A}^2.
\end{aligned}
\end{equation}
By the Cauchy--Schwarz inequality,
\begin{equation}\label{eq:score-bound}
\begin{aligned}
 |s_c^{(k)}(\mathbf z)-s_c^0(\mathbf z)|
 &\le \epsilon_c^{(k)}(\mathbf z),\\
 \epsilon_c^{(k)}(\mathbf z)
 &=\|\vect\delta_c^{(k)}\|_{\mathbf A}
   \|\mathbf z-\vect\mu_c\|_{\mathbf A}\\
 &\quad+\tfrac12\|\vect\delta_c^{(k)}\|_{\mathbf A}^2.
\end{aligned}
\end{equation}

Let $f^0(x)=\argmax_c s_c^0(\mathbf z(x))$ be the shared predictor and $c^\star=f^0(x)$. Define its winning margin and the perturbation budget as
\begin{subequations}\label{eq:margin-budget}
\begin{align}
 M^0(x)&=s_{c^\star}^0(\mathbf z(x))-\max_{c\ne c^\star}s_c^0(\mathbf z(x)),\\
 B^{(k)}(x)&=\epsilon_{c^\star}^{(k)}(\mathbf z(x))
 +\max_{c\ne c^\star}\epsilon_c^{(k)}(\mathbf z(x)).
\end{align}
\end{subequations}
\begin{corollary}[Prediction Stability]\label{cor:preservation}
For fixed fitted statistics and $\mathbf A\succ0$,
\begin{equation}\label{eq:margin-condition}
 M^0(x)>B^{(k)}(x)\ \Longrightarrow\ f^{(k)}(x)=f^0(x).
\end{equation}
\end{corollary}
\begin{proof}
The proof and associated target-risk bounds are given in Appendix~\ref{app:score}.
\end{proof}

Direct local-mean adaptation uses local means for observed classes and shared means for absent classes.
For fixed class means, metric, and query, SASP yields no larger a perturbation budget $B^{(k)}(x)$, making the prediction-preservation condition easier to satisfy.

Inference requires only the frozen backbone, shared layer scaling, and the client-specific linear head; personalization is completed before deployment.

\section{Experiments}\label{sec:exp}
We evaluate local adaptation, generalization across the full label space, and client recognition floors through four complementary analyses.
After describing the experimental setup and evaluation protocol, we
(i) benchmark FedGRC against external baselines in Section~\ref{sec:exp-comprehensive};
(ii) compare recognition before and after personalization to examine the effects of local adaptation in Section~\ref{sec:exp-personalization};
(iii) assess the contributions of the framework's components and their internal designs through ablations in Section~\ref{sec:exp-designs}; and
(iv) examine backbone--classifier synergy and gains over strong baseline combinations in Section~\ref{sec:exp-cross}.

\subsection{Experimental Setup and Evaluation Protocol}
\label{sec:exp-setup}

\subsubsection{Baselines and Datasets}
We benchmark 21 external methods, comprising 15 complete FL methods and 6 statistical classifiers evaluated on 2 backbone types.
The resulting 27 baseline configurations fall into three categories:
\begin{itemize}
  \item \textbf{Complete FL Methods.} FedAvg~\cite{FedAvg_AISTATS17}, FedRep~\cite{FedRep_ICML21}, FedROD~\cite{FedROD_ICLR22}, FedLC~\cite{FedLC_ICML22}, FedTHE~\cite{FedTHE_ICLR23}, FedETF~\cite{FedETF_ICCV23}, FedFA~\cite{FedFA_TMC24}, pFedFDA~\cite{pFedFDA_NeurIPS24}, FedVLS~\cite{FedVLS_AAAI25}, FedDr+~\cite{FedDrPlus_TMLR25}, BTFL~\cite{BTFL_KDD25}, HPFL~\cite{HPFL_ICLR25}, FedYoYo~\cite{FedYoYo_ICCV25}, FedLog~\cite{FedLog_TMLR26}, and pFedGM~\cite{pFedGM_arXiv26} are evaluated using their own training procedures and designated classifiers to compare complete learning and prediction systems.
  \item \textbf{Statistical Classifiers on FedAvg.} CCVR~\cite{CCVR_NeurIPS21}, FedNCM~\cite{FedNCM_NeurIPS23}, Fed3R~\cite{Fed3R_ICML24}, FedCGS~\cite{FedCGS_AAAI25}, FedCOF~\cite{FedCOF_NeurIPS25}, and GH-LDA~\cite{GHLDA_ICLR26} are applied to the FedAvg backbone. These combinations assess the effects of statistical readouts on recognition with a standard federated representation.
  \item \textbf{Statistical Classifiers on FedYoYo.} CCVR, FedNCM, Fed3R, FedCGS, FedCOF, and GH-LDA are also applied to the FedYoYo backbone. These combinations provide stronger references by pairing an advanced representation-learning method with existing statistical classifiers.
\end{itemize}
Evaluation spans 5 heterogeneity settings, 3 label-space sizes, and full or dynamic partial client participation:
\begin{itemize}
  \item \textbf{CIFAR-10}~\cite{CIFAR_TR09}. We partition data from 10 classes among 20 clients using Dirichlet label partitions~\cite{Dirichlet_arXiv19}, with Dir(0.3) and Dir(0.1) representing strong and extreme label heterogeneity, respectively. All clients participate in each round.
  \item \textbf{CIFAR-100}~\cite{CIFAR_TR09}. We distribute data from 100 classes among 20 clients using Dir(0.3) and Dir(0.1) label partitions, with full participation each round.
  \item \textbf{FedRS}~\cite{FedRSBench_FedKDD25}. We use the NIID-1 partition, comprising 15 unified classes and 135 clients from 8 remote-sensing datasets. This setting combines label and sample-size imbalance with sensor- and region-dependent feature shifts. Dynamic participation is simulated by sampling 10 of the 135 clients per round to represent clients joining and leaving training.
\end{itemize}
All experiments use 3 random seeds. Results are reported as mean $\pm$ sample standard deviation, with means alone shown in the ablation, personalization, and cross-comparison tables.
Each metric is computed within each run and then averaged across seeds.
The comparison tables group the three baseline categories into separate blocks: (+) denotes an attached classifier, ($\dagger$) and ($*$) mark the best and second-best results, and ($\Delta$) gives FedGRC's gain over the best baseline for each metric.
Accuracies are reported in (\%), and gains are absolute differences in percentage points.

\subsubsection{Evaluation Protocol}\label{sec:eval-protocol}
Each deployed client model predicts over all $C$ federation-known classes. The same predictor is evaluated on its local test set $\mathcal T_{\mathrm{local}}^{(k)}$ and a shared global test set $\mathcal T_{\mathrm{global}}$ covering every class, including those rare or absent in its local training data.
The local test sets provide \textit{local} accuracy and \textit{L@25}; the shared global test set provides \textit{global}, \textit{unseen}, and \textit{rare} accuracy, together with \textit{U@25} and \textit{R@25}.

Let $L^{(k)}$ denote client $k$'s \textit{local} accuracy and $m_L^{(k)}=|\mathcal T_{\mathrm{local}}^{(k)}|$ its local test-set size. For class $c$, $a_c^{(k)}$ is the number of correct predictions made by client $k$ on $\mathcal T_{\mathrm{global}}$, and $t_c>0$ is the number of test samples in that class.

\begin{itemize}
  \item \textbf{Local Adaptation and Global All-Class Generalization.} The overall accuracies are
  \begin{equation}\label{eq:eval-local-global}
    \textit{local}=\frac{\sum_k m_L^{(k)}L^{(k)}}{\sum_k m_L^{(k)}},\quad
    \textit{global}=\frac1{KC}\sum_{k,c}\frac{a_c^{(k)}}{t_c}.
  \end{equation}
  The \textit{local} metric weights clients by their local test-set sizes. The \textit{global} metric gives equal weight to classes within each client and then to clients, measuring the deployed client models' recognition on a common class-balanced all-class distribution.

  \item \textbf{Locally Unseen and Rare-Class Performance.} For client $k$, locally unseen classes form $\mathcal U^{(k)}=\{c:n_c^{(k)}=0\}$, with $m_U^{(k)}=\sum_{c\in\mathcal U^{(k)}}t_c$ corresponding test samples. For clients with locally unseen classes, define client-level and aggregate \textit{unseen} accuracies as
  \begin{equation}\label{eq:eval-client-subsets}
    U^{(k)}=\frac{\sum_{c\in\mathcal U^{(k)}}a_c^{(k)}}{m_U^{(k)}},\,
    \textit{unseen}=\frac{\sum_k m_U^{(k)}U^{(k)}}{\sum_k m_U^{(k)}}.
  \end{equation}
  Clients with no locally unseen classes are excluded from these sums and the \textit{U@25} ranking, while contributing to all other metrics.

  The locally rare set $\mathcal R^{(k)}$ contains the 2 classes with the smallest positive local training counts, with ties broken by class ID; if the client has 1 observed class, that class forms the set. With $m_R^{(k)}=\sum_{c\in\mathcal R^{(k)}}t_c$ corresponding test samples, the client-level accuracy and aggregate \textit{rare} metric are
  \begin{equation}\label{eq:eval-masked}
    R^{(k)}=\frac{\sum_{c\in\mathcal R^{(k)}}a_c^{(k)}}{m_R^{(k)}},\quad
    \textit{rare}=\frac{\sum_k m_R^{(k)}R^{(k)}}{\sum_k m_R^{(k)}}.
  \end{equation}
  Both aggregate metrics weight each client's accuracy by its corresponding class-subset test count.

  \item \textbf{Ability-Specific Personalization Floors.} The \textit{L@25}, \textit{U@25}, and \textit{R@25} metrics average the accuracies of the worst-performing 25\% of clients in local adaptation, unseen-class recognition, and rare-class recognition, respectively.
  Each metric summarizes one ability's lower tail using its own client ranking.

  For a metric with $m_v$ eligible clients, let $\mathbf v$ collect their accuracies, sorted as $v_{(1)}\le\cdots\le v_{(m_v)}$. The target tail size is $q=\lceil K/4\rceil$, based on all $K$ clients, with $q_v=\min(q,m_v)$ eligible clients included in the average. The common averaging rule is
  \begin{equation}\label{eq:eval-tail}
    \operatorname{Tail}_{25}(\mathbf v)=\frac1{q_v}\sum_{j=1}^{q_v}v_{(j)}.
  \end{equation}
  Writing the client accuracy vectors as $\mathbf L$, $\mathbf U$, and $\mathbf R$, respectively, gives $\textit{L@25}=\operatorname{Tail}_{25}(\mathbf L)$, $\textit{U@25}=\operatorname{Tail}_{25}(\mathbf U)$, and $\textit{R@25}=\operatorname{Tail}_{25}(\mathbf R)$. The target is 5 of 20 clients for CIFAR and 34 of 135 for FedRS.

  \item \textbf{Overall Personalization Floor.} The \textit{F@25} metric takes the minimum of each client's local, unseen-class, and rare-class accuracies before averaging the lowest 25\% of the resulting client scores.
  For clients with no locally unseen classes, the minimum is taken over local and rare-class accuracies. With $\mathbf F=(F^{(1)},\ldots,F^{(K)})^\top$,
  \begin{align}
    F^{(k)}&=\min\{L^{(k)},U^{(k)},R^{(k)}\},\label{eq:eval-client-floor}\\
    \textit{F@25}&=\operatorname{Tail}_{25}(\mathbf F).\label{eq:eval-f25}
  \end{align}
  Taking the within-client minimum retains each client's weakest ability; averaging the lowest 25\% then captures the clients with the weakest such scores. The resulting personalization floor summarizes weaknesses across both abilities and clients, even when the weakest ability differs from one client to another. Appendix~\ref{app:more} explains how this aggregation and $\min\{\textit{L@25},\textit{U@25},\textit{R@25}\}$ differ and complement each other.
\end{itemize}

\subsubsection{Implementation}
Appendix~\ref{app:updates} details FedGRC's implementation, and Table~\ref{tab:fedgrc-parameters} consolidates its parameter settings and configuration.
All experiments were conducted on an Ubuntu 22.04.5 LTS server with two Intel Xeon Platinum 8481C CPUs and four NVIDIA GeForce RTX 4090 GPUs, using Python 3.11.15 and PyTorch 2.0.1 with CUDA 11.8.

\begin{table*}[!t]
\centering
\caption{Comprehensive comparison on CIFAR-10~\cite{CIFAR_TR09} under Dirichlet label heterogeneity. All methods are trained using ResNet-18. Symbols: ($\dagger$) denotes the best, ($*$) the second-best, and ($\Delta$) the gain over the best baseline.}
\label{tab:main-cifar10}
\footnotesize
\renewcommand{\arraystretch}{1.05}
\def\tabbody{%
\toprule
\multirow{2}{*}{\footnotesize Method} & \multirow{2}{*}{\footnotesize Venue} & \multicolumn{5}{c}{CIFAR-10, Dir(0.3), 20 clients} &  & \multicolumn{5}{c}{CIFAR-10, Dir(0.1), 20 clients} \\
\cmidrule(lr){3-7}\cmidrule(lr){9-13}
& & local & global & unseen & rare & F@25 &  & local & global & unseen & rare & F@25 \\
\midrule
\multicolumn{13}{@{}l}{\grp{Complete FL methods (method-specific training and classifiers)}} \\
FedAvg~\cite{FedAvg_AISTATS17} & AISTATS'17 & \ms{60.77}{0.39} & \ms{59.05}{0.98} & \ms{58.16}{0.59} & \ms{54.45}{1.29} & \ms{40.40}{1.94} &  & \ms{44.84}{0.52} & \ms{45.01}{0.75} & \ms{42.72}{0.95} & \ms{50.32}{0.22} & \ms{21.84}{0.92} \\
FedRep~\cite{FedRep_ICML21} & ICML'21 & \ms{79.25}{0.25} & \ms{30.77}{0.09} & \ms{0.00}{0.00} & \ms{7.44}{0.20} & \ms{0.00}{0.00} &  & \ms{88.29}{0.07} & \ms{18.35}{0.06} & \ms{0.00}{0.00} & \ms{13.20}{0.53} & \ms{0.00}{0.00} \\
FedROD~\cite{FedROD_ICLR22} & ICLR'22 & \ms{81.51}{0.09} & \ms{59.94}{0.99} & \ms{69.62}{1.03} & \ms{52.60}{0.79} & \ms{41.81}{0.58} &  & \ms{89.26}{0.15} & \ms{48.35}{0.62} & \ms{50.47}{0.59} & \ms{43.39}{0.81} & \ms{22.73}{1.21} \\
FedLC~\cite{FedLC_ICML22} & ICML'22 & \ms{61.16}{0.28} & \ms{59.76}{0.53} & \ms{70.19}{0.58} & \ms{52.05}{0.25} & \ms{39.19}{0.32} &  & \ms{44.88}{0.11} & \ms{45.18}{0.46} & \ms{47.80}{0.39} & \ms{38.44}{0.87} & \ms{10.48}{1.09} \\
FedTHE~\cite{FedTHE_ICLR23} & ICLR'23 & \ms{74.96}{1.01} & \ms{55.20}{0.97} & \ms{61.31}{0.88} & \ms{48.97}{1.04} & \ms{38.81}{1.55} &  & \ms{82.89}{0.55} & \ms{43.12}{0.60} & \ms{43.34}{0.69} & \ms{42.73}{0.40} & \ms{27.68}{0.52} \\
FedETF~\cite{FedETF_ICCV23} & ICCV'23 & \ms{70.66}{0.38} & \ms{69.27}{0.86} & \ms{74.62}{1.40} & \ms{63.54}{0.51} & \ms{52.88}{0.24} &  & \ms{57.35}{0.14} & \ms{56.86}{0.43} & \ms{55.79}{0.65} & \ms{60.03}{0.16} & \ms{33.84}{4.61} \\
FedFA~\cite{FedFA_TMC24} & TMC'24 & \ms{39.05}{3.79} & \ms{38.27}{3.07} & \ms{35.93}{2.68} & \ms{35.61}{3.52} & \ms{14.47}{6.92} &  & \ms{18.40}{1.48} & \ms{18.45}{1.54} & \ms{17.55}{3.90} & \ms{21.74}{8.16} & \ms{0.35}{0.31} \\
pFedFDA~\cite{pFedFDA_NeurIPS24} & NeurIPS'24 & \ms{80.81}{0.28} & \ms{37.99}{0.21} & \ms{6.10}{1.40} & \ms{15.96}{0.09} & \ms{1.70}{0.55} &  & \ms{88.29}{0.18} & \ms{19.38}{0.27} & \ms{0.28}{0.05} & \ms{15.44}{0.70} & \ms{0.00}{0.01} \\
FedVLS~\cite{FedVLS_AAAI25} & AAAI'25 & \ms{60.76}{0.16} & \ms{59.28}{0.78} & \ms{69.32}{0.81} & \ms{51.39}{1.00} & \ms{37.43}{1.04} &  & \ms{46.32}{0.26} & \ms{46.25}{0.37} & \ms{48.51}{0.29} & \ms{40.56}{0.52} & \ms{14.10}{0.47} \\
FedDr+~\cite{FedDrPlus_TMLR25} & TMLR'25 & \ms{69.16}{1.09} & \ms{68.31}{1.34} & \ms{67.26}{2.31} & \ms{64.36}{1.31} & \ms{51.22}{2.05} &  & \ms{32.30}{1.96} & \ms{32.49}{2.10} & \ms{28.88}{1.52} & \ms{37.21}{4.41} & \ms{1.66}{1.56} \\
BTFL~\cite{BTFL_KDD25} & KDD'25 & \ms{80.30}{0.27} & \ms{53.91}{0.40} & \ms{30.00}{0.70} & \ms{46.87}{0.89} & \ms{18.50}{1.67} &  & \msb{89.45}{0.25} & \ms{29.64}{0.50} & \ms{7.45}{0.59} & \ms{38.92}{0.80} & \ms{0.84}{0.44} \\
HPFL~\cite{HPFL_ICLR25} & ICLR'25 & \mss{81.72}{0.41} & \ms{50.06}{0.35} & \ms{50.55}{0.24} & \ms{44.24}{0.79} & \ms{17.68}{1.64} &  & \mss{89.41}{0.16} & \ms{26.03}{0.46} & \ms{24.83}{0.78} & \ms{25.39}{0.93} & \ms{10.86}{2.31} \\
FedYoYo~\cite{FedYoYo_ICCV25} & ICCV'25 & \ms{76.96}{0.06} & \ms{76.64}{0.48} & \ms{77.17}{1.50} & \ms{74.52}{0.57} & \ms{68.18}{1.05} &  & \ms{59.15}{0.97} & \ms{58.80}{0.81} & \ms{56.21}{1.44} & \ms{66.12}{1.43} & \ms{28.23}{3.75} \\
FedLog~\cite{FedLog_TMLR26} & TMLR'26 & \ms{76.80}{0.21} & \ms{29.32}{0.06} & \ms{0.00}{0.00} & \ms{6.17}{0.29} & \ms{0.00}{0.00} &  & \ms{88.90}{0.35} & \ms{19.50}{0.21} & \ms{0.00}{0.00} & \ms{15.11}{0.25} & \ms{0.00}{0.00} \\
pFedGM~\cite{pFedGM_arXiv26} & arXiv'26 & \ms{81.51}{0.11} & \ms{36.16}{0.44} & \ms{0.14}{0.15} & \ms{15.36}{0.38} & \ms{0.00}{0.00} &  & \ms{87.37}{0.13} & \ms{17.33}{0.59} & \ms{0.00}{0.00} & \ms{11.74}{1.22} & \ms{0.00}{0.00} \\
\addlinespace[2pt]
\multicolumn{13}{@{}l}{\grp{Statistical classifiers assembled on a vanilla backbone (FedAvg~+)}} \\
+ CCVR~\cite{CCVR_NeurIPS21} & NeurIPS'21 & \ms{63.05}{0.02} & \ms{61.46}{1.05} & \ms{65.64}{1.12} & \ms{56.94}{1.24} & \ms{47.99}{1.90} &  & \ms{51.98}{0.52} & \ms{51.32}{0.63} & \ms{50.91}{0.60} & \ms{52.82}{0.84} & \ms{39.17}{0.88} \\
+ FedNCM~\cite{FedNCM_NeurIPS23} & NeurIPS'23 & \ms{63.17}{0.24} & \ms{61.82}{0.88} & \ms{66.69}{0.71} & \ms{57.15}{1.23} & \ms{48.51}{1.61} &  & \ms{52.03}{0.53} & \ms{51.72}{0.11} & \ms{51.23}{0.27} & \ms{53.39}{0.13} & \ms{40.60}{1.08} \\
+ Fed3R~\cite{Fed3R_ICML24} & ICML'24 & \ms{63.49}{0.24} & \ms{62.42}{1.24} & \ms{66.17}{0.92} & \ms{57.06}{1.67} & \ms{44.84}{2.27} &  & \ms{55.50}{0.30} & \ms{54.76}{0.23} & \ms{53.93}{0.27} & \ms{57.60}{0.11} & \ms{41.35}{0.57} \\
+ FedCGS~\cite{FedCGS_AAAI25} & AAAI'25 & \ms{63.15}{0.38} & \ms{61.77}{0.71} & \ms{65.78}{0.73} & \ms{57.98}{0.85} & \ms{50.80}{1.64} &  & \ms{55.08}{0.40} & \ms{54.36}{0.24} & \ms{54.50}{0.11} & \ms{54.33}{0.55} & \ms{44.89}{0.34} \\
+ FedCOF~\cite{FedCOF_NeurIPS25} & NeurIPS'25 & \ms{55.78}{0.69} & \ms{54.10}{0.44} & \ms{55.74}{0.42} & \ms{50.18}{0.23} & \ms{39.53}{0.22} &  & \ms{46.77}{0.60} & \ms{46.68}{0.36} & \ms{46.00}{0.25} & \ms{48.95}{1.13} & \ms{34.66}{1.32} \\
+ GH-LDA~\cite{GHLDA_ICLR26} & ICLR'26 & \ms{62.83}{0.39} & \ms{61.36}{0.92} & \ms{65.11}{1.00} & \ms{57.95}{0.96} & \ms{50.63}{1.83} &  & \ms{55.22}{0.27} & \ms{54.89}{0.40} & \ms{54.91}{0.46} & \ms{55.23}{0.39} & \ms{45.24}{0.41} \\
\addlinespace[2pt]
\multicolumn{13}{@{}l}{\grp{Statistical classifiers assembled on the SOTA backbone (FedYoYo~+)}} \\
+ CCVR~\cite{CCVR_NeurIPS21} & NeurIPS'21 & \ms{78.18}{0.19} & \ms{77.53}{0.06} & \ms{80.71}{0.28} & \ms{73.38}{0.17} & \ms{65.89}{0.93} &  & \ms{67.80}{0.10} & \ms{66.78}{0.48} & \ms{66.34}{0.59} & \ms{68.10}{0.30} & \ms{53.25}{1.32} \\
+ FedNCM~\cite{FedNCM_NeurIPS23} & NeurIPS'23 & \ms{77.69}{0.14} & \ms{77.30}{0.42} & \ms{80.69}{0.47} & \ms{72.94}{0.51} & \ms{64.72}{1.11} &  & \ms{63.56}{0.60} & \ms{61.93}{0.12} & \ms{61.12}{0.09} & \ms{64.74}{0.70} & \ms{50.43}{0.72} \\
+ Fed3R~\cite{Fed3R_ICML24} & ICML'24 & \ms{79.19}{0.23} & \ms{78.39}{0.20} & \mss{82.18}{0.33} & \ms{73.80}{0.10} & \ms{65.10}{0.38} &  & \ms{70.71}{0.17} & \mss{69.55}{0.37} & \ms{68.87}{0.46} & \mss{71.61}{0.47} & \ms{52.49}{2.01} \\
+ FedCGS~\cite{FedCGS_AAAI25} & AAAI'25 & \ms{79.08}{0.28} & \mss{78.59}{0.09} & \ms{81.53}{0.36} & \mss{75.34}{0.23} & \mss{69.72}{0.69} &  & \ms{70.40}{0.35} & \ms{69.03}{0.36} & \mss{68.91}{0.41} & \ms{69.33}{0.23} & \mss{61.03}{0.14} \\
+ FedCOF~\cite{FedCOF_NeurIPS25} & NeurIPS'25 & \ms{76.70}{0.10} & \ms{76.23}{0.19} & \ms{79.67}{0.29} & \ms{72.16}{0.42} & \ms{62.47}{0.92} &  & \ms{63.99}{0.47} & \ms{62.76}{0.16} & \ms{62.06}{0.19} & \ms{64.98}{0.32} & \ms{50.88}{1.22} \\
+ GH-LDA~\cite{GHLDA_ICLR26} & ICLR'26 & \ms{78.89}{0.23} & \ms{78.53}{0.08} & \ms{81.40}{0.34} & \ms{75.28}{0.10} & \ms{69.63}{0.53} &  & \ms{69.82}{0.41} & \ms{68.68}{0.35} & \ms{68.42}{0.40} & \ms{69.42}{0.51} & \ms{59.40}{0.58} \\
\midrule
\rowcolor{gray!12}
\textbf{FedGRC (Ours)} & -- & \msb{82.16}{0.27} & \msb{80.76}{0.31} & \msb{84.21}{0.46} & \msb{77.51}{0.41} & \msb{71.85}{0.66} &  & \ms{73.81}{0.46} & \msb{73.05}{0.33} & \msb{73.88}{0.40} & \msb{71.95}{0.41} & \msb{63.85}{0.13} \\
$\Delta$ vs.\ Best &  & \dltp{$+$0.44} & \dltp{$+$2.17} & \dltp{$+$2.03} & \dltp{$+$2.17} & \dltp{$+$2.13} &  & \dltm{$-$15.64} & \dltp{$+$3.50} & \dltp{$+$4.97} & \dltp{$+$0.34} & \dltp{$+$2.82} \\
\bottomrule}%
\setlength{\tabcolsep}{0pt}%
\sbox{\tabbox}{\begin{tabular}{ll rrrrr @{\hskip2\tabcolsep}c@{\hskip1\tabcolsep} rrrrr}\tabbody\end{tabular}}%
\ifdim\wd\tabbox<\textwidth
  \setlength{\tabcolsep}{\dimexpr(\textwidth-\wd\tabbox)/27\relax}%
\else
  \setlength{\tabcolsep}{1pt}%
\fi
\begin{tabular}{ll rrrrr @{\hskip2\tabcolsep}c@{\hskip1\tabcolsep} rrrrr}\tabbody\end{tabular}
\end{table*}

\begin{table*}[!t]
\centering
\caption{Comprehensive comparison on CIFAR-100~\cite{CIFAR_TR09} under Dirichlet label heterogeneity. All methods are trained using ResNet-18. Symbols: ($\dagger$) denotes the best, ($*$) the second-best, and ($\Delta$) the gain over the best baseline.}
\label{tab:main-cifar100}
\footnotesize
\renewcommand{\arraystretch}{1.05}
\def\tabbody{%
\toprule
\multirow{2}{*}{\footnotesize Method} & \multirow{2}{*}{\footnotesize Venue} & \multicolumn{5}{c}{CIFAR-100, Dir(0.3), 20 clients} &  & \multicolumn{5}{c}{CIFAR-100, Dir(0.1), 20 clients} \\
\cmidrule(lr){3-7}\cmidrule(lr){9-13}
& & local & global & unseen & rare & F@25 &  & local & global & unseen & rare & F@25 \\
\midrule
\multicolumn{13}{@{}l}{\grp{Complete FL methods (method-specific training and classifiers)}} \\
FedAvg~\cite{FedAvg_AISTATS17} & AISTATS'17 & \ms{35.37}{0.45} & \ms{34.80}{0.82} & \ms{33.33}{0.97} & \ms{36.57}{1.28} & \ms{23.00}{4.13} &  & \ms{32.85}{0.27} & \ms{32.73}{0.23} & \ms{32.39}{0.23} & \ms{28.85}{1.25} & \ms{17.11}{1.83} \\
FedRep~\cite{FedRep_ICML21} & ICML'21 & \ms{37.74}{0.29} & \ms{11.88}{0.17} & \ms{0.00}{0.00} & \ms{0.37}{0.15} & \ms{0.00}{0.00} &  & \ms{49.81}{0.31} & \ms{9.33}{0.08} & \ms{0.00}{0.00} & \ms{0.23}{0.11} & \ms{0.00}{0.00} \\
FedROD~\cite{FedROD_ICLR22} & ICLR'22 & \ms{48.66}{0.12} & \ms{34.72}{0.40} & \ms{34.99}{0.51} & \ms{35.10}{0.86} & \ms{23.67}{1.77} &  & \msb{56.97}{0.75} & \ms{32.93}{0.69} & \ms{33.39}{0.67} & \ms{28.82}{0.45} & \ms{18.22}{1.80} \\
FedLC~\cite{FedLC_ICML22} & ICML'22 & \ms{32.78}{0.17} & \ms{32.04}{0.30} & \ms{33.03}{0.10} & \ms{31.53}{1.15} & \ms{20.50}{1.53} &  & \ms{30.78}{0.25} & \ms{30.91}{0.16} & \ms{31.73}{0.18} & \ms{26.86}{0.36} & \ms{16.72}{0.81} \\
FedTHE~\cite{FedTHE_ICLR23} & ICLR'23 & \ms{37.03}{0.32} & \ms{28.76}{0.43} & \ms{28.64}{0.37} & \ms{29.44}{0.83} & \ms{18.94}{0.74} &  & \ms{44.69}{0.14} & \ms{27.29}{0.16} & \ms{27.56}{0.13} & \ms{24.67}{1.05} & \ms{14.78}{1.55} \\
FedETF~\cite{FedETF_ICCV23} & ICCV'23 & \ms{39.08}{0.08} & \ms{38.77}{0.14} & \ms{36.94}{0.63} & \ms{40.57}{1.11} & \ms{24.02}{3.05} &  & \ms{38.00}{0.63} & \ms{37.80}{0.45} & \ms{37.20}{0.45} & \ms{34.52}{0.36} & \ms{22.17}{1.43} \\
FedFA~\cite{FedFA_TMC24} & TMC'24 & \ms{7.53}{2.62} & \ms{7.66}{2.78} & \ms{6.29}{1.98} & \ms{12.01}{5.21} & \ms{2.41}{2.54} &  & \ms{3.20}{0.37} & \ms{3.06}{0.48} & \ms{2.93}{0.51} & \ms{5.39}{4.18} & \ms{0.00}{0.00} \\
pFedFDA~\cite{pFedFDA_NeurIPS24} & NeurIPS'24 & \ms{46.04}{0.54} & \ms{25.64}{0.28} & \ms{14.27}{0.22} & \ms{21.65}{1.15} & \ms{9.62}{1.88} &  & \ms{54.74}{0.55} & \ms{21.20}{0.69} & \ms{12.35}{0.82} & \ms{15.54}{0.80} & \ms{7.17}{0.44} \\
FedVLS~\cite{FedVLS_AAAI25} & AAAI'25 & \ms{31.20}{0.65} & \ms{30.79}{0.71} & \ms{33.41}{0.53} & \ms{29.51}{1.60} & \ms{19.06}{1.09} &  & \ms{29.87}{0.49} & \ms{29.38}{1.01} & \ms{30.70}{1.08} & \ms{24.86}{1.10} & \ms{14.76}{0.37} \\
FedDr+~\cite{FedDrPlus_TMLR25} & TMLR'25 & \ms{39.36}{0.41} & \ms{39.13}{0.17} & \ms{36.42}{0.18} & \ms{42.57}{0.63} & \ms{25.21}{0.80} &  & \ms{34.26}{0.93} & \ms{34.59}{0.71} & \ms{33.11}{0.80} & \ms{31.96}{0.95} & \ms{20.67}{0.73} \\
BTFL~\cite{BTFL_KDD25} & KDD'25 & \ms{40.74}{0.29} & \ms{32.50}{0.58} & \ms{27.26}{1.01} & \ms{34.08}{0.38} & \ms{20.85}{0.57} &  & \ms{50.17}{0.54} & \ms{29.64}{0.25} & \ms{25.13}{0.20} & \ms{29.17}{1.09} & \ms{17.91}{0.70} \\
HPFL~\cite{HPFL_ICLR25} & ICLR'25 & \ms{43.88}{0.12} & \ms{32.92}{0.33} & \ms{31.66}{0.60} & \ms{34.01}{0.97} & \ms{20.47}{3.20} &  & \ms{53.13}{0.82} & \ms{27.24}{0.32} & \ms{26.89}{0.32} & \ms{24.26}{0.78} & \ms{14.78}{1.13} \\
FedYoYo~\cite{FedYoYo_ICCV25} & ICCV'25 & \ms{35.71}{1.65} & \ms{34.59}{1.72} & \ms{35.90}{2.08} & \ms{31.75}{1.41} & \ms{13.17}{3.08} &  & \ms{39.97}{1.35} & \ms{39.25}{1.69} & \ms{39.01}{1.59} & \ms{40.90}{1.78} & \ms{28.39}{5.48} \\
FedLog~\cite{FedLog_TMLR26} & TMLR'26 & \ms{35.26}{0.10} & \ms{10.19}{0.03} & \ms{0.01}{0.00} & \ms{0.28}{0.02} & \ms{0.00}{0.00} &  & \ms{47.81}{0.17} & \ms{8.38}{0.09} & \ms{0.00}{0.00} & \ms{0.15}{0.05} & \ms{0.00}{0.00} \\
pFedGM~\cite{pFedGM_arXiv26} & arXiv'26 & \ms{45.53}{0.73} & \ms{25.40}{0.20} & \ms{6.14}{0.51} & \ms{22.57}{0.54} & \ms{4.21}{0.38} &  & \mss{54.75}{0.23} & \ms{14.83}{0.22} & \ms{1.46}{0.20} & \ms{13.31}{1.37} & \ms{0.70}{0.10} \\
\addlinespace[2pt]
\multicolumn{13}{@{}l}{\grp{Statistical classifiers assembled on a vanilla backbone (FedAvg~+)}} \\
+ CCVR~\cite{CCVR_NeurIPS21} & NeurIPS'21 & \ms{35.09}{0.40} & \ms{34.56}{0.74} & \ms{34.80}{0.76} & \ms{34.64}{1.37} & \ms{22.33}{2.97} &  & \ms{33.56}{0.41} & \ms{33.48}{0.47} & \ms{33.84}{0.42} & \ms{30.98}{1.05} & \ms{18.83}{0.73} \\
+ FedNCM~\cite{FedNCM_NeurIPS23} & NeurIPS'23 & \ms{36.14}{0.18} & \ms{35.47}{0.66} & \ms{35.52}{0.67} & \ms{36.11}{1.40} & \ms{24.27}{3.10} &  & \ms{34.24}{0.43} & \ms{34.22}{0.27} & \ms{34.53}{0.14} & \ms{30.27}{0.77} & \ms{18.78}{1.00} \\
+ Fed3R~\cite{Fed3R_ICML24} & ICML'24 & \ms{34.27}{0.46} & \ms{33.67}{0.39} & \ms{32.91}{0.35} & \ms{35.07}{1.83} & \ms{21.72}{3.38} &  & \ms{33.10}{0.31} & \ms{32.87}{0.14} & \ms{32.97}{0.15} & \ms{28.78}{0.52} & \ms{16.78}{1.39} \\
+ FedCGS~\cite{FedCGS_AAAI25} & AAAI'25 & \ms{33.54}{0.10} & \ms{33.05}{0.59} & \ms{33.67}{0.54} & \ms{33.21}{1.07} & \ms{21.67}{3.38} &  & \ms{32.60}{0.37} & \ms{32.24}{0.59} & \ms{32.78}{0.58} & \ms{28.85}{0.63} & \ms{16.50}{1.59} \\
+ FedCOF~\cite{FedCOF_NeurIPS25} & NeurIPS'25 & \ms{35.19}{0.24} & \ms{34.54}{0.55} & \ms{34.20}{0.84} & \ms{35.72}{0.94} & \ms{23.89}{3.73} &  & \ms{33.43}{0.69} & \ms{32.97}{0.19} & \ms{33.10}{0.26} & \ms{29.29}{1.19} & \ms{17.89}{0.75} \\
+ GH-LDA~\cite{GHLDA_ICLR26} & ICLR'26 & \ms{34.56}{0.06} & \ms{34.10}{0.64} & \ms{34.80}{0.53} & \ms{33.93}{0.98} & \ms{22.78}{3.26} &  & \ms{33.50}{0.20} & \ms{33.08}{0.29} & \ms{33.60}{0.32} & \ms{30.32}{0.59} & \ms{18.28}{0.25} \\
\addlinespace[2pt]
\multicolumn{13}{@{}l}{\grp{Statistical classifiers assembled on the SOTA backbone (FedYoYo~+)}} \\
+ CCVR~\cite{CCVR_NeurIPS21} & NeurIPS'21 & \ms{52.38}{0.20} & \mss{51.56}{0.33} & \mss{50.88}{0.55} & \ms{55.58}{1.13} & \mss{41.10}{1.36} &  & \ms{49.95}{0.19} & \ms{49.71}{0.31} & \ms{49.57}{0.27} & \ms{47.44}{0.35} & \ms{36.22}{0.67} \\
+ FedNCM~\cite{FedNCM_NeurIPS23} & NeurIPS'23 & \ms{49.36}{0.33} & \ms{48.51}{0.16} & \ms{47.11}{0.17} & \ms{53.73}{0.36} & \ms{36.43}{1.18} &  & \ms{45.61}{0.24} & \ms{44.95}{0.53} & \ms{44.50}{0.50} & \ms{42.48}{0.97} & \ms{28.11}{0.42} \\
+ Fed3R~\cite{Fed3R_ICML24} & ICML'24 & \mss{52.43}{0.07} & \ms{51.26}{0.10} & \ms{49.69}{0.31} & \mss{56.46}{1.43} & \ms{36.39}{1.08} &  & \ms{50.00}{0.03} & \mss{49.78}{0.39} & \ms{49.40}{0.36} & \ms{46.40}{0.36} & \ms{31.72}{1.08} \\
+ FedCGS~\cite{FedCGS_AAAI25} & AAAI'25 & \ms{51.77}{0.13} & \ms{50.51}{0.33} & \ms{50.52}{0.36} & \ms{52.29}{1.08} & \ms{40.92}{1.32} &  & \ms{49.91}{0.32} & \ms{49.72}{0.82} & \mss{49.68}{0.77} & \mss{48.36}{0.60} & \mss{38.49}{1.44} \\
+ FedCOF~\cite{FedCOF_NeurIPS25} & NeurIPS'25 & \ms{50.92}{0.30} & \ms{50.10}{0.26} & \ms{48.88}{0.30} & \ms{55.11}{0.48} & \ms{39.68}{0.41} &  & \ms{47.59}{0.30} & \ms{46.84}{0.63} & \ms{46.36}{0.59} & \ms{45.30}{0.43} & \ms{33.61}{0.98} \\
+ GH-LDA~\cite{GHLDA_ICLR26} & ICLR'26 & \ms{51.39}{0.15} & \ms{50.41}{0.32} & \ms{49.93}{0.36} & \ms{52.61}{0.85} & \ms{39.97}{0.93} &  & \ms{49.58}{0.31} & \ms{49.02}{0.64} & \ms{48.92}{0.63} & \ms{47.03}{0.33} & \ms{36.82}{1.11} \\
\midrule
\rowcolor{gray!12}
\textbf{FedGRC (Ours)} & -- & \msb{55.61}{0.32} & \msb{55.71}{0.65} & \msb{55.75}{0.78} & \msb{58.00}{0.52} & \msb{47.56}{1.83} &  & \ms{53.25}{0.35} & \msb{53.22}{0.26} & \msb{53.61}{0.32} & \msb{52.92}{0.71} & \msb{41.83}{1.53} \\
$\Delta$ vs.\ Best &  & \dltp{$+$3.18} & \dltp{$+$4.15} & \dltp{$+$4.87} & \dltp{$+$1.54} & \dltp{$+$6.46} &  & \dltm{$-$3.72} & \dltp{$+$3.44} & \dltp{$+$3.93} & \dltp{$+$4.56} & \dltp{$+$3.34} \\
\bottomrule}%
\setlength{\tabcolsep}{0pt}%
\sbox{\tabbox}{\begin{tabular}{ll rrrrr @{\hskip2\tabcolsep}c@{\hskip1\tabcolsep} rrrrr}\tabbody\end{tabular}}%
\ifdim\wd\tabbox<\textwidth
  \setlength{\tabcolsep}{\dimexpr(\textwidth-\wd\tabbox)/27\relax}%
\else
  \setlength{\tabcolsep}{1pt}%
\fi
\begin{tabular}{ll rrrrr @{\hskip2\tabcolsep}c@{\hskip1\tabcolsep} rrrrr}\tabbody\end{tabular}
\end{table*}

\begin{table*}[!t]
\centering
\caption{Comprehensive comparison on FedRS~\cite{FedRSBench_FedKDD25} with 8 remote-sensing sources, 15 classes, and 135 clients. All methods are trained using ResNet-18. Symbols: ($\dagger$) denotes the best, ($*$) the second-best, and ($\Delta$) the gain over the best baseline.}
\label{tab:main-fedrs}
\footnotesize
\renewcommand{\arraystretch}{1.05}
\def\tabbody{%
\toprule
\multirow{2}{*}{\footnotesize Method} & \multirow{2}{*}{\footnotesize Venue} & \multicolumn{4}{c}{Overall performance} & & \multicolumn{4}{c}{Worst 25\% of clients} \\
\cmidrule(lr){3-6}\cmidrule(lr){8-11}
& & local & global & unseen & rare & & L@25 & U@25 & R@25 & F@25 \\
\midrule
\multicolumn{11}{@{}l}{\grp{Complete FL methods (method-specific training and classifiers)}} \\
FedAvg~\cite{FedAvg_AISTATS17} & AISTATS'17 & \ms{59.97}{0.28} & \ms{61.27}{0.49} & \ms{59.15}{0.93} & \ms{64.32}{2.03} &  & \ms{41.49}{3.59} & \ms{40.56}{4.74} & \ms{47.87}{2.98} & \ms{37.10}{3.28} \\
FedRep~\cite{FedRep_ICML21} & ICML'21 & \ms{71.30}{0.46} & \ms{25.19}{0.45} & \ms{0.00}{0.00} & \ms{18.93}{0.44} &  & \ms{62.48}{0.42} & \ms{0.00}{0.00} & \ms{4.52}{0.07} & \ms{0.00}{0.00} \\
FedROD~\cite{FedROD_ICLR22} & ICLR'22 & \ms{71.62}{0.11} & \ms{62.26}{2.19} & \ms{63.13}{3.71} & \ms{62.61}{0.72} &  & \ms{61.63}{1.03} & \ms{50.80}{4.48} & \ms{49.37}{2.53} & \ms{46.63}{3.26} \\
FedLC~\cite{FedLC_ICML22} & ICML'22 & \ms{57.10}{0.78} & \ms{61.77}{0.22} & \ms{63.21}{0.25} & \ms{61.56}{1.55} &  & \ms{39.74}{0.90} & \ms{48.59}{0.86} & \ms{47.95}{2.04} & \ms{37.88}{0.49} \\
FedTHE~\cite{FedTHE_ICLR23} & ICLR'23 & \ms{67.03}{0.62} & \ms{58.21}{1.76} & \ms{57.20}{2.27} & \ms{60.50}{1.51} &  & \ms{54.53}{0.84} & \ms{40.76}{3.27} & \ms{45.58}{1.69} & \ms{39.26}{3.15} \\
FedETF~\cite{FedETF_ICCV23} & ICCV'23 & \ms{59.99}{1.43} & \ms{60.60}{1.89} & \ms{59.28}{1.17} & \ms{62.44}{3.69} &  & \ms{39.01}{7.60} & \ms{40.55}{8.12} & \ms{45.63}{3.68} & \ms{35.74}{6.79} \\
FedFA~\cite{FedFA_TMC24} & TMC'24 & \ms{52.15}{4.69} & \ms{52.30}{5.88} & \ms{47.27}{7.39} & \ms{57.89}{4.34} &  & \ms{33.00}{4.36} & \ms{34.78}{11.14} & \ms{42.30}{6.33} & \ms{29.74}{7.94} \\
pFedFDA~\cite{pFedFDA_NeurIPS24} & NeurIPS'24 & \ms{70.27}{1.14} & \ms{47.59}{1.03} & \ms{24.96}{2.55} & \ms{49.57}{0.38} &  & \ms{59.45}{1.37} & \ms{15.65}{3.32} & \ms{33.53}{1.10} & \ms{15.60}{3.29} \\
FedVLS~\cite{FedVLS_AAAI25} & AAAI'25 & \ms{53.81}{0.75} & \ms{57.89}{0.99} & \ms{57.63}{1.81} & \ms{58.59}{2.14} &  & \ms{33.63}{1.72} & \ms{45.67}{1.83} & \ms{45.25}{2.16} & \ms{32.60}{1.66} \\
FedDr+~\cite{FedDrPlus_TMLR25} & TMLR'25 & \ms{62.14}{0.67} & \ms{62.19}{1.27} & \ms{59.76}{0.90} & \ms{65.98}{2.19} &  & \ms{42.86}{2.23} & \ms{39.59}{5.03} & \ms{48.90}{2.28} & \ms{37.31}{3.86} \\
BTFL~\cite{BTFL_KDD25} & KDD'25 & \ms{67.40}{0.85} & \ms{57.99}{2.25} & \ms{47.98}{3.42} & \ms{63.75}{1.95} &  & \ms{54.82}{1.22} & \ms{29.89}{5.35} & \ms{47.33}{2.52} & \ms{29.74}{5.31} \\
HPFL~\cite{HPFL_ICLR25} & ICLR'25 & \ms{65.94}{0.68} & \ms{61.70}{0.15} & \ms{59.57}{0.85} & \ms{64.65}{1.72} &  & \ms{53.51}{1.05} & \ms{43.03}{3.31} & \ms{49.87}{1.33} & \ms{41.84}{2.38} \\
FedYoYo~\cite{FedYoYo_ICCV25} & ICCV'25 & \ms{69.36}{1.16} & \ms{65.01}{1.47} & \ms{59.38}{2.50} & \ms{71.55}{1.29} &  & \ms{54.12}{3.20} & \ms{50.17}{9.80} & \ms{55.26}{2.26} & \ms{45.66}{8.37} \\
FedLog~\cite{FedLog_TMLR26} & TMLR'26 & \ms{52.31}{1.27} & \ms{16.88}{0.28} & \ms{0.02}{0.03} & \ms{11.92}{0.67} &  & \ms{41.11}{1.11} & \ms{0.00}{0.00} & \ms{0.72}{0.26} & \ms{0.00}{0.00} \\
pFedGM~\cite{pFedGM_arXiv26} & arXiv'26 & \ms{67.34}{0.72} & \ms{42.42}{1.39} & \ms{10.09}{1.50} & \ms{53.25}{1.55} &  & \ms{54.71}{0.66} & \ms{4.88}{1.17} & \ms{38.28}{2.71} & \ms{4.88}{1.17} \\
\addlinespace[2pt]
\multicolumn{11}{@{}l}{\grp{Statistical classifiers assembled on a vanilla backbone (FedAvg~+)}} \\
+ CCVR~\cite{CCVR_NeurIPS21} & NeurIPS'21 & \ms{60.73}{0.36} & \ms{61.08}{0.81} & \ms{58.62}{1.29} & \ms{64.60}{0.53} &  & \ms{48.51}{0.66} & \ms{50.87}{3.89} & \ms{53.62}{1.31} & \ms{45.81}{1.75} \\
+ FedNCM~\cite{FedNCM_NeurIPS23} & NeurIPS'23 & \ms{60.49}{0.44} & \ms{60.87}{1.01} & \ms{59.07}{1.35} & \ms{64.05}{0.98} &  & \ms{46.15}{1.15} & \ms{50.82}{4.77} & \ms{51.32}{1.40} & \ms{43.57}{1.91} \\
+ Fed3R~\cite{Fed3R_ICML24} & ICML'24 & \ms{60.69}{0.21} & \ms{60.75}{0.36} & \ms{57.91}{0.73} & \ms{64.48}{0.68} &  & \ms{47.57}{0.33} & \ms{46.39}{1.18} & \ms{51.57}{1.99} & \ms{43.59}{0.31} \\
+ FedCGS~\cite{FedCGS_AAAI25} & AAAI'25 & \ms{60.10}{0.08} & \ms{57.51}{1.12} & \ms{52.66}{1.51} & \ms{62.90}{0.54} &  & \ms{47.48}{0.24} & \ms{46.33}{0.81} & \ms{52.22}{1.35} & \ms{43.94}{0.38} \\
+ FedCOF~\cite{FedCOF_NeurIPS25} & NeurIPS'25 & \ms{59.65}{0.08} & \ms{60.31}{1.12} & \ms{58.70}{1.31} & \ms{62.81}{0.67} &  & \ms{47.42}{0.78} & \ms{45.35}{2.27} & \ms{50.92}{1.13} & \ms{42.97}{1.42} \\
+ GH-LDA~\cite{GHLDA_ICLR26} & ICLR'26 & \ms{60.70}{0.14} & \ms{59.81}{0.10} & \ms{55.98}{0.27} & \ms{64.24}{0.72} &  & \ms{49.01}{0.50} & \ms{50.93}{2.24} & \ms{55.15}{1.67} & \ms{47.15}{0.68} \\
\addlinespace[2pt]
\multicolumn{11}{@{}l}{\grp{Statistical classifiers assembled on the SOTA backbone (FedYoYo~+)}} \\
+ CCVR~\cite{CCVR_NeurIPS21} & NeurIPS'21 & \ms{84.74}{0.28} & \ms{82.85}{1.02} & \ms{82.11}{1.07} & \ms{84.34}{1.24} &  & \ms{76.38}{0.50} & \mss{80.70}{0.67} & \ms{78.24}{0.86} & \ms{74.66}{0.02} \\
+ FedNCM~\cite{FedNCM_NeurIPS23} & NeurIPS'23 & \ms{81.71}{0.95} & \ms{81.48}{0.82} & \ms{81.65}{0.86} & \ms{82.44}{0.69} &  & \ms{68.11}{1.94} & \ms{79.20}{2.15} & \ms{74.49}{1.63} & \ms{67.19}{1.96} \\
+ Fed3R~\cite{Fed3R_ICML24} & ICML'24 & \mss{86.96}{0.04} & \mss{83.78}{0.68} & \mss{83.88}{0.67} & \ms{84.78}{1.20} &  & \ms{79.68}{1.15} & \ms{76.75}{3.34} & \ms{77.57}{1.44} & \ms{74.10}{1.84} \\
+ FedCGS~\cite{FedCGS_AAAI25} & AAAI'25 & \ms{86.56}{0.24} & \ms{83.75}{0.75} & \ms{82.05}{1.06} & \ms{85.90}{0.76} &  & \ms{79.72}{0.80} & \ms{79.86}{1.93} & \mss{80.82}{0.93} & \mss{77.04}{1.48} \\
+ FedCOF~\cite{FedCOF_NeurIPS25} & NeurIPS'25 & \ms{84.41}{0.26} & \ms{83.54}{0.86} & \ms{83.68}{0.95} & \ms{84.40}{1.22} &  & \ms{76.36}{0.30} & \ms{79.97}{4.09} & \ms{78.14}{1.42} & \ms{74.22}{1.16} \\
+ GH-LDA~\cite{GHLDA_ICLR26} & ICLR'26 & \ms{86.59}{0.02} & \mss{83.78}{0.72} & \ms{82.12}{0.97} & \mss{86.10}{0.82} &  & \mss{80.04}{0.53} & \ms{79.51}{2.66} & \ms{80.52}{0.82} & \ms{76.86}{1.33} \\
\midrule
\rowcolor{gray!12}
\textbf{FedGRC (Ours)} & -- & \msb{89.11}{0.05} & \msb{87.08}{0.49} & \msb{87.34}{0.45} & \msb{87.54}{0.24} &  & \msb{84.28}{0.81} & \msb{85.53}{0.79} & \msb{82.91}{0.66} & \msb{81.65}{1.54} \\
$\Delta$ vs.\ Best & & \dltp{$+$2.15} & \dltp{$+$3.30} & \dltp{$+$3.46} & \dltp{$+$1.44} & & \dltp{$+$4.24} & \dltp{$+$4.83} & \dltp{$+$2.09} & \dltp{$+$4.61} \\
\bottomrule}%
\setlength{\tabcolsep}{0pt}%
\sbox{\tabbox}{\begin{tabular}{ll rrrr @{\hskip2\tabcolsep}c@{\hskip0.4\tabcolsep} rrrr}\tabbody\end{tabular}}%
\ifdim\wd\tabbox<\textwidth
  \setlength{\tabcolsep}{\dimexpr(\textwidth-\wd\tabbox)/22\relax}%
\else
  \setlength{\tabcolsep}{1pt}%
\fi
\begin{tabular}{ll rrrr @{\hskip2\tabcolsep}c@{\hskip0.4\tabcolsep} rrrr}\tabbody\end{tabular}
\end{table*}

\subsection{Comprehensive Evaluation of Personalization, Generalization, and Recognition Floors}
\label{sec:exp-comprehensive}
Tables~\ref{tab:main-cifar10}--\ref{tab:main-fedrs} compare local adaptation, generalization across the full label space, and client recognition floors.
FedGRC achieves the highest mean \textit{global}, \textit{unseen}, \textit{rare}, and \textit{F@25} scores in all 5 heterogeneity settings. Its gains over the respective best baselines are 2.17\%--4.15\% in \textit{global} accuracy and 2.13\%--6.46\% in \textit{F@25}.
It also leads in \textit{local} accuracy on both CIFAR datasets under Dir(0.3) and on FedRS; the two Dir(0.1) settings expose a trade-off between local adaptation and broader recognition.

\textbf{Local Adaptation and Class Coverage.}
As heterogeneity increases from Dir(0.3) to Dir(0.1), pFedFDA~\cite{pFedFDA_NeurIPS24}, FedLog~\cite{FedLog_TMLR26}, and pFedGM~\cite{pFedGM_arXiv26} achieve higher \textit{local} but lower \textit{global} accuracy on both CIFAR datasets (Fig.~\ref{fig:motivation}).
The same pattern occurs for FedRep~\cite{FedRep_ICML21}, FedROD~\cite{FedROD_ICLR22}, FedTHE~\cite{FedTHE_ICLR23}, BTFL~\cite{BTFL_KDD25}, and HPFL~\cite{HPFL_ICLR25}.
For example, the \textit{local} accuracy of FedLog~\cite{FedLog_TMLR26} on CIFAR-10 rises from 76.80\% to 88.90\%, while its \textit{global} accuracy falls from 29.32\% to 19.50\%.
As labels become more concentrated, frequent classes receive more weight in local evaluation, so higher \textit{local} accuracy can coexist with weaker recognition across the full label space.
Indeed, under Dir(0.1), pFedFDA~\cite{pFedFDA_NeurIPS24}, FedLog~\cite{FedLog_TMLR26}, and pFedGM~\cite{pFedGM_arXiv26} attain 87.37\%--88.90\% \textit{local} accuracy on CIFAR-10, yet their \textit{F@25} scores round to 0.00\%.
Higher local scores therefore conceal severe weaknesses in individual clients' least reliable recognition abilities.

\textbf{Benefits of Statistical Classifiers.}
The gains from statistical classification depend on the underlying representation.
On CIFAR-100 Dir(0.3), the six statistical classifiers raise FedYoYo's~\cite{FedYoYo_ICCV25} \textit{global} accuracy from 34.59\% to 48.51\%--51.56\%, whereas the same classifiers on FedAvg~\cite{FedAvg_AISTATS17} yield 33.05\%--35.47\%, compared with 34.80\% for its original head.
This contrast underscores the importance of designing the representation and statistical classifier together.
FedGRC further exceeds the strongest baseline combinations in \textit{global} accuracy and \textit{F@25} across all five settings.

\textbf{Comprehensive Performance on FedRS.}
On FedRS (Table~\ref{tab:main-fedrs}), FedGRC reaches 89.11\% \textit{local} and 87.08\% \textit{global} accuracy, exceeding the respective best baselines by 2.15\% and 3.30\%.
It also improves \textit{unseen} and \textit{rare} accuracy by 3.46\% and 1.44\%, respectively.
These gains span local adaptation and all-class generalization, including recognition of classes with little or no local training support.
This combination is especially relevant to remote-sensing clients with limited class coverage and sensor and geographic heterogeneity.

\textbf{Trade-Offs under Extreme Label Heterogeneity.}
Under Dir(0.1), FedGRC trails the best \textit{local} baseline by 15.64\% on CIFAR-10 and 3.72\% on CIFAR-100, despite leading in \textit{global}, \textit{unseen}, \textit{rare}, and \textit{F@25}.
On CIFAR-10, BTFL~\cite{BTFL_KDD25} attains 89.45\% \textit{local} accuracy, compared with FedGRC's 73.81\%. BTFL's \textit{global} accuracy and \textit{F@25} are 29.64\% and 0.84\%, compared with 73.05\% and 63.85\% for FedGRC.
Under extreme label heterogeneity, these broader recognition and client-floor gains come at a measurable local-accuracy cost.

\textbf{Recognition Floors across Clients.}
On FedRS, Table~\ref{tab:main-fedrs} shows gains of 4.24\%, 4.83\%, and 2.09\% in \textit{L@25}, \textit{U@25}, and \textit{R@25}, respectively, over the best baseline for each metric.
These gains extend the aggregate improvements to the weakest client groups in each recognition ability.
FedGRC also raises \textit{F@25} by 4.61\% over its best baseline.
This gain captures a higher overall personalization floor even when the weakest ability differs across clients.
The CIFAR results in Appendix~\ref{app:more}, Tables~\ref{tab:pfp-cifar10} and~\ref{tab:pfp-cifar100}, likewise show the highest \textit{U@25}, \textit{R@25}, and \textit{F@25} in all 4 CIFAR heterogeneity settings.
Under Dir(0.1), however, \textit{L@25} trails the best baseline by 13.57\% on CIFAR-10 and 3.59\% on CIFAR-100, showing that the local-adaptation trade-off also affects the lower tail.
We next examine how personalization changes this balance.

\subsection{Recognition Before and After Personalization}
\label{sec:exp-personalization}
Table~\ref{tab:personalization} compares predictors immediately before and after terminal personalization for 3 methods with explicit local adaptation stages.
FedETF~\cite{FedETF_ICCV23} adapts both the representation and classifier, whereas pFedGM~\cite{pFedGM_arXiv26} freezes its shared generator and adapts the statistical classifier.
FedGRC keeps the shared CGRL backbone, MMS layer scales, and covariance metric fixed and replaces shared class means with SASP means, isolating the effect of its final adaptation stage.

\begin{table}[!htbp]
\centering
\fontsize{8}{9.5}\selectfont
\caption{Recognition before and after local personalization. Before: shared predictors. After: FedETF~\cite{FedETF_ICCV23}, full-model adaptation; pFedGM~\cite{pFedGM_arXiv26}, classifier adaptation with the generator unchanged; FedGRC, SASP with CGRL+MMS unchanged. The signed values in \textit{parentheses} indicate changes (After minus Before).}
\label{tab:personalization}
\renewcommand{\arraystretch}{1.0}
\setlength{\tabcolsep}{1.6pt}
\begin{tabular}{ll*{5}{>{\centering\arraybackslash}p{30pt}}}
\toprule
Method & State & local & global & unseen & rare & F@25 \\
\midrule
\multicolumn{7}{l}{\grp{(a) CIFAR-100, Dir(0.3)}} \\
 & Before & 39.08 & 38.77 & 36.94 & 40.57 & 24.02 \\
 &  & 46.20 & 25.58 & 14.04 & 18.72 & 1.33 \\
\multicolumn{1}{c}{\multirow[c]{-3}{*}{FedETF~\cite{FedETF_ICCV23}}} & \multirow{-2}{*}{After} & (\gp{+7.12}) & (\gm{-13.19}) & (\gm{-22.90}) & (\gm{-21.85}) & (\gm{-22.69}) \\
\addlinespace[1pt]
 & Before & 32.13 & 32.39 & 30.97 & 34.03 & 21.75 \\
 &  & 45.53 & 25.40 & 6.14 & 22.57 & 4.21 \\
\multicolumn{1}{c}{\multirow[c]{-3}{*}{pFedGM~\cite{pFedGM_arXiv26}}} & \multirow{-2}{*}{After} & (\gp{+13.40}) & (\gm{-6.99}) & (\gm{-24.83}) & (\gm{-11.46}) & (\gm{-17.54}) \\
\addlinespace[1pt]
\rowcolor{gray!12}
 & Before & 55.53 & 55.98 & 55.67 & 57.83 & 46.93 \\
\rowcolor{gray!12}
 &  & 55.61 & 55.71 & 55.75 & 58.00 & 47.56 \\
\rowcolor{gray!12}
\multicolumn{1}{c}{\multirow[c]{-3}{*}{\makecell[c]{\textbf{FedGRC}\\\textbf{(Ours)}}}} & \multirow{-2}{*}{After} & (\gp{+0.08}) & (\gm{-0.27}) & (\gp{+0.08}) & (\gp{+0.17}) & (\gp{+0.63}) \\
\addlinespace[1pt]
\midrule
\multicolumn{7}{l}{\grp{(b) FedRS}} \\
 & Before & 59.99 & 60.60 & 59.28 & 62.44 & 35.74 \\
 &  & 75.74 & 35.35 & 16.95 & 29.81 & 3.96 \\
\multicolumn{1}{c}{\multirow[c]{-3}{*}{FedETF~\cite{FedETF_ICCV23}}} & \multirow{-2}{*}{After} & (\gp{+15.75}) & (\gm{-25.25}) & (\gm{-42.33}) & (\gm{-32.63}) & (\gm{-31.78}) \\
\addlinespace[1pt]
 & Before & 57.75 & 56.38 & 52.51 & 60.51 & 33.29 \\
 &  & 67.34 & 42.42 & 10.09 & 53.25 & 4.88 \\
\multicolumn{1}{c}{\multirow[c]{-3}{*}{pFedGM~\cite{pFedGM_arXiv26}}} & \multirow{-2}{*}{After} & (\gp{+9.59}) & (\gm{-13.96}) & (\gm{-42.42}) & (\gm{-7.26}) & (\gm{-28.41}) \\
\addlinespace[1pt]
\rowcolor{gray!12}
 & Before & 87.70 & 86.89 & 86.87 & 87.11 & 80.04 \\
\rowcolor{gray!12}
 &  & 89.11 & 87.08 & 87.34 & 87.54 & 81.65 \\
\rowcolor{gray!12}
\multicolumn{1}{c}{\multirow[c]{-3}{*}{\makecell[c]{\textbf{FedGRC}\\\textbf{(Ours)}}}} & \multirow{-2}{*}{After} & (\gp{+1.41}) & (\gp{+0.19}) & (\gp{+0.47}) & (\gp{+0.43}) & (\gp{+1.61}) \\
\addlinespace[1pt]
\bottomrule
\end{tabular}
\end{table}

Both reference methods improve \textit{local} accuracy while substantially reducing generalization and client recognition floors.
On CIFAR-100 Dir(0.3), their \textit{local} gains of 7.12\%--13.40\% accompany losses of 6.99\%--13.19\% in \textit{global} accuracy and 17.54\%--22.69\% in \textit{F@25}.
The same pattern occurs on FedRS: FedETF~\cite{FedETF_ICCV23} gains 15.75\% locally but loses 25.25\% in \textit{global} accuracy and 31.78\% in \textit{F@25}.
These changes show that fitting the local distribution can substantially weaken recognition beyond it, even when adaptation is confined to the classifier.
They motivate controlling the influence of local evidence during personalization.

SASP improves local adaptation and recognition floors while largely retaining the shared classifier's all-class accuracy.
On CIFAR-100 Dir(0.3), it raises \textit{local} accuracy from 55.53\% to 55.61\% and \textit{F@25} from 46.93\% to 47.56\%, with a 0.27\% decrease in \textit{global} accuracy.
On FedRS, all five metrics improve, including gains of 1.41\% in \textit{local} accuracy and 1.61\% in \textit{F@25}.
These results support class-mean adaptation within a shared representation and metric as a way to preserve broad recognition during personalization.
Next, ablations isolate the contributions of representation learning, statistical readout, and support-adaptive mean estimation.

\subsection{Assessing Contributions of the Proposed Designs}
\label{sec:exp-designs}

\subsubsection{Ablation of FedGRC: Component Contributions to Generalization and Personalized Recognition}
\label{sec:exp-construction}

Table~\ref{tab:abl-construction} traces the construction of FedGRC on CIFAR-100 Dir(0.3) and FedRS by adding CGRL, MMS, and SASP sequentially. The initial system is trained with supervised cross-entropy averaged over paired weakly and strongly augmented views of each example and evaluated with the aggregated linear head.
CGRL adds geometric guidance during representation learning while retaining the linear head; MMS replaces that head with a shared multilayer statistical readout; SASP then adapts the class means.
Reported changes are relative to the initial system, while comparisons between adjacent stages reveal the effect of each addition.

The complete configuration improves all 5 metrics over the initial system in both settings, demonstrating gains in local adaptation, broader recognition, and client floors from the combined design.

With the aggregated linear head retained, CGRL improves all five metrics on FedRS.
On CIFAR-100 Dir(0.3), adding CGRL while retaining the aggregated linear head reduces \textit{F@25} from 37.77\% to 29.17\%; replacing that head with MMS raises it to 46.93\%.
MMS also improves all five metrics on FedRS, including a 15.32\% gain in \textit{F@25} over the CGRL-only stage.
These comparisons support matching the statistical readout to geometric representation learning.

\begin{table*}[!t]
  \centering
  \begin{minipage}[t]{\textwidth}
    \vspace{0pt}
    \captionsetup{type=table}
    \centering
\caption{FedGRC module construction on CIFAR-100 Dir(0.3) and FedRS: CGRL, MMS, and SASP are added sequentially. (\yes) denotes inclusion. The signed values in \textit{parentheses} indicate changes relative to the initial system.}
\label{tab:abl-construction}
\footnotesize
\settowidth{\incw}{($-$00.00)}
\renewcommand{\arraystretch}{1.05}
\def\tabbody{%
\toprule
SASP & MMS & CGRL & local & global & unseen & rare & F@25 \\
\midrule
\multicolumn{8}{l}{\grp{(a) CIFAR-100~\cite{CIFAR_TR09}, Dir(0.3), 20 clients}} \\
\no & \no & \no & \abn{47.19} & \abn{45.67} & \abn{42.24} & \abn{50.46} & \abn{37.77} \\
\no & \no & \yes & \ab{46.46}{\gm{$-$0.73}} & \ab{46.53}{\gp{$+$0.86}} & \ab{42.01}{\gm{$-$0.23}} & \ab{45.83}{\gm{$-$4.63}} & \ab{29.17}{\gm{$-$8.60}} \\
\no & \yes & \yes & \ab{55.53}{\gp{$+$8.34}} & \ab{55.98}{\gp{$+$10.31}} & \ab{55.67}{\gp{$+$13.43}} & \ab{57.83}{\gp{$+$7.37}} & \ab{46.93}{\gp{$+$9.16}} \\
\rowcolor{gray!12}
\textbf{\yes} & \textbf{\yes} & \textbf{\yes} & \ab{\textbf{55.61}}{\gp{$+$8.42}} & \ab{\textbf{55.71}}{\gp{$+$10.04}} & \ab{\textbf{55.75}}{\gp{$+$13.51}} & \ab{\textbf{58.00}}{\gp{$+$7.54}} & \ab{\textbf{47.56}}{\gp{$+$9.79}} \\
\addlinespace[3pt]
\multicolumn{8}{l}{\grp{(b) FedRS~\cite{FedRSBench_FedKDD25}, 8 sources, 15 classes, 135 clients (10 per round)}} \\
\no & \no & \no & \abn{78.09} & \abn{78.70} & \abn{76.45} & \abn{81.84} & \abn{64.09} \\
\no & \no & \yes & \ab{81.07}{\gp{$+$2.98}} & \ab{81.03}{\gp{$+$2.33}} & \ab{79.01}{\gp{$+$2.56}} & \ab{84.03}{\gp{$+$2.19}} & \ab{64.72}{\gp{$+$0.63}} \\
\no & \yes & \yes & \ab{87.70}{\gp{$+$9.61}} & \ab{86.89}{\gp{$+$8.19}} & \ab{86.87}{\gp{$+$10.42}} & \ab{87.11}{\gp{$+$5.27}} & \ab{80.04}{\gp{$+$15.95}} \\
\rowcolor{gray!12}
\textbf{\yes} & \textbf{\yes} & \textbf{\yes} & \ab{\textbf{89.11}}{\gp{$+$11.02}} & \ab{\textbf{87.08}}{\gp{$+$8.38}} & \ab{\textbf{87.34}}{\gp{$+$10.89}} & \ab{\textbf{87.54}}{\gp{$+$5.70}} & \ab{\textbf{81.65}}{\gp{$+$17.56}} \\
\bottomrule}%
\setlength{\tabcolsep}{0pt}%
\sbox{\tabbox}{\begin{tabular}{ccc rrrrr}\tabbody\end{tabular}}%
\ifdim\wd\tabbox<\textwidth
  \setlength{\tabcolsep}{\dimexpr(\textwidth-\wd\tabbox)/16\relax}%
\else
  \setlength{\tabcolsep}{1pt}%
\fi
\ifdim\tabcolsep>8pt \setlength{\tabcolsep}{8pt}\fi
\begin{tabular}{ccc rrrrr}\tabbody\end{tabular}

  \end{minipage}

  \par\vspace{12pt}

  \noindent
  \begin{minipage}[t]{\dimexpr(\textwidth-\columnsep)/2\relax}
    \vspace{0pt}
    \captionsetup{type=table}
    \centering
\caption{Ablation of directional and radial guidance in CGRL. (\yes) denotes inclusion. All configurations retain cross-entropy training and the same MMS+SASP readout. The signed values in \textit{parentheses} indicate changes relative to training without either guidance component.}
\label{tab:preview-cgrl}
\footnotesize
\renewcommand{\arraystretch}{1.15}
\def\splitcell#1#2{\ensuremath{\vcenter{\hbox{\shortstack[r]{#1\\[-0.5pt](#2)}}}}}
\def\tabbody{%
\toprule
Radial & Direction & \multicolumn{1}{c}{local} & \multicolumn{1}{c}{global} & \multicolumn{1}{c}{unseen} & \multicolumn{1}{c}{rare} & \multicolumn{1}{c}{F@25} \\
\midrule
\multicolumn{7}{l}{\grp{(a) CIFAR-100, Dir(0.3), 20 clients}} \\
\addlinespace[2pt]
\no & \no & 49.60 & 49.56 & 50.29 & 51.42 & 40.19 \\
\addlinespace[3pt]
\yes & \no & \splitcell{49.97}{\gp{$+$0.37}} & \splitcell{50.46}{\gp{$+$0.90}} & \splitcell{51.59}{\gp{$+$1.30}} & \splitcell{51.21}{\gm{$-$0.21}} & \splitcell{40.17}{\gm{$-$0.02}} \\
\addlinespace[3pt]
\no & \yes & \splitcell{54.99}{\gp{$+$5.39}} & \splitcell{55.06}{\gp{$+$5.50}} & \splitcell{55.42}{\gp{$+$5.13}} & \splitcell{57.13}{\gp{$+$5.71}} & \splitcell{46.44}{\gp{$+$6.25}} \\
\addlinespace[3pt]
\rowcolor{gray!12}
\textbf{\yes} & \textbf{\yes} & \splitcell{\textbf{55.61}}{\gp{$+$6.01}} & \splitcell{\textbf{55.71}}{\gp{$+$6.15}} & \splitcell{\textbf{55.75}}{\gp{$+$5.46}} & \splitcell{\textbf{58.00}}{\gp{$+$6.58}} & \splitcell{\textbf{47.56}}{\gp{$+$7.37}} \\
\midrule
\multicolumn{7}{l}{\grp{(b) FedRS, 135 clients}} \\
\addlinespace[2pt]
\no & \no & 89.21 & 82.40 & 81.46 & 86.37 & 76.53 \\
\addlinespace[3pt]
\yes & \no & \splitcell{89.03}{\gm{$-$0.18}} & \splitcell{83.53}{\gp{$+$1.13}} & \splitcell{84.34}{\gp{$+$2.88}} & \splitcell{85.93}{\gm{$-$0.44}} & \splitcell{78.68}{\gp{$+$2.15}} \\
\addlinespace[3pt]
\no & \yes & \splitcell{89.23}{\gp{$+$0.02}} & \splitcell{86.70}{\gp{$+$4.30}} & \splitcell{87.29}{\gp{$+$5.83}} & \splitcell{87.71}{\gp{$+$1.34}} & \splitcell{79.42}{\gp{$+$2.89}} \\
\addlinespace[3pt]
\rowcolor{gray!12}
\textbf{\yes} & \textbf{\yes} & \splitcell{\textbf{89.11}}{\gm{$-$0.10}} & \splitcell{\textbf{87.08}}{\gp{$+$4.68}} & \splitcell{\textbf{87.34}}{\gp{$+$5.88}} & \splitcell{\textbf{87.54}}{\gp{$+$1.17}} & \splitcell{\textbf{81.65}}{\gp{$+$5.12}} \\
\bottomrule}%
\setlength{\tabcolsep}{0pt}
\sbox{\tabbox}{\begin{tabular}{cc rrrrr}\tabbody\end{tabular}}
\setlength{\tabcolsep}{\dimexpr(\linewidth-\wd\tabbox)/14\relax}
\begin{tabular}{cc rrrrr}\tabbody\end{tabular}

  \end{minipage}\hfill%
  \begin{minipage}[t]{\dimexpr(\textwidth-\columnsep)/2\relax}
    \vspace{0pt}
    \captionsetup{type=table}
    \centering
\caption{Ablation of the statistical metric and multilayer features in MMS. (\yes) denotes inclusion. All configurations use the same CGRL backbone and SASP procedure. The signed values in \textit{parentheses} indicate changes relative to Euclidean classification using only the last-layer features.}
\label{tab:preview-mms}
\footnotesize
\renewcommand{\arraystretch}{1.15}
\def\splitcell#1#2{\ensuremath{\vcenter{\hbox{\shortstack[r]{#1\\[-0.5pt](#2)}}}}}
\def\tabbody{%
\toprule
\makecell{Multi-\\layer} & \makecell{Stat.\\metric} & \multicolumn{1}{c}{local} & \multicolumn{1}{c}{global} & \multicolumn{1}{c}{unseen} & \multicolumn{1}{c}{rare} & \multicolumn{1}{c}{F@25} \\
\midrule
\multicolumn{7}{l}{\grp{(a) CIFAR-100, Dir(0.3), 20 clients}} \\
\addlinespace[2pt]
\no & \no & 55.00 & 54.25 & 53.24 & 57.38 & 43.42 \\
\addlinespace[3pt]
\no & \yes & \splitcell{55.62}{\gp{$+$0.62}} & \splitcell{55.31}{\gp{$+$1.06}} & \splitcell{55.07}{\gp{$+$1.83}} & \splitcell{57.25}{\gm{$-$0.13}} & \splitcell{45.94}{\gp{$+$2.52}} \\
\addlinespace[3pt]
\rowcolor{gray!12}
\textbf{\yes} & \textbf{\yes} & \splitcell{\textbf{55.61}}{\gp{$+$0.61}} & \splitcell{\textbf{55.71}}{\gp{$+$1.46}} & \splitcell{\textbf{55.75}}{\gp{$+$2.51}} & \splitcell{\textbf{58.00}}{\gp{$+$0.62}} & \splitcell{\textbf{47.56}}{\gp{$+$4.14}} \\
\midrule
\multicolumn{7}{l}{\grp{(b) FedRS, 135 clients}} \\
\addlinespace[2pt]
\no & \no & 85.12 & 84.29 & 84.87 & 85.26 & 73.18 \\
\addlinespace[3pt]
\no & \yes & \splitcell{87.99}{\gp{$+$2.87}} & \splitcell{85.99}{\gp{$+$1.70}} & \splitcell{86.39}{\gp{$+$1.52}} & \splitcell{86.41}{\gp{$+$1.15}} & \splitcell{79.43}{\gp{$+$6.25}} \\
\addlinespace[3pt]
\rowcolor{gray!12}
\textbf{\yes} & \textbf{\yes} & \splitcell{\textbf{89.11}}{\gp{$+$3.99}} & \splitcell{\textbf{87.08}}{\gp{$+$2.79}} & \splitcell{\textbf{87.34}}{\gp{$+$2.47}} & \splitcell{\textbf{87.54}}{\gp{$+$2.28}} & \splitcell{\textbf{81.65}}{\gp{$+$8.47}} \\
\bottomrule}%
\setlength{\tabcolsep}{0pt}
\sbox{\tabbox}{\begin{tabular}{cc rrrrr}\tabbody\end{tabular}}
\setlength{\tabcolsep}{\dimexpr(\linewidth-\wd\tabbox)/14\relax}
\begin{tabular}{cc rrrrr}\tabbody\end{tabular}

  \end{minipage}
\end{table*}
\begin{table*}[t]
\centering
\caption{Ablation of local means, count-based shrinkage, and local class priors in SASP. (\yes) denotes inclusion. All configurations use the same CGRL backbone and MMS readout. Locally absent classes retain shared means; without local priors, class priors are equal. The signed values in \textit{parentheses} indicate changes relative to the shared-mean classifier.}
\label{tab:preview-sasp}
\fontsize{7}{8.6}\selectfont
\renewcommand{\arraystretch}{1.15}
\settowidth{\incw}{($-$00.00)}
\def\tabbody{%
\toprule
\makecell{Count\\shrinkage} & \makecell{Local\\prior} & \makecell{Local\\mean} & \multicolumn{1}{c}{\fontsize{8.5}{10.2}\selectfont local} & \multicolumn{1}{c}{\fontsize{8.5}{10.2}\selectfont global} & \multicolumn{1}{c}{\fontsize{8.5}{10.2}\selectfont unseen} & \multicolumn{1}{c}{\fontsize{8.5}{10.2}\selectfont rare} & \multicolumn{1}{c}{\fontsize{8.5}{10.2}\selectfont L@25} & \multicolumn{1}{c}{\fontsize{8.5}{10.2}\selectfont U@25} & \multicolumn{1}{c}{\fontsize{8.5}{10.2}\selectfont R@25} & \multicolumn{1}{c}{\fontsize{8.5}{10.2}\selectfont F@25} \\
\midrule
\multicolumn{11}{l}{\grp{(a) CIFAR-100, Dir(0.3), 20 clients}} \\
\addlinespace[2pt]
\no & \no & \no & \abn{55.53} & \abn{55.98} & \abn{55.67} & \abn{57.83} & \abn{52.75} & \abn{52.71} & \abn{47.67} & \abn{46.93} \\
\addlinespace[1.2pt]
\no & \no & \yes & \ab{57.30}{\gp{$+$1.77}} & \ab{45.17}{\gm{$-$10.81}} & \ab{61.88}{\gp{$+$6.21}} & \ab{1.17}{\gm{$-$56.66}} & \ab{54.47}{\gp{$+$1.72}} & \ab{59.38}{\gp{$+$6.67}} & \ab{0.00}{\gm{$-$47.67}} & \ab{0.00}{\gm{$-$46.93}} \\
\addlinespace[1.2pt]
\no & \yes & \yes & \ab{58.91}{\gp{$+$3.38}} & \ab{45.29}{\gm{$-$10.69}} & \ab{60.66}{\gp{$+$4.99}} & \ab{1.21}{\gm{$-$56.62}} & \ab{56.19}{\gp{$+$3.44}} & \ab{58.24}{\gp{$+$5.53}} & \ab{0.00}{\gm{$-$47.67}} & \ab{0.00}{\gm{$-$46.93}} \\
\addlinespace[1.2pt]
\yes & \yes & \yes & \ab{57.21}{\gp{$+$1.68}} & \ab{55.63}{\gm{$-$0.35}} & \ab{54.42}{\gm{$-$1.25}} & \ab{57.04}{\gm{$-$0.79}} & \ab{54.30}{\gp{$+$1.55}} & \ab{51.31}{\gm{$-$1.40}} & \ab{45.67}{\gm{$-$2.00}} & \ab{44.33}{\gm{$-$2.60}} \\
\addlinespace[1.2pt]
\rowcolor{gray!12}
\textbf{\yes} & \textbf{\no} & \textbf{\yes} & \ab{\textbf{55.61}}{\gp{$+$0.08}} & \ab{\textbf{55.71}}{\gm{$-$0.27}} & \ab{\textbf{55.75}}{\gp{$+$0.08}} & \ab{\textbf{58.00}}{\gp{$+$0.17}} & \ab{\textbf{52.68}}{\gm{$-$0.07}} & \ab{\textbf{52.90}}{\gp{$+$0.19}} & \ab{\textbf{48.33}}{\gp{$+$0.66}} & \ab{\textbf{47.56}}{\gp{$+$0.63}} \\
\midrule
\multicolumn{11}{l}{\grp{(b) FedRS, 135 clients}} \\
\addlinespace[2pt]
\no & \no & \no & \abn{87.70} & \abn{86.89} & \abn{86.87} & \abn{87.11} & \abn{81.73} & \abn{84.35} & \abn{82.59} & \abn{80.04} \\
\addlinespace[1.2pt]
\no & \no & \yes & \ab{89.98}{\gp{$+$2.28}} & \ab{80.71}{\gm{$-$6.18}} & \ab{89.69}{\gp{$+$2.82}} & \ab{44.44}{\gm{$-$42.67}} & \ab{84.33}{\gp{$+$2.60}} & \ab{88.36}{\gp{$+$4.01}} & \ab{9.53}{\gm{$-$73.06}} & \ab{9.53}{\gm{$-$70.51}} \\
\addlinespace[1.2pt]
\no & \yes & \yes & \ab{90.43}{\gp{$+$2.73}} & \ab{80.57}{\gm{$-$6.32}} & \ab{89.08}{\gp{$+$2.21}} & \ab{43.98}{\gm{$-$43.13}} & \ab{85.29}{\gp{$+$3.56}} & \ab{87.32}{\gp{$+$2.97}} & \ab{8.93}{\gm{$-$73.66}} & \ab{8.93}{\gm{$-$71.11}} \\
\addlinespace[1.2pt]
\yes & \yes & \yes & \ab{89.60}{\gp{$+$1.90}} & \ab{86.89}{\gp{$+$0.00}} & \ab{86.50}{\gm{$-$0.37}} & \ab{87.36}{\gp{$+$0.25}} & \ab{84.76}{\gp{$+$3.03}} & \ab{83.44}{\gm{$-$0.91}} & \ab{82.74}{\gp{$+$0.15}} & \ab{81.41}{\gp{$+$1.37}} \\
\addlinespace[1.2pt]
\rowcolor{gray!12}
\textbf{\yes} & \textbf{\no} & \textbf{\yes} & \ab{\textbf{89.11}}{\gp{$+$1.41}} & \ab{\textbf{87.08}}{\gp{$+$0.19}} & \ab{\textbf{87.34}}{\gp{$+$0.47}} & \ab{\textbf{87.54}}{\gp{$+$0.43}} & \ab{\textbf{84.28}}{\gp{$+$2.55}} & \ab{\textbf{85.53}}{\gp{$+$1.18}} & \ab{\textbf{82.91}}{\gp{$+$0.32}} & \ab{\textbf{81.65}}{\gp{$+$1.61}} \\
\bottomrule}%
\setlength{\tabcolsep}{0pt}
\sbox{\tabbox}{\begin{tabular}{ccc rrrrrrrr}\tabbody\end{tabular}}
\setlength{\tabcolsep}{\dimexpr(\textwidth-\wd\tabbox)/22\relax}
\begin{tabular}{ccc rrrrrrrr}\tabbody\end{tabular}
\end{table*}

SASP further improves local adaptation and recognition floors through support-dependent mean updates, completing the transition from shared representations to personalized classification.
We next examine each module's design with the other configurations unchanged.

\subsubsection{Ablation of CGRL: Directional Alignment and Magnitude Regulation for Statistical Classification}
Table~\ref{tab:preview-cgrl} separates the effects of directional and radial guidance, with cross-entropy training and the MMS+SASP readout retained in every configuration.
Within each dataset block, the rows correspond to (i) neither guidance component; (ii) radial guidance only; (iii) directional guidance only; and (iv) both components.
Directional guidance combines supervised contrastive learning with shared class directions, while radial guidance regulates feature magnitudes.
Used separately, directional guidance produces larger \textit{global} accuracy gains than radial guidance on both datasets.
On CIFAR-100 Dir(0.3), radial guidance alone improves \textit{global} accuracy by 0.90\% while leaving \textit{F@25} nearly unchanged.
Adding radial guidance improves all five metrics over directional guidance alone, including \textit{rare} accuracy from 57.13\% to 58.00\% and \textit{F@25} from 46.44\% to 47.56\%.
These gains support complementary magnitude regulation and directional discrimination.

On FedRS, adding radial guidance to directional guidance further improves \textit{global} accuracy by 0.38\% and \textit{F@25} by 2.23\%, while reducing \textit{local} and \textit{rare} accuracy by 0.12\% and 0.17\%.
The magnitude penalty regularizes within-class log-norm dispersion and the mean log norm's deviation from the shared reference (Lemma~\ref{lem:radial}).
It therefore complements directional guidance by regulating the magnitude structure retained in the class means and covariance used for statistical classification.

Using the same full MMS+SASP readout, CGRL improves \textit{global} accuracy and \textit{F@25} over cross-entropy training by 6.15\% and 7.37\% on CIFAR-100 Dir(0.3), and by 4.68\% and 5.12\% on FedRS.
These gains under the same MMS+SASP configuration demonstrate the contribution of geometric consensus to the deployed statistical classifier.

\subsubsection{Ablation of MMS: Contributions of Statistical Geometry and Multilayer Feature Integration}
Table~\ref{tab:preview-mms} evaluates the statistical metric and multilayer integration with the same CGRL backbone and SASP procedure.
Within each setting, the rows progress from (i) Euclidean classification with raw last-layer features to (ii) statistical classification with last-layer features and (iii) statistical classification with multilayer features, the complete MMS design.
The statistical metric increases \textit{F@25} by 2.52\% on CIFAR-100 Dir(0.3) and 6.25\% on FedRS.
Multilayer integration adds a further 1.62\% and 2.22\%, respectively, while also improving \textit{global} and \textit{unseen} accuracy in both settings.
These gains reflect the complementary roles of the two MMS components: the covariance-based metric accounts for within-class variability, while multilayer integration supplies information from several representation depths.
On FedRS, the two components together raise \textit{global} accuracy by 2.79\% and \textit{F@25} by 8.47\%, improving all-class recognition alongside client recognition floors.

\subsubsection{Ablation of SASP: Contributions of Support-Adaptive Mean Estimation and Local Class Priors}
The configurations in Table~\ref{tab:preview-sasp} vary mean adaptation and class priors while using the same CGRL backbone, MMS layer scales, and shared covariance metric.
Within each dataset block, the rows correspond to (i) shared means; (ii) direct local means; (iii) direct local means with a local class prior, as used in pFedFDA~\cite{pFedFDA_NeurIPS24}; (iv) count-based shrinkage of local means with a local class prior; and (v) count-based shrinkage of local means with equal class priors, constituting SASP.
Configurations (i) and (ii) also use equal class priors; all configurations retain shared means for classes absent from local training.

\begin{table*}[!t]
\centering
\caption{Backbone--classifier cross-comparison against strong baseline combinations across all heterogeneity settings. $B_0$: FedYoYo~\cite{FedYoYo_ICCV25}; $B_1$: the CGRL-trained FedGRC backbone. $H_0$: CCVR~\cite{CCVR_NeurIPS21} on CIFAR-100 Dir(0.3), and FedCGS~\cite{FedCGS_AAAI25} elsewhere; $H_1$: MMS+SASP. The signed values in \textit{parentheses} indicate gains over $B_0H_0$; bold entries identify full FedGRC. For either metric, $M_{ij}=M(B_i,H_j)$, and the shaded row reports the excess gain $I=M_{11}-M_{10}-M_{01}+M_{00}$.}
\label{tab:preview-cross}
\fontsize{7}{8.6}\selectfont
\renewcommand{\arraystretch}{1.25}
\settowidth{\incw}{($-$0.00)}
\def\crossnum#1#2{\makebox[1.3em][r]{#1}.\makebox[1.1em][l]{#2}}
\def\crosscell#1#2{#1\hspace{0.15em}\makebox[\incw][r]{(#2)}}
\def\crossbase#1{#1\hspace{0.15em}\makebox[\incw][r]{}}
\def\tabbody{%
\toprule
\multirow{2}{*}[-3pt]{\makecell[c]{Training\\backbone}} & \multirow{2}{*}[-3pt]{\makecell[c]{Classifier\\head}} & \multicolumn{2}{c}{\fontsize{8}{9.6}\selectfont CIFAR-10, Dir(0.3)} & \multicolumn{2}{c}{\fontsize{8}{9.6}\selectfont CIFAR-10, Dir(0.1)} & \multicolumn{2}{c}{\fontsize{8}{9.6}\selectfont CIFAR-100, Dir(0.3)} & \multicolumn{2}{c}{\fontsize{8}{9.6}\selectfont CIFAR-100, Dir(0.1)} & \multicolumn{2}{c}{\fontsize{8}{9.6}\selectfont FedRS} \\
\cmidrule(lr){3-4}\cmidrule(lr){5-6}\cmidrule(lr){7-8}\cmidrule(lr){9-10}\cmidrule(lr){11-12}
 & & \multicolumn{1}{c}{\fontsize{8}{9.6}\selectfont global} & \multicolumn{1}{c}{\fontsize{8}{9.6}\selectfont F@25} & \multicolumn{1}{c}{\fontsize{8}{9.6}\selectfont global} & \multicolumn{1}{c}{\fontsize{8}{9.6}\selectfont F@25} & \multicolumn{1}{c}{\fontsize{8}{9.6}\selectfont global} & \multicolumn{1}{c}{\fontsize{8}{9.6}\selectfont F@25} & \multicolumn{1}{c}{\fontsize{8}{9.6}\selectfont global} & \multicolumn{1}{c}{\fontsize{8}{9.6}\selectfont F@25} & \multicolumn{1}{c}{\fontsize{8}{9.6}\selectfont global} & \multicolumn{1}{c}{\fontsize{8}{9.6}\selectfont F@25} \\
\midrule
$B_0$ & $H_0$ & \crossbase{\crossnum{78}{59}} & \crossbase{\crossnum{69}{72}} & \crossbase{\crossnum{69}{03}} & \crossbase{\crossnum{61}{03}} & \crossbase{\crossnum{51}{56}} & \crossbase{\crossnum{41}{10}} & \crossbase{\crossnum{49}{72}} & \crossbase{\crossnum{38}{49}} & \crossbase{\crossnum{83}{75}} & \crossbase{\crossnum{77}{04}} \\
\addlinespace[3pt]
$B_0$ & $H_1$ & \crosscell{\crossnum{78}{55}}{\gm{$-$0.04}} & \crosscell{\crossnum{69}{32}}{\gm{$-$0.40}} & \crosscell{\crossnum{69}{84}}{\gp{$+$0.81}} & \crosscell{\crossnum{60}{38}}{\gm{$-$0.65}} & \crosscell{\crossnum{51}{05}}{\gm{$-$0.51}} & \crosscell{\crossnum{40}{98}}{\gm{$-$0.12}} & \crosscell{\crossnum{50}{00}}{\gp{$+$0.28}} & \crosscell{\crossnum{37}{61}}{\gm{$-$0.88}} & \crosscell{\crossnum{84}{65}}{\gp{$+$0.90}} & \crosscell{\crossnum{79}{57}}{\gp{$+$2.53}} \\
\addlinespace[5pt]
$B_1$ & $H_0$ & \crosscell{\crossnum{81}{17}}{\gp{$+$2.58}} & \crosscell{\crossnum{71}{39}}{\gp{$+$1.67}} & \crosscell{\crossnum{72}{09}}{\gp{$+$3.06}} & \crosscell{\crossnum{63}{47}}{\gp{$+$2.44}} & \crosscell{\crossnum{54}{45}}{\gp{$+$2.89}} & \crosscell{\crossnum{45}{07}}{\gp{$+$3.97}} & \crosscell{\crossnum{52}{08}}{\gp{$+$2.36}} & \crosscell{\crossnum{41}{07}}{\gp{$+$2.58}} & \crosscell{\crossnum{85}{69}}{\gp{$+$1.94}} & \crosscell{\crossnum{76}{88}}{\gm{$-$0.16}} \\
\addlinespace[3pt]
$\boldsymbol{B_1}$ & $\boldsymbol{H_1}$ & \crosscell{\crossnum{\textbf{80}}{\textbf{76}}}{\gp{$+$2.17}} & \crosscell{\crossnum{\textbf{71}}{\textbf{85}}}{\gp{$+$2.13}} & \crosscell{\crossnum{\textbf{73}}{\textbf{05}}}{\gp{$+$4.02}} & \crosscell{\crossnum{\textbf{63}}{\textbf{85}}}{\gp{$+$2.82}} & \crosscell{\crossnum{\textbf{55}}{\textbf{71}}}{\gp{$+$4.15}} & \crosscell{\crossnum{\textbf{47}}{\textbf{56}}}{\gp{$+$6.46}} & \crosscell{\crossnum{\textbf{53}}{\textbf{22}}}{\gp{$+$3.50}} & \crosscell{\crossnum{\textbf{41}}{\textbf{83}}}{\gp{$+$3.34}} & \crosscell{\crossnum{\textbf{87}}{\textbf{08}}}{\gp{$+$3.33}} & \crosscell{\crossnum{\textbf{81}}{\textbf{65}}}{\gp{$+$4.61}} \\
\addlinespace[3pt]
\midrule
\rowcolor{gray!12}
\multicolumn{2}{c}{Excess gain $I$} & \crossbase{\gm{\crossnum{-0}{37}}} & \crossbase{\gp{\crossnum{+0}{86}}} & \crossbase{\gp{\crossnum{+0}{15}}} & \crossbase{\gp{\crossnum{+1}{03}}} & \crossbase{\gp{\crossnum{+1}{77}}} & \crossbase{\gp{\crossnum{+2}{61}}} & \crossbase{\gp{\crossnum{+0}{86}}} & \crossbase{\gp{\crossnum{+1}{64}}} & \crossbase{\gp{\crossnum{+0}{49}}} & \crossbase{\gp{\crossnum{+2}{24}}} \\
\bottomrule}
\setlength{\tabcolsep}{0pt}
\sbox{\tabbox}{\begin{tabular}{cc rrrrrrrrrr}\tabbody\end{tabular}}
\setlength{\tabcolsep}{\dimexpr(\textwidth-\wd\tabbox)/24\relax}
\begin{tabular}{cc rrrrrrrrrr}\tabbody\end{tabular}
\end{table*}

\textbf{Shrinkage for Rare-Class Recognition.}
Using local means directly improves \textit{local} accuracy but severely degrades recognition of sparsely observed classes.
On CIFAR-100 Dir(0.3), replacing shared means with local estimates without shrinkage or a local prior reduces \textit{rare} accuracy from 57.83\% to 1.17\% and \textit{F@25} from 46.93\% to 0.00\%.
On FedRS, \textit{rare} accuracy falls from 87.11\% to 44.44\% and \textit{F@25} from 80.04\% to 9.53\%.
The corresponding \textit{R@25} scores of 0.00\% and 9.53\% identify rare-class recognition as the principal lower-tail weakness in these configurations.
Retaining shared means for absent classes leaves this rare-class weakness unresolved: sparsely observed classes are still represented by means estimated from limited local samples.

With equal priors, SASP restores \textit{rare} accuracy to 58.00\% on CIFAR-100 Dir(0.3) and 87.54\% on FedRS, with corresponding \textit{F@25} scores of 47.56\% and 81.65\%.
Proposition~\ref{prop:shrinkage} explains the role of support in the regularized mean-fitting objective: the shared reference has greater relative influence when local sample support is small.
Shrinkage therefore restrains updates based on sparse observations, addressing the rare-class degradation seen with direct local estimation.

\textbf{Mean Adaptation and Class Priors.}
With shrinkage retained, replacing local class priors with equal priors on CIFAR-100 Dir(0.3) improves \textit{unseen} and \textit{rare} accuracy by 1.33\% and 0.96\%, respectively, and raises \textit{F@25} from 44.33\% to 47.56\%, at a cost of 1.60\% in \textit{local} accuracy.
On FedRS, equal priors improve \textit{global}, \textit{unseen}, and \textit{rare} accuracy and raise \textit{F@25} by 0.24\%, while reducing \textit{local} accuracy by 0.49\%.
The improvements extend to unseen-class recognition among weaker clients, with \textit{U@25} increasing by 1.59\% and 2.09\% on CIFAR-100 and FedRS, respectively.
These results support separating the roles of class counts and priors: counts govern mean adaptation, while equal priors support broader recognition and stronger client floors.

In both settings, direct local means achieve higher \textit{unseen} accuracy, whereas SASP yields substantially higher \textit{rare} accuracy and \textit{F@25}.
Unseen-class predictions can change because updates to observed-class means alter competing scores, even when absent-class means remain shared.
For the same fitted statistics, shrinkage reduces the perturbation budget in Corollary~\ref{cor:preservation}.

\subsection{Performance Gains over Strong Baseline Combinations and Backbone--Classifier Synergy}
\label{sec:exp-cross}
Table~\ref{tab:preview-cross} uses the baseline combinations with the highest \textit{F@25} in Tables~\ref{tab:main-cifar10}--\ref{tab:main-fedrs} as references: the FedYoYo~\cite{FedYoYo_ICCV25} backbone ($B_0$) paired with CCVR~\cite{CCVR_NeurIPS21} on CIFAR-100 Dir(0.3) or FedCGS~\cite{FedCGS_AAAI25} elsewhere ($H_0$).
The corresponding FedGRC components are the CGRL-trained backbone ($B_1$) and MMS+SASP classifier ($H_1$).
Replacing the backbone, the classifier, or both separates their individual effects from the gains achieved through their combination.

The complete combination achieves higher \textit{F@25} than either single-component replacement in all five settings.
On CIFAR-100 Dir(0.3), replacing FedYoYo with the CGRL backbone raises \textit{F@25} by 3.97\% under the same CCVR classifier.
MMS+SASP then adds 2.49\% on the CGRL backbone, compared with a 0.12\% decrease on FedYoYo.
On FedRS, replacing only the backbone changes \textit{F@25} by $-0.16$\%, whereas replacing both components yields a 4.61\% gain.
These comparisons demonstrate that the gains from each component depend on its pairing with the other.

The excess gain $I$ measures how much the joint improvement exceeds the sum of the gains from the two individual replacements.
It is positive for \textit{F@25} in all five settings, ranging from 0.86\% to 2.61\%, and for \textit{global} accuracy in four settings.
Together, these results support jointly designing geometric representation learning and statistical personalization to strengthen client recognition floors.

\section{Conclusion}\label{sec:conclusion}
We revisited personalized federated learning from the perspective of global recognition and personalization floors, showing that high local accuracy alone can conceal weaknesses across classes and clients.
Our evaluation protocol exposes these weaknesses through complementary assessments of local adaptation, recognition across federation-known classes, and the weakest assessed abilities of poorly performing clients.
Guided by this perspective, FedGRC couples geometric representation consensus with multilayer statistical classification and support-adaptive class-mean shrinkage.
Extensive comparisons and ablations demonstrate improved global recognition and personalization floors across heterogeneous settings, with simultaneous gains in local adaptation on FedRS.
These findings support the joint design of representations and personalized classifiers for practical deployment.

Under extreme label heterogeneity, broader recognition gains accompany substantially lower local accuracy than the strongest local baselines, highlighting the remaining trade-off between generalization and local adaptation.
FedGRC retains a shared representation and covariance metric during personalization; its shrinkage weights are determined by local class sample counts and do not explicitly account for the magnitude of client-specific domain shifts.
We will explore personalization jointly guided by class support and distribution shift, alongside representation and metric learning robust to domain shift, to improve local adaptation while preserving broad recognition and personalization floors.

\putbib[references]
\clearpage
\end{bibunit}

\setcounter{page}{1}
\bibliographynamespace{appendix}
\begin{bibunit}[IEEEtran]
\applybibliographycontrol
\twocolumn[\hypertarget{fedgrc.appendix.start}{}\begin{center}
{\sffamily\LARGE\bfseries Supplementary Material\par}
\vspace{7pt}
{\sffamily\large Is High Local Accuracy Enough? Geometric Consensus with\\
Support-Adaptive Statistical Personalization for Federated Learning\par}
\end{center}\vspace{10pt}]
\markboth{}{}
\raggedbottom
\setcounter{equation}{0}
\renewcommand{\theequation}{S\arabic{equation}}
\renewcommand{\theHequation}{supp.\arabic{equation}}
\appendices
\section{\texorpdfstring{\raggedright Theoretical Analysis and Implementation Details}{Theoretical Analysis and Implementation Details}}\label{app:proof}
This supplement uses the notation introduced in Section~\ref{sec:prelim}.
Equations and theoretical results from the main article retain their original numbering; equations in this supplement use the prefix S.

\subsection{Proof of Lemma~\ref{lem:radial}}\label{app:radial}
Fix the class, layer, and round and use the notation of Lemma~\ref{lem:radial}.
Assume $p_\chi=\Pr(\chi=1)>0$ and a finite conditional second moment. With $\bar\rho_\chi=\mathbb E[\rho\mid\chi=1]$,
\begin{equation}
\begin{aligned}
 \mathbb E[\chi(\rho-a)^2]
 &=p_\chi\mathbb E[(\rho-a)^2\mid\chi=1]\\
 &=p_\chi\mathbb E[(\rho-\bar\rho_\chi+\bar\rho_\chi-a)^2\mid\chi=1]\\
 &=p_\chi\Var(\rho\mid\chi=1)+p_\chi(\bar\rho_\chi-a)^2.
\end{aligned}
\end{equation}
The cross term vanishes by conditional centering.
This proves Eq.~\eqref{eq:radial-decomposition}.
If $p_\chi=0$, the contribution is zero.
The same identity holds under the empirical distribution of a fixed batch.

\subsection{Complementary Feature Geometry}\label{app:gradients}
Fix one local step and treat all broadcast references as constants.
For a nonzero feature $\mathbf h$, put $r=\|\mathbf h\|_2$ and $\mathbf u=\mathbf h/r$.
The differential of normalization is
\begin{equation}
 \mathrm d\mathbf u
 =\frac{\mathrm d\mathbf h}{r}
 -\frac{\mathbf h\mathbf h^\top\mathrm d\mathbf h}{r^3}
 =\frac{\mathbf I-\mathbf u\mathbf u^\top}{r}\,\mathrm d\mathbf h.
\end{equation}
For $q_c(\mathbf u)=\exp(\mathbf u^\top\mathbf p_c/\tau)/\sum_r\exp(\mathbf u^\top\mathbf p_r/\tau)$ and $\ell_{\mathrm p}=-\log q_y(\mathbf u)$, differentiating the log-sum-exp gives
\begin{equation}
 \nabla_{\mathbf u}\ell_{\mathrm p}
 =\frac1\tau\Bigl(\sum_cq_c(\mathbf u)\mathbf p_c-\mathbf p_y\Bigr).
\end{equation}
The chain rule yields Eq.~\eqref{eq:prototype-gradient}.
The full directional loss depends on raw features only through their normalized versions.
Applying the chain rule to every occurrence of $\mathbf u_i$, including contrastive comparisons, gives
\begin{equation}
 \mathbf g_i=\frac{\mathbf I-\mathbf u_i\mathbf u_i^\top}{\|\mathbf h_i\|_2}
 \nabla_{\mathbf u_i}\mathcal L_{\mathrm{dir}}.
\end{equation}
Since $\mathbf h_i^\top(\mathbf I-\mathbf u_i\mathbf u_i^\top)=\mathbf0^\top$, the directional gradient is tangent to the sphere at $\mathbf h_i$.
For an isolated feature-space step of size $s$, this also gives
\begin{equation}
 \|\mathbf h_i-s\mathbf g_i\|_2^2
 =\|\mathbf h_i\|_2^2+s^2\|\mathbf g_i\|_2^2.
\end{equation}
Thus, the change in squared magnitude is second order in the step size.

For the radial term, $\rho=\tfrac12\log(\|\mathbf h\|_2^2+\varepsilon^2)$, hence
\begin{equation}
 \nabla_{\mathbf h}\rho
 =\frac{\mathbf h}{\|\mathbf h\|_2^2+\varepsilon^2},\;
 \nabla_{\mathbf h}[\chi(\rho-a)^2]
 =2\chi(\rho-a)\nabla_{\mathbf h}\rho.
\end{equation}
The indicator $\chi$ and reference $a$ are held fixed, matching the implementation's stop-gradient convention.
This yields Eq.~\eqref{eq:radial-gradient}.
The radial gradient parallels $\mathbf h$ and is orthogonal to the same final-layer feature's directional gradient.

The implementation clamps the normalization denominator from below; the unit-vector identities apply above this floor.
When the radial gate is active, the configured safe-norm threshold places the feature well above the directional normalization floor.
Through the backbone Jacobian $\mathbf J_i=\partial\mathbf h_i/\partial\vect\phi$, the corresponding parameter-gradient inner product becomes $\mathbf g_i^\top\mathbf J_i\mathbf J_i^\top\mathbf v_i$; the orthogonality above describes the feature-space gradients.

To relate this geometry to raw-feature scatter, let $\mathbf h_i=R_i\mathbf u_i$, $\mathbf h_j=R_j\mathbf u_j$, with $R_i,R_j>0$ and unit directions.
Expanding the squared distance gives
\begin{equation}
\begin{aligned}
 \|\mathbf h_i-\mathbf h_j\|_2^2
 &=R_i^2+R_j^2-2R_iR_j\mathbf u_i^\top\mathbf u_j\\
 &=(R_i-R_j)^2+2R_iR_j(1-\mathbf u_i^\top\mathbf u_j)\\
 &=(R_i-R_j)^2+R_iR_j\|\mathbf u_i-\mathbf u_j\|_2^2.
\end{aligned}
\end{equation}
This proves Eq.~\eqref{eq:angular-radial-distance}.

To derive Eq.~\eqref{eq:pairwise-scatter}, let $\{\mathbf h_i\}_{i=1}^{N_c}$ denote the class-$c$ features from all clients at layer $\ell$ under a common backbone.
Their shared class mean and covariance are
\begin{equation}
\begin{aligned}
 \bar{\mathbf h}_{\ell,c}&=\frac1{N_c}\sum_{i=1}^{N_c}\mathbf h_i,\\
 \widehat{\vect\Sigma}_{\ell,c}
 &=\frac1{N_c}\sum_{i=1}^{N_c}
 (\mathbf h_i-\bar{\mathbf h}_{\ell,c})
 (\mathbf h_i-\bar{\mathbf h}_{\ell,c})^\top.
\end{aligned}
\end{equation}
Using $\sum_i\mathbf h_i=N_c\bar{\mathbf h}_{\ell,c}$, expansion about the mean gives
\begin{equation}
\begin{aligned}
 \tr(\widehat{\vect\Sigma}_{\ell,c})
 &=\frac1{N_c}\sum_i\|\mathbf h_i-\bar{\mathbf h}_{\ell,c}\|_2^2\\
 &=\frac1{N_c}\sum_i\|\mathbf h_i\|_2^2
   -\|\bar{\mathbf h}_{\ell,c}\|_2^2.
\end{aligned}
\end{equation}
The pairwise sum has the same centered form:
\begin{equation}
\begin{aligned}
 \sum_{i,j}\|\mathbf h_i-\mathbf h_j\|_2^2
 &=2N_c\sum_i\|\mathbf h_i\|_2^2
   -2\Big\|\sum_i\mathbf h_i\Big\|_2^2\\
 &=2N_c^2\tr(\widehat{\vect\Sigma}_{\ell,c}).
\end{aligned}
\end{equation}
Dividing by $2N_c^2$ and substituting the angular--radial distance decomposition yields
\begin{equation}
\begin{aligned}
 \tr(\widehat{\vect\Sigma}_{\ell,c})
 ={}&\frac1{2N_c^2}\sum_{i,j}(R_i-R_j)^2\\
 &+\frac1{2N_c^2}\sum_{i,j}R_iR_j\|\mathbf u_i-\mathbf u_j\|_2^2.
\end{aligned}
\end{equation}
This proves Eq.~\eqref{eq:pairwise-scatter}.
The sum includes pairs from different clients and therefore captures both local dispersion and cross-client differences within the same class.
The radial contribution can remain nonzero even when all directions agree; the angular contribution also depends on the feature magnitudes through $R_iR_j$.

For MMS, all features above are recomputed under the frozen backbone $\vect\phi_T$.
For a fixed layer scale $s_\ell>0$, the scaled distance is $s_\ell^{-2}\|\mathbf h_i-\mathbf h_j\|_2^2$, and each nonzero raw feature's log norm shifts by $-\log s_\ell$.
Norm ratios, within-class raw log-norm variance, and differences between class means of raw log norms are therefore unchanged.
The layer-$\ell$ residual block for a sample of class $c$ is
\begin{equation}
 \mathbf e_{\ell i}^{(k)}
 =\frac{\mathbf h_\ell(x_i^{(k)};\vect\phi_T)-\bar{\mathbf h}_{\ell,c}}{s_\ell}.
\end{equation}
Taking the trace in Eq.~\eqref{eq:mms-covariance} and summing the squared norms of these blocks yields
\begin{equation}
\begin{aligned}
 \tr(\widehat{\vect\Sigma}_{\mathrm w})
 &=\frac1N\sum_{k,i}\sum_{\ell\in\Lambda}\|\mathbf e_{\ell i}^{(k)}\|_2^2\\
 &=\sum_{c=1}^{C}\frac{N_c}{N}\sum_{\ell\in\Lambda}
   \frac{\tr(\widehat{\vect\Sigma}_{\ell,c})}{s_\ell^2}.
\end{aligned}
\end{equation}
This proves Eq.~\eqref{eq:mms-scatter-trace}, linking class- and layer-specific scatter to the pooled within-class covariance used by MMS.
The trace characterizes total within-class variation; the full covariance additionally retains directional structure and cross-layer correlations.
These identities provide a geometric interpretation of CGRL: the magnitude loss regulates the log norms of views with $\chi_{\ell i}=1$, while directional supervision also promotes between-class separation.
CGRL optimizes these directional and magnitude objectives, which differ from covariance-trace minimization.

\subsection{Proof of Proposition~\ref{prop:scatter}}\label{app:scatter}
For $n_c^{(k)}>0$, define the local class covariance as
\begin{equation}
 \widehat{\vect\Sigma}_c^{(k)}
 =\frac1{n_c^{(k)}}\sum_{i\in\mathcal I_c^{(k)}}
 \big(\mathbf z(x_i^{(k)})-\vect\mu_c^{(k)}\big)
 \big(\mathbf z(x_i^{(k)})-\vect\mu_c^{(k)}\big)^\top.
\end{equation}
Let $\mathbf r_i^{(k)}=\mathbf z(x_i^{(k)})-\vect\mu_c^{(k)}$ within this class and $\vect\Delta_c^{(k)}=\vect\mu_c^{(k)}-\vect\mu_c$.
Since $\sum_{i\in\mathcal I_c^{(k)}}\mathbf r_i^{(k)}=\mathbf0$, the cross terms vanish in
\begin{equation}
\begin{aligned}
 &\sum_{i\in\mathcal I_c^{(k)}}
 (\mathbf z(x_i^{(k)})-\vect\mu_c)(\mathbf z(x_i^{(k)})-\vect\mu_c)^\top\\
 &\quad=\sum_{i\in\mathcal I_c^{(k)}}
 (\mathbf r_i^{(k)}+\vect\Delta_c^{(k)})
 (\mathbf r_i^{(k)}+\vect\Delta_c^{(k)})^\top\\
 &\quad=n_c^{(k)}\widehat{\vect\Sigma}_c^{(k)}
       +n_c^{(k)}\vect\Delta_c^{(k)}(\vect\Delta_c^{(k)})^\top.
\end{aligned}
\end{equation}
Summing over all nonempty groups and dividing by $N$ proves the stated decomposition.

\subsection{Regularized Metric Bounds}\label{app:conditioning}
Write $\widehat{\vect\Sigma}_{\mathrm w}=\mathbf Q\diag(\sigma_1,\ldots,\sigma_d)\mathbf Q^\top$ with $\sigma_j\ge0$.
Positive trace and $\gamma>0$ imply $\lambda_A=\gamma\sum_j\sigma_j/d>0$.
Consequently,
\begin{equation}
 \mathbf A
 =\mathbf Q\diag\!\left(\frac1{\sigma_1+\lambda_A},\ldots,
 \frac1{\sigma_d+\lambda_A}\right)\mathbf Q^\top
\end{equation}
is positive definite.
Its largest eigenvalue is $1/(\sigma_{\min}+\lambda_A)$ and its smallest is $1/(\sigma_{\max}+\lambda_A)$.
These eigenvalues give the operator-norm and condition-number identities in Eq.~\eqref{eq:metric-conditioning}.
Using $\sigma_{\min}\ge0$ and $\sigma_{\max}\le\sum_j\sigma_j$ gives
\begin{equation}
 \kappa_2(\mathbf A)
 \le 1+\frac{\sigma_{\max}}{\lambda_A}
 \le 1+\frac{\sum_j\sigma_j}{\gamma\sum_j\sigma_j/d}
 =1+\frac d\gamma.
\end{equation}
This bound covers singular empirical covariances and depends only on dimension and relative ridge strength.

\subsection{Construction from Additive Moments}\label{app:moments}
Let $\mathbf h_i^{(k)}$ be the concatenation of unscaled layer features extracted with the same frozen backbone.
For each client, retain the counts and moments
\begin{equation}\label{eq:client-feature-sums}
 \mathbf t_c^{(k)}=\sum_{i\in\mathcal I_c^{(k)}}\mathbf h_i^{(k)},
 \;
 \mathbf G^{(k)}=\sum_i\mathbf h_i^{(k)}(\mathbf h_i^{(k)})^\top.
\end{equation}
The additive quantities $\mathbf t_c=\sum_k\mathbf t_c^{(k)}$, $\mathbf G=\sum_k\mathbf G^{(k)}$, and $N_c=\sum_kn_c^{(k)}$ give
\begin{equation}\label{eq:pooled-feature-covariance}
 \bar{\mathbf h}=\frac1N\sum_c\mathbf t_c,
 \;
 \widehat{\vect\Sigma}_{\mathrm{tot}}
 =\frac{\mathbf G}{N}-\bar{\mathbf h}\bar{\mathbf h}^{\top}.
\end{equation}
Its diagonal layer blocks give the shared scales $s_\ell$ via Eq.~\eqref{eq:layer-variance}.
Set $\mathbf D_s=\operatorname{blockdiag}_{\ell\in\Lambda}(s_\ell^{-1}\mathbf I_{d_\ell})$.
Then
\begin{equation}\label{eq:statistics-from-sums}
\begin{aligned}
 \vect\mu_c&=\mathbf D_s\mathbf t_c/N_c,\\
 \vect\mu_c^{(k)}&=\mathbf D_s\mathbf t_c^{(k)}/n_c^{(k)}
                  \quad(n_c^{(k)}>0),\\
 \widehat{\vect\Sigma}_{\mathrm w}
 &=\frac1N\mathbf D_s\mathbf G\mathbf D_s
   -\sum_c\frac{N_c}{N}\vect\mu_c\vect\mu_c^\top.
\end{aligned}
\end{equation}
The last equality follows by expanding the residuals about their shared class means and using
\[
 \sum_{k,i:y_i^{(k)}=c}\mathbf z(x_i^{(k)})=N_c\vect\mu_c.
\]
This construction retains cross-layer covariance blocks and equals the pooled computation in exact arithmetic.

\subsection{Proof of Proposition~\ref{prop:shrinkage}}\label{app:shrinkage}
Fix $\mathbf A\succ0$ and $\vect\mu_c$ in the objective $J_c^{(k)}$.
For brevity, write $n=n_c^{(k)}$ and $\mathbf t=\sum_{i\in\mathcal I_c^{(k)}}\mathbf z(x_i^{(k)})$.
Differentiation gives
\begin{equation}
\begin{aligned}
 \nabla J_c^{(k)}(\mathbf m)
 &=2\mathbf A\big[(n+m_0)\mathbf m-\mathbf t-m_0\vect\mu_c\big],\\
 \nabla^2J_c^{(k)}(\mathbf m)&=2(n+m_0)\mathbf A\succ0.
\end{aligned}
\end{equation}
Strict convexity gives a unique minimizer.
Setting the gradient to zero yields
\begin{equation}
 \mathbf m^\star=\frac{\mathbf t+m_0\vect\mu_c}{n+m_0}.
\end{equation}
For $n>0$, use $\mathbf t=n\vect\mu_c^{(k)}$ to obtain
\begin{equation}
 \mathbf m^\star
 =\frac{m_0}{n+m_0}\vect\mu_c
 +\frac{n}{n+m_0}\vect\mu_c^{(k)}
 =\widetilde{\vect\mu}_c^{(k)}.
\end{equation}
For $n=0$, the data sum is empty, $\mathbf t=\mathbf0$, and $\mathbf m^\star=\vect\mu_c$.
For $n>0$, subtracting the local center gives
\begin{equation}
 \widetilde{\vect\mu}_c^{(k)}-\vect\mu_c^{(k)}
 =(1-\alpha_c^{(k)})(\vect\mu_c-\vect\mu_c^{(k)}).
\end{equation}
Positive homogeneity of the $\mathbf A$-norm yields Eq.~\eqref{eq:local-fit}.
For fixed centers, $1-\alpha(n)=m_0/(n+m_0)$ decreases with support, reducing the local-center residual.

\subsection{Statistical Head Construction}\label{app:head}
All argmax operations use the same fixed tie-breaking rule.
Expanding the personalized squared distance gives
\begin{equation}\label{eq:distance-expansion}
\begin{aligned}
 -\tfrac12\|\mathbf z-\widetilde{\vect\mu}_c^{(k)}\|_{\mathbf A}^2
 ={}&-\tfrac12\mathbf z^\top\mathbf A\mathbf z
 +(\widetilde{\vect\mu}_c^{(k)})^\top\mathbf A\mathbf z\\
 &-\tfrac12(\widetilde{\vect\mu}_c^{(k)})^\top\mathbf A\widetilde{\vect\mu}_c^{(k)}.
\end{aligned}
\end{equation}
Dropping the class-independent first term yields the statistical head in Eq.~\eqref{eq:sasp-head}.
Under $\mathbf z\mid y=c\sim\mathcal N(\widetilde{\vect\mu}_c^{(k)},\mathbf A^{-1})$ and equal priors $\pi_c=1/C$, the Bayes rule is
\begin{equation}\label{eq:balanced-decision}
 \argmax_c\frac1C\mathcal N(\mathbf z;\widetilde{\vect\mu}_c^{(k)},\mathbf A^{-1})
 =\argmax_c s_c^{(k)}(\mathbf z).
\end{equation}

\subsection{Proof of Corollary~\ref{cor:preservation} and Risk Bounds}\label{app:score}
Let $s_c^0(\mathbf z)=\vect\mu_c^\top\mathbf A\mathbf z-\tfrac12\vect\mu_c^\top\mathbf A\vect\mu_c$ denote the shared score.
The personalized displacement is $\vect\delta_c^{(k)}=\widetilde{\vect\mu}_c^{(k)}-\vect\mu_c$.
Substitution into Eq.~\eqref{eq:head-score} and subtraction of $s_c^0(\mathbf z)$ yield
\begin{equation}\label{eq:score-perturbation}
\begin{aligned}
 s_c^{(k)}(\mathbf z)-s_c^0(\mathbf z)
 ={}&(\vect\delta_c^{(k)})^\top\mathbf A(\mathbf z-\vect\mu_c)\\
 &-\tfrac12\|\vect\delta_c^{(k)}\|_{\mathbf A}^2.
\end{aligned}
\end{equation}
Cauchy--Schwarz in the $\mathbf A$-inner product implies
\begin{equation}
\begin{aligned}
 |s_c^{(k)}(\mathbf z)-s_c^0(\mathbf z)|&\le\epsilon_c^{(k)}(\mathbf z),\\
 \epsilon_c^{(k)}(\mathbf z)
 &=\|\vect\delta_c^{(k)}\|_{\mathbf A}\|\mathbf z-\vect\mu_c\|_{\mathbf A}
   +\tfrac12\|\vect\delta_c^{(k)}\|_{\mathbf A}^2.
\end{aligned}
\end{equation}
Substituting $\vect\delta_c^{(k)}=\alpha_c^{(k)}\vect\Delta_c^{(k)}$ into Eq.~\eqref{eq:score-perturbation} recovers Eq.~\eqref{eq:score-change}.
For a locally absent class, both the displacement and the score perturbation are zero.
For $c^\star=f^0(x)$, $\mathbf z=\mathbf z(x)$, and any $c\ne c^\star$, the adapted margin obeys
\begin{equation}
\begin{aligned}
 &s_{c^\star}^{(k)}(\mathbf z)-s_c^{(k)}(\mathbf z)\\
 &\quad\ge s_{c^\star}^0(\mathbf z)-s_c^0(\mathbf z)
 -\epsilon_{c^\star}^{(k)}(\mathbf z)-\epsilon_c^{(k)}(\mathbf z)\\
 &\quad\ge M^0(x)-B^{(k)}(x).
\end{aligned}
\end{equation}
If $M^0(x)>B^{(k)}(x)$, every adapted margin is positive, so $c^\star$ remains the unique maximizer, proving Eq.~\eqref{eq:margin-condition}.
Ties among the largest shared scores have $M^0(x)=0$ and fall within the complementary event.

\textbf{Risk Bounds.}
If $f^{(k)}(x)\ne y$, either $f^0(x)\ne y$ or adaptation changes a correct shared prediction.
Combining this observation with the contrapositive of Eq.~\eqref{eq:margin-condition} yields
\begin{equation}
\begin{aligned}
 &\mathbf1\{f^{(k)}(x)\ne y\}\\
 &\quad\le\mathbf1\{f^0(x)\ne y\}
   +\mathbf1\{f^{(k)}(x)\ne f^0(x)\}\\
 &\quad\le\mathbf1\{f^0(x)\ne y\}
   +\mathbf1\{M^0(x)\le B^{(k)}(x)\}
\end{aligned}
\end{equation}
for every labeled example.
Taking expectations over any labeled target distribution $R$ with input marginal $R_X$ gives
\begin{equation}\label{eq:risk-control}
 \mathcal E_R(f^{(k)})
 \le \mathcal E_R(f^0)
 +\Pr_{x\sim R_X}[M^0(x)\le B^{(k)}(x)].
\end{equation}
Setting $R=Q$ yields the bound for the class-balanced all-class target in Section~\ref{sec:problem}.
For class $c$,
\begin{equation}\label{eq:class-risk-control}
 e_c^{(k)}
 \le e_c^0+\Pr_{x\sim Q_c}[M^0(x)\le B^{(k)}(x)],
\end{equation}
where $e_c^0$ is the shared predictor's class-wise error.
Mixtures over locally rare or unseen classes yield the corresponding subset risk bounds.

\subsection{Reference Updates and Implementation Conventions}\label{app:updates}
Table~\ref{tab:fedgrc-parameters} summarizes FedGRC's parameters and modeling choices.
Reference initialization, statistics accumulation, and updates are detailed below.

\begin{table}[!t]
\centering
\caption{Key parameters and implementation details of FedGRC.}
\label{tab:fedgrc-parameters}
\footnotesize
\renewcommand{\arraystretch}{1.08}
\setlength{\tabcolsep}{2pt}
\begin{tabular}{@{}>{\raggedright\arraybackslash}p{0.28\columnwidth}>{\raggedright\arraybackslash}p{\dimexpr0.72\columnwidth-2\tabcolsep\relax}@{}}
\toprule
Configuration & Description \\
\midrule
\multicolumn{2}{@{}l}{\grp{A. Representation}}\\*
ResNet-18 & Trained from scratch; no pretrained weights. \\
$\Lambda=\{2,3,4\}$ & Pooled dimensions: 128, 256, 512;\newline concatenated dimension $d=896$. \\
\addlinespace[4pt]
\multicolumn{2}{@{}l}{\grp{B. Consensus-Guided Representation Learning (CGRL)}}\\*
$\lambda_{\mathrm{dir}}=\lambda_{\mathrm{rad}}=1$ & Directional and magnitude guidance weights. \\
$\beta=1$ & Shared-prototype weight in the directional loss. \\
$\tau=0.1$ & Shared temperature for supervised\newline contrastive and prototype guidance. \\
$T_w=50$ & Warm-up rounds for directional\newline and magnitude guidance losses. \\
$\xi=\zeta=0.5$ & New-observation weights for directional\newline (normalized) and radial EMAs. \\
$\varepsilon=10^{-6}$ & Stability constant for the safe feature norm. \\
$r_{\min}=10^{-3}$ & Safe-norm threshold for magnitude guidance. \\
\addlinespace[4pt]
\multicolumn{2}{@{}l}{\grp{C. Multilayer Metric Synthesis (MMS)}}\\*
$\gamma=0.1$ & Shared ridge coefficient: $\lambda_A=\gamma\operatorname{tr}(\widehat{\vect\Sigma}_{\mathrm w})/d$. \\
\addlinespace[4pt]
\multicolumn{2}{@{}l}{\grp{D. Support-Adaptive Shrinkage Personalization (SASP)}}\\*
$m_0=64$ & Shared-reference strength,\newline fixed across clients and classes. \\
$\pi_c=1/C$ & Equal priors; no local class-frequency bias. \\
\bottomrule
\end{tabular}
\end{table}

The schedule $w_t=\min\{(t+1)/T_w,1\}$ gradually introduces the auxiliary terms over $T_w$ warm-up rounds.
Each participating client performs $E^{(k)}$ mini-batch updates of $\vect\theta=(\vect\phi,\vect\psi)$, and the server aggregates
\begin{equation}\label{eq:model-aggregation}
 \vect\theta_{t+1}=\sum_{k\in\mathcal S_t}\omega_t^{(k)}\vect\theta^{(k)},
 \; \omega_t^{(k)}=\frac{N^{(k)}}{\sum_{j\in\mathcal S_t}N^{(j)}}.
\end{equation}

After local training, a pass with the updated backbone $\vect\phi^{(k)}$ in evaluation mode and without random augmentation yields a direction for each locally observed class ($n_c^{(k)}>0$):
\begin{equation}\label{eq:direction-update}
 \bar{\mathbf u}_c^{(k)}
 =\unit{\frac1{n_c^{(k)}}\sum_{(x,c)\in\mathcal D_c^{(k)}}
 \unit{\mathbf h_L(x;\vect\phi^{(k)})}}.
\end{equation}
For a class observed in round $t$, count-weighted aggregation and normalized averaging give
\begin{subequations}\label{eq:direction-ema}
\begin{align}
 \widehat{\mathbf p}_c(t)
 &=\unit{\sum_{k\in\mathcal S_t:n_c^{(k)}>0}n_c^{(k)}\bar{\mathbf u}_c^{(k)}},\\
 \mathbf p_c(t+1)
 &=\unit{(1-\xi)\mathbf p_c(t)+\xi\widehat{\mathbf p}_c(t)}.
\end{align}
\end{subequations}
Here $\xi$ is the weight assigned to the current round's direction estimate; this update retains a common orientation while moderating changes between rounds.

Each class direction stays zero until the class is first observed; prototype loss starts only after the first aggregation.
Once enabled, prototype cross-entropy uses all class references, assigning zero logits to zero references and including samples with zero target references.
Classes absent from a round retain their previous direction.

Radial statistics are accumulated during local training.
Let $\mathcal Z_t^{(k)}$ be the multiset of all training occurrences processed by client $k$ in round $t$, counting repeated occurrences separately.
For each occurrence $i$, $\rho_\ell(x_i)$ is evaluated before its associated parameter update.
Define
\begin{equation}\label{eq:radial-statistics}
\begin{aligned}
 H_{\ell,c}^{(k)}(t)&=\sum_{i\in\mathcal Z_t^{(k)}:y_i=c}\rho_\ell(x_i),\\
 q_c^{(k)}(t)&=\sum_{i\in\mathcal Z_t^{(k)}}\mathbf1\{y_i=c\}.
\end{aligned}
\end{equation}
These sums and counts are accumulated locally during training and used to update the shared magnitude references.
For $q_c(t)=\sum_{k\in\mathcal S_t}q_c^{(k)}(t)>0$, the pooled log-norm mean is
\begin{equation}
 \widehat a_{\ell,c}(t)
 =\frac{\sum_{k\in\mathcal S_t}H_{\ell,c}^{(k)}(t)}{q_c(t)}.
\end{equation}
All training occurrences contribute to the reference statistics, whether or not the radial penalty is applied.
For a class observed both previously and in the current round, the reference update assigns weight $\zeta$ to the current estimate.
With observation flag $v_c(t)$, the complete update is
\begin{equation}\label{eq:radial-update}
\begin{aligned}
 &a_{\ell,c}(t+1)\\
 &\quad=
 \begin{cases}
 a_{\ell,c}(t),&q_c(t)=0,\\
 \widehat a_{\ell,c}(t),&q_c(t)>0,\ v_c(t)=0,\\
 (1-\zeta)a_{\ell,c}(t)+\zeta\widehat a_{\ell,c}(t),
     &\text{otherwise},
 \end{cases}\\
 &v_c(t+1)=v_c(t)\lor\mathbf1\{q_c(t)>0\}.
\end{aligned}
\end{equation}
Both references and flags are broadcast for the next round.

The occurrence counts $q_c^{(k)}(t)$ are distinct from the training-set counts $n_c^{(k)}$ used by SASP.
Terminal statistics are collected from all registered clients using the final post-aggregation backbone, with augmentation disabled and normalization statistics fixed.

\begingroup
\makeatletter
\setlength{\@dblfpsep}{10pt}
\makeatother
\section{Supplementary Experiments}\label{app:more}
This appendix details the dataset partitions and client-level data distributions, then presents further analysis and results on personalization floors and resource overhead.

\subsection{Datasets and Client Distributions}\label{app:datasets}
\begin{itemize}
\item \textbf{CIFAR-10 and CIFAR-100.}
CIFAR-10 and CIFAR-100~\cite{CIFAR_TR09} each contain 60,000 images spanning 10 and 100 classes, respectively.
We pool the original training and test sets and reserve 6,000 class-balanced images as a shared global test set, yielding 600 images per class for CIFAR-10 and 60 for CIFAR-100.
The remaining 54,000 images are distributed among 20 clients using Dirichlet label partitions~\cite{Dirichlet_arXiv19} with concentration parameters 0.3 and 0.1; each client's data are then randomly split into approximately 75\% training and 25\% local test samples.
Figs.~\ref{fig:app-cifar10-distributions} and~\ref{fig:app-cifar100-distributions} show the resulting client--class distributions.
Both datasets exhibit unequal class counts and incomplete class coverage across clients, with Dir(0.1) producing more concentrated distributions and more missing classes than Dir(0.3).

\item \textbf{FedRS.}
FedRS~\cite{FedRSBench_FedKDD25} integrates NaSC-TG2~\cite{NaSCTG22021}, WHU-RS19~\cite{WHURS192010,WHURS192011}, EuroSAT~\cite{EuroSAT2019}, AID~\cite{AID2017}, NWPU-RESISC45~\cite{NWPURESISC452017}, UCM~\cite{UCM2010}, OPTIMAL-31~\cite{OPTIMAL312019}, and RSD46-WHU~\cite{RSD46WHU2017a,RSD46WHU2017b} into a unified label space of 15 remote-sensing scene classes.
We use its NIID-1 partition, comprising 135 clients, each drawing images from a single source dataset.
The 83,548 images held by clients are split by class within each client into approximately 75\% training and 25\% local test samples, retaining training examples for every observed class.
The shared global test set contains 940 images from the benchmark's $T_B$ set and covers all 15 classes.
Fig.~\ref{fig:app-fedrs-niid1} shows differences in class coverage and sample counts across clients, with colors identifying the eight source datasets.
These sources span different sensors, spatial resolutions, and geographic regions, introducing domain heterogeneity alongside label and sample-size imbalance.
\end{itemize}

\begin{figure*}[!tp]
\centering
\includegraphics[width=0.96\textwidth]{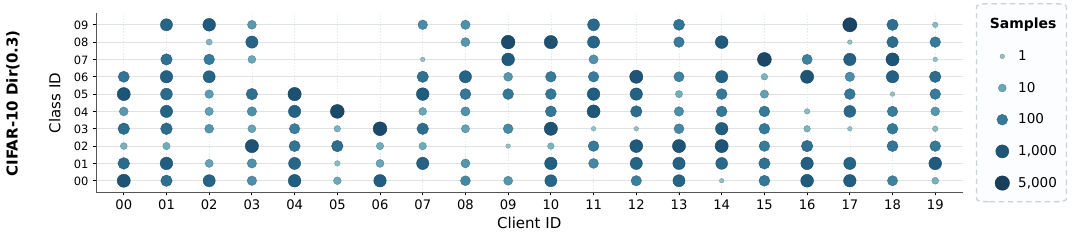}
\par\vspace{6pt}
\includegraphics[width=0.96\textwidth]{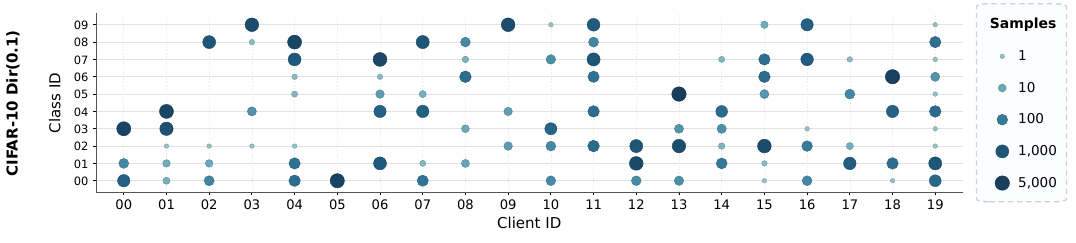}
\caption{Heterogeneity in client--class distributions of training and local test samples on CIFAR-10~\cite{CIFAR_TR09} under Dir(0.3) and Dir(0.1). Blank entries indicate zero samples.}
\label{fig:app-cifar10-distributions}
\end{figure*}

\begin{figure*}[!tp]
\centering
\includegraphics[width=\textwidth]{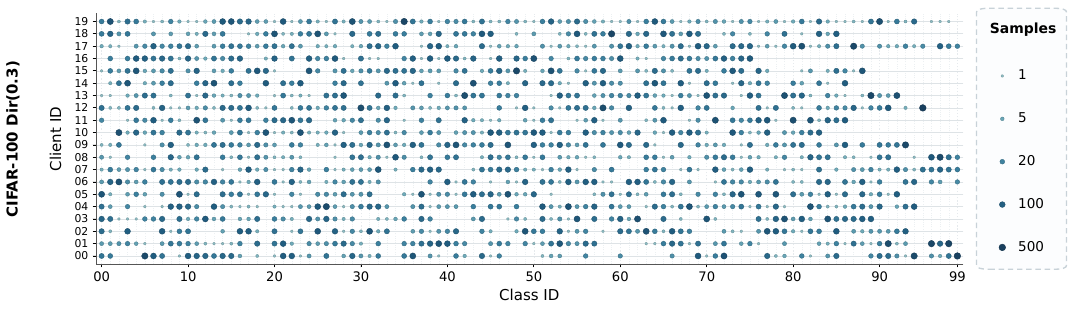}
\par\vspace{6pt}
\includegraphics[width=\textwidth]{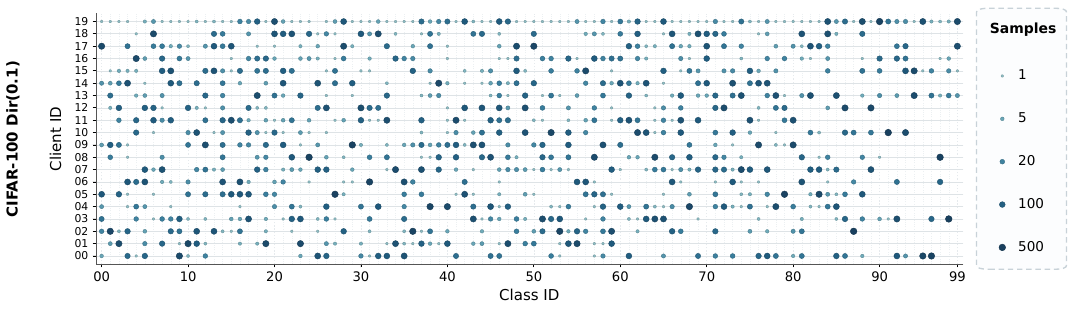}
\caption{Heterogeneity in client--class distributions of training and local test samples on CIFAR-100~\cite{CIFAR_TR09} under Dir(0.3) and Dir(0.1). Blank entries indicate zero samples.}
\label{fig:app-cifar100-distributions}
\end{figure*}

\begin{figure*}[!tp]
\centering
\includegraphics[width=\textwidth]{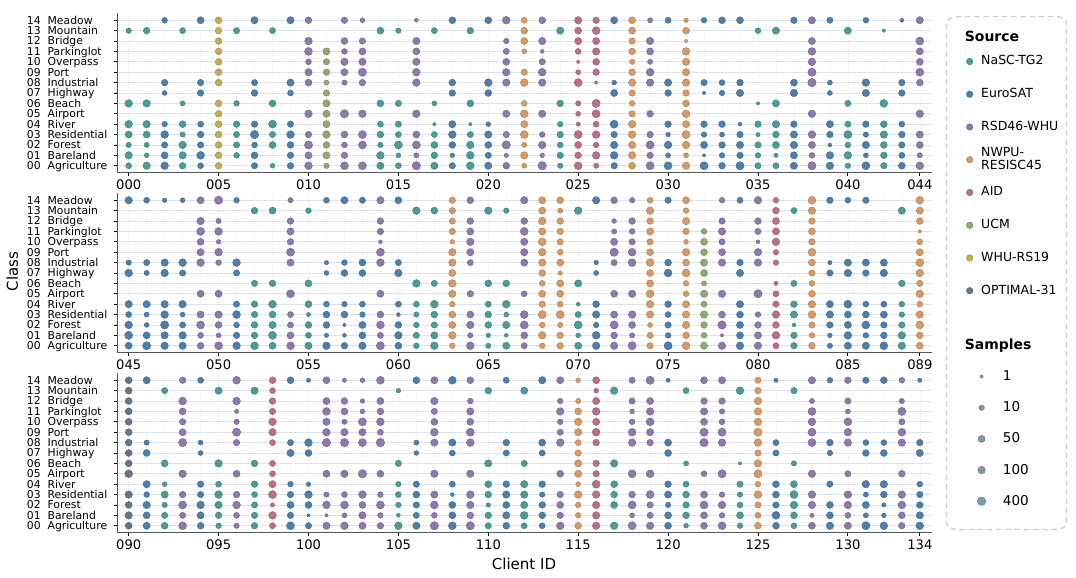}
\caption{Heterogeneity in client--class distributions of training and local test samples on FedRS~\cite{FedRSBench_FedKDD25} under NIID-1. Blank entries indicate zero samples.}
\label{fig:app-fedrs-niid1}
\end{figure*}

\subsection{Personalization Floors: Analysis and Results}\label{app:floor-results}
\textbf{Client-Wise Minima before Aggregation.}
The order of aggregation determines which client weaknesses a recognition-floor metric captures.
Using the client-level minima $F^{(k)}$ from Section~\ref{sec:eval-protocol}, consider the two summaries
\begin{gather*}
 \textit{F@25}=\operatorname{Tail}_{25}\bigl((F^{(k)})_{k=1}^{K}\bigr),\\
 \min\{\textit{L@25},\textit{U@25},\textit{R@25}\}.
\end{gather*}
The first ranks clients by their lowest accuracy across the assessed abilities and averages the lowest scores; the second takes the minimum of three ability-specific tail averages.
They need not coincide because the weakest ability can differ across clients.
For example, some clients may struggle with unseen classes while others struggle with rare classes.
The client group used for \textit{F@25} can include both types according to their weakest scores, whereas the minimum of the three tail averages summarizes a single ability.
Thus, \textit{F@25} captures a personalization floor across abilities and clients, while \textit{L@25}, \textit{U@25}, and \textit{R@25} characterize weaknesses in individual abilities.
These views are complementary.
Clients without unseen classes contribute to \textit{F@25} through their local and rare-class accuracies; an ability with no participating clients is reported as n/a.

\textbf{Additional Results on CIFAR.}
Tables~\ref{tab:pfp-cifar10} and~\ref{tab:pfp-cifar100} extend the lower-tail evaluation to CIFAR-10 and CIFAR-100, using the same comparison categories and protocol as Section~\ref{sec:eval-protocol}.
FedGRC achieves the highest \textit{U@25}, \textit{R@25}, and \textit{F@25} in all 4 CIFAR heterogeneity settings.
Its gains over the best \textit{F@25} baseline are 2.13\% and 2.82\% on CIFAR-10, and 6.46\% and 3.34\% on CIFAR-100, for Dir(0.3) and Dir(0.1), respectively.
Under Dir(0.3), it also leads in \textit{L@25}; under Dir(0.1), its \textit{L@25} trails the best baseline by 13.57\% on CIFAR-10 and 3.59\% on CIFAR-100.
These results show stronger recognition floors alongside a trade-off in local adaptation under extreme label heterogeneity.

\begin{table*}[!tp]
\centering
\caption{Client recognition floors on CIFAR-10~\cite{CIFAR_TR09}. L@25, U@25, R@25, and F@25 average the lowest 25\% (5 of 20 clients), with clients ranked separately for each metric.}
\label{tab:pfp-cifar10}
\footnotesize
\renewcommand{\arraystretch}{1.05}
\def\tabbody{%
\toprule
\multirow{2}{*}{\footnotesize Method} & \multirow{2}{*}{\footnotesize Venue} & \multicolumn{4}{c}{CIFAR-10, Dir(0.3), 20 clients} &  & \multicolumn{4}{c}{CIFAR-10, Dir(0.1), 20 clients} \\
\cmidrule(lr){3-6}\cmidrule(lr){8-11}
& & L@25 & U@25 & R@25 & F@25 &  & L@25 & U@25 & R@25 & F@25 \\
\midrule
\multicolumn{11}{@{}l}{\grp{Complete FL methods (method-specific training and classifiers)}} \\
FedAvg~\cite{FedAvg_AISTATS17} & AISTATS'17 & \ms{52.25}{0.25} & \ms{47.32}{1.08} & \ms{40.81}{2.39} & \ms{40.40}{1.94} &  & \ms{21.87}{0.96} & \ms{35.88}{0.61} & \ms{34.51}{0.08} & \ms{21.84}{0.92} \\
FedRep~\cite{FedRep_ICML21} & ICML'21 & \ms{69.54}{0.51} & \ms{0.00}{0.00} & \ms{0.24}{0.10} & \ms{0.00}{0.00} &  & \ms{78.65}{0.56} & \ms{0.00}{0.00} & \ms{0.05}{0.04} & \ms{0.00}{0.00} \\
FedROD~\cite{FedROD_ICLR22} & ICLR'22 & \ms{72.43}{0.27} & \ms{53.76}{0.64} & \ms{43.20}{1.02} & \ms{41.81}{0.58} &  & \ms{79.87}{0.32} & \ms{38.19}{0.63} & \ms{23.87}{1.14} & \ms{22.73}{1.21} \\
FedLC~\cite{FedLC_ICML22} & ICML'22 & \ms{47.91}{0.62} & \ms{54.33}{0.80} & \ms{41.37}{0.33} & \ms{39.19}{0.32} &  & \ms{15.95}{0.77} & \ms{33.69}{0.54} & \ms{15.79}{1.04} & \ms{10.48}{1.09} \\
FedTHE~\cite{FedTHE_ICLR23} & ICLR'23 & \ms{64.21}{1.08} & \ms{50.12}{1.27} & \ms{38.96}{1.60} & \ms{38.81}{1.55} &  & \ms{72.44}{1.82} & \ms{35.60}{0.65} & \ms{28.31}{0.48} & \ms{27.68}{0.52} \\
FedETF~\cite{FedETF_ICCV23} & ICCV'23 & \ms{61.76}{0.43} & \ms{64.76}{0.44} & \ms{54.48}{0.37} & \ms{52.88}{0.24} &  & \ms{39.26}{2.76} & \ms{47.01}{1.12} & \ms{42.40}{1.67} & \ms{33.84}{4.61} \\
FedFA~\cite{FedFA_TMC24} & TMC'24 & \ms{24.29}{5.90} & \ms{20.43}{6.52} & \ms{20.94}{10.77} & \ms{14.47}{6.92} &  & \ms{1.18}{1.14} & \ms{5.61}{1.46} & \ms{2.53}{2.03} & \ms{0.35}{0.31} \\
pFedFDA~\cite{pFedFDA_NeurIPS24} & NeurIPS'24 & \ms{71.89}{0.72} & \ms{2.07}{0.64} & \ms{4.60}{0.24} & \ms{1.70}{0.55} &  & \ms{79.39}{0.42} & \ms{0.01}{0.01} & \ms{0.42}{0.14} & \ms{0.00}{0.01} \\
FedVLS~\cite{FedVLS_AAAI25} & AAAI'25 & \ms{47.17}{0.64} & \ms{55.66}{1.26} & \ms{40.07}{0.64} & \ms{37.43}{1.04} &  & \ms{21.26}{0.57} & \ms{35.76}{0.32} & \ms{18.96}{0.26} & \ms{14.10}{0.47} \\
FedDr+~\cite{FedDrPlus_TMLR25} & TMLR'25 & \ms{61.86}{1.35} & \ms{57.95}{2.56} & \ms{51.75}{2.21} & \ms{51.22}{2.05} &  & \ms{3.66}{1.83} & \ms{11.12}{3.83} & \ms{6.27}{6.37} & \ms{1.66}{1.56} \\
BTFL~\cite{BTFL_KDD25} & KDD'25 & \ms{72.00}{0.38} & \ms{18.50}{1.67} & \ms{31.78}{1.22} & \ms{18.50}{1.67} &  & \mss{80.82}{1.09} & \ms{0.84}{0.44} & \ms{12.94}{1.20} & \ms{0.84}{0.44} \\
HPFL~\cite{HPFL_ICLR25} & ICLR'25 & \ms{72.64}{0.69} & \ms{30.29}{2.18} & \ms{19.97}{2.27} & \ms{17.68}{1.64} &  & \msb{80.95}{0.85} & \ms{17.59}{2.33} & \ms{11.63}{2.32} & \ms{10.86}{2.31} \\
FedYoYo~\cite{FedYoYo_ICCV25} & ICCV'25 & \ms{72.35}{0.39} & \ms{73.15}{1.52} & \ms{68.52}{1.13} & \ms{68.18}{1.05} &  & \ms{30.62}{3.22} & \ms{44.49}{2.17} & \ms{47.46}{1.01} & \ms{28.23}{3.75} \\
FedLog~\cite{FedLog_TMLR26} & TMLR'26 & \ms{66.30}{0.22} & \ms{0.00}{0.00} & \ms{0.19}{0.10} & \ms{0.00}{0.00} &  & \ms{78.46}{0.32} & \ms{0.00}{0.00} & \ms{0.09}{0.04} & \ms{0.00}{0.00} \\
pFedGM~\cite{pFedGM_arXiv26} & arXiv'26 & \ms{73.31}{0.59} & \ms{0.00}{0.00} & \ms{1.92}{0.07} & \ms{0.00}{0.00} &  & \ms{77.80}{0.48} & \ms{0.00}{0.00} & \ms{0.08}{0.06} & \ms{0.00}{0.00} \\
\addlinespace[2pt]
\multicolumn{11}{@{}l}{\grp{Statistical classifiers assembled on a vanilla backbone (FedAvg~+)}} \\
+ CCVR~\cite{CCVR_NeurIPS21} & NeurIPS'21 & \ms{55.58}{0.21} & \ms{55.37}{1.57} & \ms{49.94}{1.73} & \ms{47.99}{1.90} &  & \ms{40.89}{0.96} & \ms{46.76}{0.36} & \ms{44.79}{0.84} & \ms{39.17}{0.88} \\
+ FedNCM~\cite{FedNCM_NeurIPS23} & NeurIPS'23 & \ms{55.25}{0.78} & \ms{56.98}{1.00} & \ms{49.31}{1.81} & \ms{48.51}{1.61} &  & \ms{42.28}{0.90} & \ms{46.98}{0.34} & \ms{45.08}{0.67} & \ms{40.60}{1.08} \\
+ Fed3R~\cite{Fed3R_ICML24} & ICML'24 & \ms{54.96}{1.03} & \ms{54.44}{1.54} & \ms{48.32}{2.53} & \ms{44.84}{2.27} &  & \ms{43.41}{0.64} & \ms{48.12}{0.87} & \ms{47.20}{0.88} & \ms{41.35}{0.57} \\
+ FedCGS~\cite{FedCGS_AAAI25} & AAAI'25 & \ms{56.52}{0.47} & \ms{57.12}{0.90} & \ms{52.03}{1.49} & \ms{50.80}{1.64} &  & \ms{46.37}{0.39} & \ms{51.23}{0.45} & \ms{47.98}{1.00} & \ms{44.89}{0.34} \\
+ FedCOF~\cite{FedCOF_NeurIPS25} & NeurIPS'25 & \ms{48.01}{0.90} & \ms{45.63}{0.68} & \ms{41.22}{0.88} & \ms{39.53}{0.22} &  & \ms{36.60}{1.07} & \ms{41.61}{0.24} & \ms{40.05}{0.87} & \ms{34.66}{1.32} \\
+ GH-LDA~\cite{GHLDA_ICLR26} & ICLR'26 & \ms{56.74}{0.66} & \ms{56.39}{1.22} & \ms{52.13}{1.79} & \ms{50.63}{1.83} &  & \ms{46.48}{0.16} & \ms{51.34}{0.73} & \ms{48.74}{1.08} & \ms{45.24}{0.41} \\
\addlinespace[2pt]
\multicolumn{11}{@{}l}{\grp{Statistical classifiers assembled on the SOTA backbone (FedYoYo~+)}} \\
+ CCVR~\cite{CCVR_NeurIPS21} & NeurIPS'21 & \ms{71.87}{0.70} & \ms{73.68}{0.59} & \ms{67.06}{0.44} & \ms{65.89}{0.93} &  & \ms{56.40}{0.96} & \ms{61.56}{0.83} & \ms{59.17}{0.77} & \ms{53.25}{1.32} \\
+ FedNCM~\cite{FedNCM_NeurIPS23} & NeurIPS'23 & \ms{70.83}{0.38} & \ms{72.17}{0.71} & \ms{66.89}{0.82} & \ms{64.72}{1.11} &  & \ms{53.01}{0.91} & \ms{56.18}{0.18} & \ms{55.62}{0.57} & \ms{50.43}{0.72} \\
+ Fed3R~\cite{Fed3R_ICML24} & ICML'24 & \ms{72.70}{0.63} & \ms{74.27}{0.50} & \ms{66.98}{0.62} & \ms{65.10}{0.38} &  & \ms{56.56}{0.65} & \ms{62.74}{0.94} & \ms{60.93}{0.92} & \ms{52.49}{2.01} \\
+ FedCGS~\cite{FedCGS_AAAI25} & AAAI'25 & \mss{73.57}{0.66} & \ms{75.36}{0.17} & \mss{70.75}{0.33} & \mss{69.72}{0.69} &  & \ms{62.76}{0.28} & \mss{65.97}{0.48} & \mss{63.63}{0.54} & \mss{61.03}{0.14} \\
+ FedCOF~\cite{FedCOF_NeurIPS25} & NeurIPS'25 & \ms{69.86}{0.55} & \ms{69.96}{1.06} & \ms{65.20}{0.53} & \ms{62.47}{0.92} &  & \ms{53.23}{0.32} & \ms{57.90}{0.83} & \ms{57.15}{0.75} & \ms{50.88}{1.22} \\
+ GH-LDA~\cite{GHLDA_ICLR26} & ICLR'26 & \ms{73.50}{0.60} & \mss{75.45}{0.28} & \ms{70.71}{0.10} & \ms{69.63}{0.53} &  & \ms{61.20}{0.92} & \ms{64.92}{0.45} & \ms{62.83}{0.39} & \ms{59.40}{0.58} \\
\midrule
\rowcolor{gray!12}
\textbf{FedGRC (Ours)} & -- & \msb{76.97}{0.56} & \msb{79.22}{0.80} & \msb{72.78}{0.60} & \msb{71.85}{0.66} &  & \ms{67.38}{0.37} & \msb{71.20}{0.32} & \msb{65.95}{0.35} & \msb{63.85}{0.13} \\
$\Delta$ vs.\ Best &  & \dltp{$+$3.40} & \dltp{$+$3.77} & \dltp{$+$2.03} & \dltp{$+$2.13} &  & \dltm{$-$13.57} & \dltp{$+$5.23} & \dltp{$+$2.32} & \dltp{$+$2.82} \\
\bottomrule}%
\setlength{\tabcolsep}{0pt}%
\sbox{\tabbox}{\begin{tabular}{ll rrrr @{\hskip2\tabcolsep}c@{\hskip0.4\tabcolsep} rrrr}\tabbody\end{tabular}}%
\ifdim\wd\tabbox<\textwidth
  \setlength{\tabcolsep}{\dimexpr(\textwidth-\wd\tabbox)/22\relax}%
\else
  \setlength{\tabcolsep}{1pt}%
\fi
\begin{tabular}{ll rrrr @{\hskip2\tabcolsep}c@{\hskip0.4\tabcolsep} rrrr}\tabbody\end{tabular}
\end{table*}

\clearpage
\begin{table*}[!tp]
\centering
\caption{Client recognition floors on CIFAR-100~\cite{CIFAR_TR09}. L@25, U@25, R@25, and F@25 average the lowest 25\% (5 of 20 clients), with clients ranked separately for each metric.}
\label{tab:pfp-cifar100}
\footnotesize
\renewcommand{\arraystretch}{1.05}
\def\tabbody{%
\toprule
\multirow{2}{*}{\footnotesize Method} & \multirow{2}{*}{\footnotesize Venue} & \multicolumn{4}{c}{CIFAR-100, Dir(0.3), 20 clients} &  & \multicolumn{4}{c}{CIFAR-100, Dir(0.1), 20 clients} \\
\cmidrule(lr){3-6}\cmidrule(lr){8-11}
& & L@25 & U@25 & R@25 & F@25 &  & L@25 & U@25 & R@25 & F@25 \\
\midrule
\multicolumn{11}{@{}l}{\grp{Complete FL methods (method-specific training and classifiers)}} \\
FedAvg~\cite{FedAvg_AISTATS17} & AISTATS'17 & \ms{32.40}{0.99} & \ms{30.05}{1.23} & \ms{23.28}{4.25} & \ms{23.00}{4.13} &  & \ms{27.08}{0.15} & \ms{30.65}{0.11} & \ms{17.11}{1.83} & \ms{17.11}{1.83} \\
FedRep~\cite{FedRep_ICML21} & ICML'21 & \ms{33.75}{0.68} & \ms{0.00}{0.00} & \ms{0.00}{0.00} & \ms{0.00}{0.00} &  & \ms{44.22}{0.56} & \ms{0.00}{0.00} & \ms{0.00}{0.00} & \ms{0.00}{0.00} \\
FedROD~\cite{FedROD_ICLR22} & ICLR'22 & \ms{44.88}{0.67} & \ms{31.95}{0.76} & \ms{23.67}{1.77} & \ms{23.67}{1.77} &  & \msb{52.59}{1.26} & \ms{32.01}{0.95} & \ms{18.22}{1.80} & \ms{18.22}{1.80} \\
FedLC~\cite{FedLC_ICML22} & ICML'22 & \ms{29.55}{0.42} & \ms{29.94}{0.04} & \ms{20.50}{1.53} & \ms{20.50}{1.53} &  & \ms{23.55}{0.66} & \ms{30.29}{0.09} & \ms{16.89}{1.11} & \ms{16.72}{0.81} \\
FedTHE~\cite{FedTHE_ICLR23} & ICLR'23 & \ms{33.10}{0.28} & \ms{25.61}{0.47} & \ms{19.06}{0.59} & \ms{18.94}{0.74} &  & \ms{39.11}{0.10} & \ms{26.07}{0.11} & \ms{14.78}{1.55} & \ms{14.78}{1.55} \\
FedETF~\cite{FedETF_ICCV23} & ICCV'23 & \ms{35.26}{0.64} & \ms{32.99}{0.68} & \ms{24.78}{2.87} & \ms{24.02}{3.05} &  & \ms{31.82}{1.61} & \ms{35.53}{0.56} & \ms{22.17}{1.43} & \ms{22.17}{1.43} \\
FedFA~\cite{FedFA_TMC24} & TMC'24 & \ms{4.67}{2.89} & \ms{4.43}{1.74} & \ms{3.78}{4.55} & \ms{2.41}{2.54} &  & \ms{0.58}{0.38} & \ms{1.71}{0.79} & \ms{0.00}{0.00} & \ms{0.00}{0.00} \\
pFedFDA~\cite{pFedFDA_NeurIPS24} & NeurIPS'24 & \ms{42.39}{0.62} & \ms{12.00}{0.03} & \ms{10.83}{2.58} & \ms{9.62}{1.88} &  & \ms{49.57}{1.02} & \ms{10.93}{0.82} & \ms{7.17}{0.44} & \ms{7.17}{0.44} \\
FedVLS~\cite{FedVLS_AAAI25} & AAAI'25 & \ms{26.98}{0.92} & \ms{29.59}{0.45} & \ms{19.06}{1.09} & \ms{19.06}{1.09} &  & \ms{23.20}{0.65} & \ms{29.21}{0.88} & \ms{14.84}{0.29} & \ms{14.76}{0.37} \\
FedDr+~\cite{FedDrPlus_TMLR25} & TMLR'25 & \ms{34.92}{0.17} & \ms{31.97}{0.25} & \ms{27.50}{0.60} & \ms{25.21}{0.80} &  & \ms{28.11}{0.82} & \ms{30.79}{0.80} & \ms{20.67}{0.73} & \ms{20.67}{0.73} \\
BTFL~\cite{BTFL_KDD25} & KDD'25 & \ms{36.28}{0.24} & \ms{23.76}{0.58} & \ms{22.11}{0.35} & \ms{20.85}{0.57} &  & \ms{44.24}{0.98} & \ms{22.87}{0.47} & \ms{17.95}{0.75} & \ms{17.91}{0.70} \\
HPFL~\cite{HPFL_ICLR25} & ICLR'25 & \ms{40.14}{0.28} & \ms{28.08}{1.18} & \ms{20.56}{3.13} & \ms{20.47}{3.20} &  & \ms{47.90}{0.94} & \ms{25.15}{0.48} & \ms{14.78}{1.13} & \ms{14.78}{1.13} \\
FedYoYo~\cite{FedYoYo_ICCV25} & ICCV'25 & \ms{32.53}{2.02} & \ms{33.44}{2.24} & \ms{13.17}{3.08} & \ms{13.17}{3.08} &  & \ms{35.46}{1.78} & \ms{37.72}{2.10} & \ms{28.66}{5.62} & \ms{28.39}{5.48} \\
FedLog~\cite{FedLog_TMLR26} & TMLR'26 & \ms{30.53}{0.28} & \ms{0.00}{0.00} & \ms{0.00}{0.00} & \ms{0.00}{0.00} &  & \ms{42.06}{0.52} & \ms{0.00}{0.00} & \ms{0.00}{0.00} & \ms{0.00}{0.00} \\
pFedGM~\cite{pFedGM_arXiv26} & arXiv'26 & \ms{41.54}{0.73} & \ms{4.24}{0.43} & \ms{10.61}{0.77} & \ms{4.21}{0.38} &  & \mss{50.15}{0.35} & \ms{0.70}{0.10} & \ms{4.78}{1.27} & \ms{0.70}{0.10} \\
\addlinespace[2pt]
\multicolumn{11}{@{}l}{\grp{Statistical classifiers assembled on a vanilla backbone (FedAvg~+)}} \\
+ CCVR~\cite{CCVR_NeurIPS21} & NeurIPS'21 & \ms{31.99}{0.99} & \ms{31.78}{0.92} & \ms{22.33}{2.97} & \ms{22.33}{2.97} &  & \ms{27.92}{0.62} & \ms{32.42}{0.44} & \ms{18.83}{0.73} & \ms{18.83}{0.73} \\
+ FedNCM~\cite{FedNCM_NeurIPS23} & NeurIPS'23 & \ms{33.05}{0.56} & \ms{32.55}{0.86} & \ms{24.33}{3.09} & \ms{24.27}{3.10} &  & \ms{28.41}{0.15} & \ms{33.21}{0.09} & \ms{18.78}{1.00} & \ms{18.78}{1.00} \\
+ Fed3R~\cite{Fed3R_ICML24} & ICML'24 & \ms{31.25}{1.16} & \ms{29.69}{0.59} & \ms{21.83}{3.46} & \ms{21.72}{3.38} &  & \ms{26.98}{0.16} & \ms{31.51}{0.16} & \ms{16.78}{1.39} & \ms{16.78}{1.39} \\
+ FedCGS~\cite{FedCGS_AAAI25} & AAAI'25 & \ms{30.56}{0.52} & \ms{30.94}{0.65} & \ms{21.67}{3.38} & \ms{21.67}{3.38} &  & \ms{27.35}{0.57} & \ms{31.56}{0.59} & \ms{16.50}{1.59} & \ms{16.50}{1.59} \\
+ FedCOF~\cite{FedCOF_NeurIPS25} & NeurIPS'25 & \ms{32.17}{0.25} & \ms{31.22}{1.02} & \ms{23.89}{3.73} & \ms{23.89}{3.73} &  & \ms{27.66}{0.39} & \ms{31.88}{0.45} & \ms{17.89}{0.75} & \ms{17.89}{0.75} \\
+ GH-LDA~\cite{GHLDA_ICLR26} & ICLR'26 & \ms{31.45}{0.56} & \ms{32.04}{0.67} & \ms{22.78}{3.26} & \ms{22.78}{3.26} &  & \ms{28.12}{0.20} & \ms{32.37}{0.20} & \ms{18.28}{0.25} & \ms{18.28}{0.25} \\
\addlinespace[2pt]
\multicolumn{11}{@{}l}{\grp{Statistical classifiers assembled on the SOTA backbone (FedYoYo~+)}} \\
+ CCVR~\cite{CCVR_NeurIPS21} & NeurIPS'21 & \ms{49.09}{0.12} & \mss{47.52}{0.35} & \mss{42.28}{2.04} & \mss{41.10}{1.36} &  & \ms{44.61}{0.37} & \ms{48.11}{0.22} & \ms{36.22}{0.67} & \ms{36.22}{0.67} \\
+ FedNCM~\cite{FedNCM_NeurIPS23} & NeurIPS'23 & \ms{46.19}{0.63} & \ms{43.36}{0.21} & \ms{38.00}{1.32} & \ms{36.43}{1.18} &  & \ms{39.68}{0.41} & \ms{42.94}{0.53} & \ms{28.11}{0.42} & \ms{28.11}{0.42} \\
+ Fed3R~\cite{Fed3R_ICML24} & ICML'24 & \ms{48.57}{0.19} & \ms{45.91}{0.10} & \ms{38.28}{2.17} & \ms{36.39}{1.08} &  & \ms{44.04}{0.12} & \ms{47.80}{0.36} & \ms{31.72}{1.08} & \ms{31.72}{1.08} \\
+ FedCGS~\cite{FedCGS_AAAI25} & AAAI'25 & \mss{49.37}{0.21} & \ms{47.50}{0.35} & \ms{41.39}{1.55} & \ms{40.92}{1.32} &  & \ms{44.86}{0.72} & \mss{48.38}{0.72} & \mss{38.55}{1.49} & \mss{38.49}{1.44} \\
+ FedCOF~\cite{FedCOF_NeurIPS25} & NeurIPS'25 & \ms{47.68}{0.54} & \ms{45.24}{0.17} & \ms{41.22}{0.95} & \ms{39.68}{0.41} &  & \ms{42.98}{1.10} & \ms{44.72}{0.57} & \ms{33.61}{0.98} & \ms{33.61}{0.98} \\
+ GH-LDA~\cite{GHLDA_ICLR26} & ICLR'26 & \ms{48.60}{0.64} & \ms{46.77}{0.17} & \ms{40.67}{1.34} & \ms{39.97}{0.93} &  & \ms{44.54}{0.83} & \ms{47.63}{0.62} & \ms{36.89}{1.00} & \ms{36.82}{1.11} \\
\midrule
\rowcolor{gray!12}
\textbf{FedGRC (Ours)} & -- & \msb{52.68}{0.56} & \msb{52.90}{0.83} & \msb{48.33}{1.84} & \msb{47.56}{1.83} &  & \ms{49.00}{0.85} & \msb{52.58}{0.43} & \msb{41.83}{1.53} & \msb{41.83}{1.53} \\
$\Delta$ vs.\ Best &  & \dltp{$+$3.31} & \dltp{$+$5.38} & \dltp{$+$6.05} & \dltp{$+$6.46} &  & \dltm{$-$3.59} & \dltp{$+$4.20} & \dltp{$+$3.28} & \dltp{$+$3.34} \\
\bottomrule}%
\setlength{\tabcolsep}{0pt}%
\sbox{\tabbox}{\begin{tabular}{ll rrrr @{\hskip2\tabcolsep}c@{\hskip0.4\tabcolsep} rrrr}\tabbody\end{tabular}}%
\ifdim\wd\tabbox<\textwidth
  \setlength{\tabcolsep}{\dimexpr(\textwidth-\wd\tabbox)/22\relax}%
\else
  \setlength{\tabcolsep}{1pt}%
\fi
\begin{tabular}{ll rrrr @{\hskip2\tabcolsep}c@{\hskip0.4\tabcolsep} rrrr}\tabbody\end{tabular}
\end{table*}

\subsection{Communication and Computational Overhead}\label{app:overhead}
\subsubsection{Communication Cost}
As in FedAvg~\cite{FedAvg_AISTATS17}, the server broadcasts the current global model parameters to participating clients, which upload their updated parameters after local training.
FedGRC adds $O(Cd_L+C|\Lambda|)$ scalar entries in each direction per participating client per round: clients upload class directions, log-norm sums, and counts, while the server broadcasts updated directional and magnitude references alongside the model parameters.

After training, all $K$ clients participate once in the additive-moment construction of Appendix~\ref{app:moments}.
Each uploads class counts, class-wise first moments, and a single $d\times d$ second-moment matrix summed over all local samples, totaling $O(Cd+d^2)$ entries.
From the uploaded moments, the server derives the shared layer scales, class means, and covariance metric.
It broadcasts these quantities in $O(d^2+Cd+|\Lambda|)$ entries per client, after which each client constructs its personalized head locally.
For fixed feature dimensions and class count, this terminal payload is independent of local sample size.
Thus, MMS and SASP together require only one additional communication round to construct the personalized statistical heads.
In our experiments, FedGRC's total communication volume is 0.08\%, 0.41\%, and 0.25\% higher than that of FedAvg on CIFAR-10, CIFAR-100, and FedRS, respectively.

\newpage
\subsubsection{Computational Cost}
CGRL processes two augmented views per example, adding one view's forward and backward computation over single-view FedAvg for the same examples and updates.
The feature-level regularizers cost $O(B^2d_L+BCd_L+Bd)$ per update for contrastive, class-direction, and magnitude terms, respectively.
After local training in each participating round, one additional forward-only pass over the local training set extracts class directions; magnitude statistics are accumulated during training.
Server reference updates cost $O(|\mathcal S_t|C(d_L+|\Lambda|))$ per round.

To construct the personalized statistical heads with MMS and SASP, each client makes one pass over its training set using the frozen backbone; moment accumulation adds $O(N^{(k)}d^2)$ operations.
Aggregation, covariance construction, and construction of all personalized heads together cost $O(d^3+KCd^2)$ across the server and clients with dense arithmetic.
The server inverts the shared covariance once; each client constructs its head in $O(Cd^2)$ time.

At inference, multilayer features come from a single backbone pass, followed by layer scaling and an $O(Cd)$ linear head, compared with FedAvg's $O(Cd_L)$ head.
The covariance metric is incorporated into the precomputed head parameters.
Thus, FedGRC incurs negligible computational overhead at inference relative to FedAvg.

\newcommand{\AppendixNotationHeading}{\section{Notation}\label{app:notation}}
\begin{table*}[!tp]
\AppendixNotationHeading
\centering
\caption{Notation for the FedGRC framework, theoretical analysis, and experimental evaluation.}\label{tab:notation}\label{tab:notation-evaluation}
\small
\renewcommand{\arraystretch}{1.03}
\setlength{\tabcolsep}{6pt}
\begin{tabular}{@{}>{\raggedright\arraybackslash}p{0.22\textwidth}>{\raggedright\arraybackslash}p{\dimexpr0.78\textwidth-12pt\relax}@{}}
\toprule
Symbol(s) & Definition \\
\midrule
\multicolumn{2}{@{}l}{\footnotesize\grp{A. Federated Learning Problem Formulation}}\\*
$K,C,\mathcal Y$ & Client count, class count, and common federation-known label space. \\
$\mathcal D^{(k)},\mathcal D_c^{(k)},\mathcal I_c^{(k)}$ & Client $k$'s training set, class-$c$ subset, and indices of that subset. \\
$N^{(k)},N;\ n_c^{(k)},N_c$ & Training-sample counts for each client and the federation, both overall and for class $c$. \\
$\vect\theta=(\vect\phi,\vect\psi)$ & Model parameters: backbone $\vect\phi$ and classifier head $\vect\psi$. \\
$g,h_{\vect\phi},q_{\vect\psi}$ & Full model, feature extractor, and classifier head; $g=q_{\vect\psi}\circ h_{\vect\phi}$. \\
$P^{(k)},Q_c,Q$ & Local target distribution, class-$c$ evaluation distribution, and class-balanced all-class distribution. \\
$\mathcal E_R(f),e_c^{(k)}$\newline
$\mathcal E_{\mathrm{local}}^{(k)},\mathcal E_{\mathrm{global}}^{(k)}$ & Error probability of $f$ under labeled distribution $R$; client $k$'s class-$c$ error under $Q_c$; local and global risks under $P^{(k)}$ and $Q$, respectively. \\
\addlinespace[4pt]
\multicolumn{2}{@{}l}{\footnotesize\grp{B. Consensus-Guided Representation Learning}}\\*
$\Lambda,L,d_\ell,d$ & Selected layers, final layer, pooled feature dimension at $\ell$, and total dimension $d=\sum_{\ell\in\Lambda}d_\ell$. \\
$\mathbf h_\ell(x;\vect\phi),\mathbf u_i$ & Pooled layer features and unit directions; directional supervision uses the final layer. \\
$R_i,r_\ell(x),\rho_\ell(x)$ & Raw norm $R_i=\|\mathbf h_i\|_2$, safe norm $r_\ell(x)=\sqrt{\|\mathbf h_\ell(x;\vect\phi)\|_2^2+\varepsilon^2}$, and log norm $\rho_\ell(x)=\log r_\ell(x)$ for magnitude guidance. \\
$\mathbf p_c(t),a_{\ell,c}(t)$ & Shared class direction and class- and layer-specific log-norm reference at round $t$. \\
$v_c(t),\chi_{\ell i}$ & Magnitude reference initialization flag for class $c$; penalty indicator for layer $\ell$ and input $i$. \\
$B,\mathcal B_2,\mathcal P_i$ & Batch size, index set of $2B$ paired views, and same-class positive views excluding $i$. \\
$\mathcal T_{\mathrm w},\mathcal T_{\mathrm s}$ & Label-preserving weak and strong augmentation operators. \\
$\mathcal L_{\mathrm{CE}},\mathcal L_{\mathrm{CL}},\mathcal L_{\mathrm{proto}}$ & Training-head cross-entropy, SupCon loss, and shared-prototype cross-entropy. \\
$\mathcal L_{\mathrm{dir}},\mathcal L_{\mathrm{rad}},\mathcal L_{\mathrm{CGRL}}$ & Directional loss (contrastive and prototype), log-norm loss, and full local objective. \\
\addlinespace[4pt]
\multicolumn{2}{@{}l}{\footnotesize\grp{C. Statistical Modeling for Personalized Classification}}\\*
$s_\ell,\mathbf z(x)$ & Shared layer scale and concatenated scaled features $\mathbf z(x)=\operatorname{col}_{\ell\in\Lambda}(\mathbf h_\ell/s_\ell)$. \\
$\vect\mu_c,\vect\mu_c^{(k)},\widetilde{\vect\mu}_c^{(k)}$ & Shared, local, and support-dependent shrinkage centers; locally absent classes retain $\vect\mu_c$. \\
$\vect\Delta_c^{(k)},\vect\delta_c^{(k)}$ & Local-minus-shared mean displacement (zero for absent classes) and applied shift $\vect\delta_c^{(k)}=\alpha_c^{(k)}\vect\Delta_c^{(k)}$. \\
$m_0,\alpha_c^{(k)}$ & Shared-reference strength and local mean weight $\alpha_c^{(k)}=n_c^{(k)}/(n_c^{(k)}+m_0)$. \\
$\vect\Sigma,\widehat{\vect\Sigma}_{\mathrm w}$ & Common Gaussian covariance and empirical within-class covariance pooled about shared means. \\
$\widehat{\vect\Sigma}_{\ell,c},\widehat{\vect\Sigma}_c^{(k)}$ & Class covariances: pooled raw layer-$\ell$ features (divisor $N_c$); local scaled features (divisor $n_c^{(k)}$). \\
$\lambda_A,\mathbf A$ & Trace-scaled ridge $\lambda_A=\gamma\tr(\widehat{\vect\Sigma}_{\mathrm w})/d$ and shared metric $\mathbf A=(\widehat{\vect\Sigma}_{\mathrm w}+\lambda_A\mathbf I_d)^{-1}$. \\
$\pi_c$ & Class prior; FedGRC uses equal priors $\pi_c=1/C$. \\
$J_c^{(k)}(\mathbf m)$ & Mean-fitting objective balancing local fit and shared-center regularization under $\mathbf A$. \\
$\mathbf W^{(k)},\mathbf b^{(k)}$ & Client classifier from personalized means and shared $\mathbf A$; row $c$: $(\mathbf w_c^{(k)})^\top$; bias $c$: $b_c^{(k)}$. \\
\addlinespace[4pt]
\multicolumn{2}{@{}l}{\footnotesize\grp{D. Prediction Stability under Personalization}}\\*
$s_c^0,s_c^{(k)};\ f^0,f^{(k)}$ & Shared and personalized class scores and corresponding label predictors. \\
$\epsilon_c^{(k)}(\mathbf z)$ & Bound on absolute score changes from mean adaptation; distinct from safe-norm constant $\varepsilon$. \\
$M^0(x),B^{(k)}(x)$ & Shared winning margin and perturbation budget (winner's score bound plus largest competing bound); $c^\star=f^0(x)$ is the shared prediction. \\
\addlinespace[4pt]
\multicolumn{2}{@{}l}{\footnotesize\grp{E. Theoretical Analysis}}\\*
$\Pr,\mathbb E,\Var;\ \tr(\cdot)$ & Probability, expectation, variance; matrix trace (sum of diagonal entries). \\
$\mathbf1\{\cdot\},\ \|\cdot\|_2$ & Indicator (1 if true, 0 otherwise); Euclidean norm for vectors or spectral norm for matrices. \\
$\mathbf M\succ0,\ \|\mathbf v\|_{\mathbf M}$ & Positive-definite matrix and induced norm: $\|\mathbf v\|_{\mathbf M}^2=\mathbf v^\top\mathbf M\mathbf v$. \\
\addlinespace[4pt]
\multicolumn{2}{@{}l}{\footnotesize\grp{F. Experimental Setup and Evaluation}}\\*
$\mathcal T_{\mathrm{local}}^{(k)},\mathcal T_{\mathrm{global}}$ & Client $k$'s local test set and shared test set covering all federation-known classes. \\
$\mathcal U^{(k)},\mathcal R^{(k)}$ & Locally absent classes and up to two observed classes with the fewest local training samples; ties are broken by class ID. \\
$a_c^{(k)},t_c$ & Client $k$'s correct class-$c$ prediction count and class-$c$ sample count in the shared test set. \\
$m_L^{(k)},m_U^{(k)},m_R^{(k)}$ & Client $k$'s test-sample counts for local, unseen-class, and rare-class evaluation. \\
$L^{(k)},U^{(k)},R^{(k)}$ & Client accuracies on the local test set and shared-test subsets for unseen and rare classes. \\
$F^{(k)}$ & Client minimum $\min\{L^{(k)},U^{(k)},R^{(k)}\}$, omitting $U^{(k)}$ when $\mathcal U^{(k)}$ is empty. \\
$\operatorname{Tail}_{25},q,q_v,m_v$ & Mean of the lowest $q_v=\min(q,m_v)$ client scores; $q=\lceil K/4\rceil$; $m_v$: eligible-client count. \\
\textit{local}, \textit{global} & Client weights: $m_L^{(k)}$ for local accuracy; equal class weights within clients, then equal client weights for global accuracy. \\
\textit{unseen}, \textit{rare} & Subset accuracies with client weights $m_U^{(k)}$ and $m_R^{(k)}$; unseen excludes clients with $\mathcal U^{(k)}=\varnothing$. \\
\textit{L@25}, \textit{U@25}, \textit{R@25}; \textit{F@25} & $\operatorname{Tail}_{25}$ applied separately to $L^{(k)}$, $U^{(k)}$, $R^{(k)}$, and $F^{(k)}$; \textit{F@25} takes client minima before tail averaging, not the minimum of ability-specific tail means. \\
\bottomrule
\end{tabular}
\end{table*}

\clearpage
\endgroup
\twocolumn
\putbib[references]
\end{bibunit}

\end{document}